\documentclass[twoside,11pt]{article}

\usepackage{jmlr2e}

\usepackage{url}
\usepackage{amsmath,amssymb,amsbsy}
\usepackage{paralist}
\usepackage{xcolor}
\usepackage{color}
\usepackage{graphicx}
\graphicspath{{./figs/}}
\usepackage{algorithm}
\usepackage{algorithmic}
\usepackage{comment}
\usepackage{wrapfig}

\usepackage{fancyhdr}
\usepackage{cleveref}

\usepackage{enumitem}
\usepackage{caption}
\usepackage{subcaption}
\usepackage{grffile}

\newtheorem{thm}{Theorem}
\newtheorem{cor}{Corollary}
\newtheorem{defi}{Definition}
\newtheorem{lem}{Lemma}

\newcommand{\R}{\mathbb{R}}

\newcommand{\N}{\mathbb{N}}

\newcommand{\e}{\begin{equation}}
\newcommand{\ee}{\end{equation}}
\newcommand{\en}{\begin{equation*}}
\newcommand{\een}{\end{equation*}}
\newcommand{\eqn}{\begin{eqnarray}}
\newcommand{\eeqn}{\end{eqnarray}}
\newcommand{\bmat}{\begin{bmatrix}}
\newcommand{\emat}{\end{bmatrix}}

\DeclareMathAlphabet\mathbfcal{OMS}{cmsy}{b}{n}
\renewcommand{\P}[1]{\operatorname{\mathbb{P}}\left[#1\right]}
\newcommand{\E}{\operatorname{\mathbb{E}}}
\newcommand{\dif}{\operatorname{d}}

\renewcommand{\d}[1]{\,\mathrm{d}#1}

\newcommand{\vct}[1]{\boldsymbol{#1}}
\newcommand{\mtx}[1]{\boldsymbol{#1}}

\newcommand{\T}{\mathrm{T}}

\newcommand{\Span}{\operatorname{Span}}
\newcommand{\trace}{\operatorname{trace}}

\newcommand{\sign}{\operatorname{sign}}
\newcommand{\sgn}{\operatorname{sgn}}
\newcommand{\Sign}{\operatorname{Sgn}}

\newcommand{\set}[1]{\mathbb{#1}}

\DeclareMathOperator*{\minimize}{\text{minimize}}

\DeclareMathOperator*{\argmin}{\text{arg~min}}

\def \st {\operatorname*{s.t.\ }}

\newcommand{\eps}{\epsilon}

\newcommand{\vect}[1]{\boldsymbol{#1}}

\newcommand{\bfcalX}{\mathbfcal{X}}
\newcommand{\bfcalO}{\mathbfcal{O}}

\newcommand{\calA}{\mathcal{A}}

\newcommand{\calF}{\mathcal{F}}

\newcommand{\calH}{\mathcal{H}}

\newcommand{\calN}{\mathcal{N}}

\newcommand{\calP}{\mathcal{P}}

\newcommand{\calS}{\mathcal{S}}

\newcommand{\calX}{\mathcal{X}}

\newcommand{\va}{\vct{a}}
\newcommand{\vb}{\vct{b}}

\newcommand{\vd}{\vct{d}}
\newcommand{\ve}{\vct{e}}

\newcommand{\vg}{\vct{g}}

\newcommand{\vn}{\vct{n}}
\newcommand{\vo}{\vct{o}}
\newcommand{\vp}{\vct{p}}
\newcommand{\vq}{\vct{q}}

\newcommand{\vs}{\vct{s}}

\newcommand{\vv}{\vct{v}}
\newcommand{\vw}{\vct{w}}
\newcommand{\vx}{\vct{x}}

\newcommand{\vxi}{\vct{\xi}}

\newcommand{\vzero}{\vct{0}}

\newcommand{\mE}{\mtx{E}}

\newcommand{\mL}{\mtx{L}}

\newcommand{\mP}{\mtx{P}}

\newcommand{\mS}{\mtx{S}}

\newcommand{\mV}{\mtx{V}}

\newcommand{\mGamma}{\mtx{\Gamma}}

\newcommand{\mId}{{\bf I}}

\newcommand{\mzero}{{\bf 0}}

\newcommand{\setA}{\set{A}}
\newcommand{\setB}{\set{B}}

\newcommand{\setF}{\set{F}}
\newcommand{\setG}{\set{G}}

\newcommand{\setS}{\set{S}}

\graphicspath{{./figs/}}

\newlength{\imgwidth}
\newcommand{\note}[1]{\textcolor{blue}{\bf [{\em Note:} #1]}}

\newboolean{twoColVersion}
\setboolean{twoColVersion}{false}
\newcommand{\twoCol}[2]{\ifthenelse{\boolean{twoColVersion}} {#1} {#2} }

\usepackage{booktabs}
\usepackage{multirow}
\usepackage[normalem]{ulem}
\useunder{\uline}{\ul}{}
\ShortHeadings{Dual Principal Component Pursuit}{Zhu, Wang, Robinson, Naiman, Vidal, and Tsakiris}
\firstpageno{1}

\begin{document}
	
	\title{Dual Principal Component Pursuit:\\ Probability Analysis and Efficient Algorithms\thanks{A preliminary version of this paper highlighting some of the key results also appeared in~\cite{Zhu:DPCP-NIPS18}, compared to which the current paper includes the formal proofs and additional results concerning the convergence of Alternating Linerization and Projection Method in \Cref{sec:ALP}.}}

	
\author{\name Zhihui Zhu \email zzhu29@jhu.edu \\
    \addr Mathematical Institute for Data Science, Johns Hopkins University,
    Baltimore,  USA.
		\AND
		\name Yifan Wang  \email wangyf@shanghaitech.edu.cn\\
			\addr School of Information Science and Technology, ShanghaiTech University, Shanghai, China.
		\AND 
		\name Daniel P. Robinson \email daniel.p.robinson@jhu.edu\\
		\addr Department of Applied Mathematics and Statistics, Johns Hopkins University,
		Baltimore,  USA.
		\AND 
		\name  Daniel Q. Naiman \email daniel.naiman@jhu.edu\\
		\addr Department of Applied Mathematics and Statistics, Johns Hopkins University,
		Baltimore,  USA.
		\AND
		\name Rene Vidal \email rvidal@jhu.edu\\
		\addr Mathematical Institute for Data Science, Johns Hopkins University,
		Baltimore,  USA.\\
		\name Manolis C. Tsakiris \email mtsakiris@shanghaitech.edu.cn  \\
		\addr School of Information Science and Technology, ShanghaiTech University, Shanghai, China.}
	
	\editor{}
	
	\maketitle
	
	\begin{abstract}
	Recent methods for learning a linear subspace from data corrupted by outliers are based on convex $\ell_1$ and nuclear norm optimization and require the dimension of the subspace and the number of outliers to be sufficiently small \citep{Xu:NIPS10}. In sharp contrast, the recently proposed \emph{Dual Principal Component Pursuit (DPCP)} method \citep{Tsakiris:DPCPICCV15} can provably handle subspaces of high dimension by solving a non-convex $\ell_1$ optimization problem on the sphere. However, its geometric analysis is based on quantities that are difficult to interpret and are not amenable to  statistical analysis. In this paper we provide a refined geometric analysis and a new statistical analysis that show that DPCP can tolerate as many outliers as the {\em square} of the number of inliers, thus improving upon other provably correct robust PCA methods. We also propose a scalable {\em Projected Sub-Gradient Method} method (DPCP-PSGM) for solving the DPCP problem and show it admits linear convergence even though the underlying optimization problem is non-convex and non-smooth. Experiments on road plane detection from 3D point cloud data demonstrate that DPCP-PSGM can be more efficient than the traditional RANSAC algorithm, which is one of the most popular methods for such computer vision applications.
	\end{abstract}
	


\section{Introduction}

Fitting a linear subspace to a dataset corrupted by outliers is a fundamental problem in machine learning and statistics, primarily known as \emph{(Robust) Principal Component Analysis (PCA)} or Robust Subspace Recovery \citep{Jolliffe:Springer86,Candes:ACM10,Lerman:ReviewArXiv18}. The classical formulation of PCA, dating back to Carl F. Gauss, is based on minimizing the sum of squares of the distances of all points in the dataset to the estimated linear subspace. Although this problem is non-convex, it admits a closed form solution given by the span of the top eigenvectors of the data covariance matrix. 
Nevertheless, it is well-known that the presence of outliers can severely affect the quality of the computed solution because the Euclidean norm is not robust to outliers. To clearly specify the problem, given a unit $\ell_2$-norm dataset $\widetilde \bfcalX = [\bfcalX \, \,  \bfcalO] \mGamma \in \R^{D\times L}$, where $\bfcalX \in \R^{D\times N}$ are inlier points spanning a $d$-dimensional subspace $\calS$ of $\R^D$, $\bfcalO$ are outlier points having no linear structure, and $\mGamma$ is an unknown permutation, the goal  is to recover the inlier space $\calS$ or equivalently to cluster the points into inliers and outliers.

\subsection{Related Work}

The sensitivity of classical $\ell_2$-based PCA to outliers has been addressed by using robust maximum likelihood covariance estimators, such as the one in \citep{Tyler:AS87}. However, the associated optimization problems are non-convex and thus difficult to provide global optimality guarantees. Another classical approach is the exhaustive-search method of \emph{Random Sampling And Consensus (RANSAC)} \citep{RANSAC}, which given a time budget, computes at each iteration a $d$-dimensional subspace as the span of $d$ randomly chosen points, and outputs the subspace that agrees with the largest number of points. Even though RANSAC is currently one of the most popular methods in many computer vision applications such as multiple view geometry \citep{Hartley-Zisserman04}, its performance is sensitive to the choice of a thresholding parameter.  Moreover, the number of required samplings may become prohibitive in cases when the number of outliers is very large and/or the subspace dimension $d$ is large and close to the dimension $D$ of the ambient space (i.e., the high relative dimension case).

As an alternative to traditional robust subspace learning methods, during the last decade ideas from compressed sensing \citep{Candes:SPM08} have given rise to a new class of methods that are based on convex optimization, and admit elegant theoretical analyses and efficient algorithmic implementations.

\paragraph{$\ell_{2,1}$-RPCA}
A prominent example is based on decomposing the data matrix into low-rank (corresponding to the inliers) and column-sparse parts (corresponding to the ourliers) \citep{Xu:NIPS10}. Specifically, $\ell_{2,1}$-RPCA solves the following problem
\e
\minimize_{\mL,\mE} \|\mL\|_* + \lambda \|\mE\|_{2,1}, \ \st \ \mL + \mE = \widetilde \bfcalX,
\label{eq:l21RPCA}\ee
where $\|\mL\|_*$ denotes the nulcear norm of $\mL$ (a.k.a the sum of the singular values of $\mL$) and $\|\mE\|_{2,1}$ amounts to the sum of the $\ell_2$ norm of each column in $\mE$. The main insight behind $\ell_{2,1}$-RPCA  is that $[\bfcalX \, \,  \mzero] \mGamma$ has rank $d$ which is small for a relative low dimensional subspace, while $[\mzero \, \,  \bfcalO] \mGamma$ has zero columns. Since both the $\ell_{2,1}$ norm and the nuclear norm are convex functions and the later can be optimized efficiently via SemiDefinite Programming (SDP), \eqref{eq:l21RPCA}  can be solved using a general SDP solver such as SDPT3 or SeDuMi. However, such methods do not scale well to large
data-sets and the SDP  becomes prohibitive even for modest problem sizes. As a consequence, the authors \citep{Xu:NIPS10} proposed an efficient algorithm based on proximal gradient descent. The main limitation of this method is that it is theoretically justifiable only for subspaces of low relative dimension $d/D$ (and hence $[\bfcalX \, \,  \mzero] \mGamma$ is low-rank compared to the ambient dimension $D$).

\paragraph{Coherence Pursuit (CoP)} Motivated by the observation that inliers lying in a low-dimensional subspace are expected to be correlated with each other while the outliers do not admit such strucuture or property, CoP \citep{Rahmani:arXiv16}
first measures the coherence of each point with every other point in the dataset. Compared to the outliers, the inliers tend to have higher coherence with the rest of the points in the data set. Thus, the inliers are then simply detected by selecting sufficiently number of points that have the highest coherence with the rest of the points until they span a $d$-dimensional subspace. This algorithm can be implemented very efficiently as it only involves computing the Gram matrix of the data set for the coherence and also as demonstrated in \citep{Rahmani:arXiv16}, CoP has a competitive performance. Similar to the $\ell_{2,1}$-RPCA, the main limitation of CoP is that it is theoretically guaranteed only for $d<\sqrt{D}$. However, for applications such as $3$D point cloud analysis, two/three-view geometry in computer vision, and system identification, a subspace of dimension $D-1$ (high relative dimension) is sought \citep{Spath:Numerische87, Wang:CVPR15-Regression}. 

\paragraph{REAPER}
A promising direction towards handling subspaces of high relative dimension is minimizing the sum of the distances of the points to the subspace\footnote{$\text{eig}(\cdot)$ denotes the set of eigenvalues.}
\e
\minimize_{\mP_d} \sum_{j=1}^L \|(\mId - \mP_d)\widetilde \vx_j\|_2 \ \st \  \trace(\mP_d) = d, \mP_d = \mP_d^\top, \text{eig}(\mP_d)=\{0,1\}, 
\label{eq:REAPER}\ee
where $\mP_d$ is an orthogonal projection onto any $d$-dimensinal subspace. Due to the nonconvexity of the set of orthogonal projections, \eqref{eq:REAPER} is a non-convex problem, which REAPER \citep{lerman2015robust} relaxes to the following SDP:
\e
\minimize_{\mP_d} \sum_{j=1}^L \|(\mId - \mP_d)\widetilde \vx_j\|_2 \ \st \  \trace(\mP_d) = d, \ \mzero \preceq \mP_d \preceq \mId.
\label{eq:REAPER relaxed}\ee
To overcome the computational issue of the above SDP solved by general SDP solvers, an \emph{Iteratively Reweighted Least Squares} (IRLS) scheme is proposed in \citep{lerman2015robust} to obtain a numerical solution of
\eqref{eq:REAPER relaxed}. Even though in practice REAPER outperforms low-rank methods \citep{Xu:NIPS10,Soltanolkotabi:AS14,You:CVPR17,Rahmani:arXiv16} for subspaces of high relative dimension, its theoretical guarantees still require $d<(D-1)/2$. 

\paragraph{Geodesic Gradient Descent (GGD)}The limitation mentioned above of REAPER is improved upon by the recent work of \citep{Maunu:arXiv17} that directly considers \eqref{eq:REAPER} by parameterizing the orthogonal projector $\mP_d = \mV\mV^\top$ and addressing its equivalent form
\e
\minimize_{\mV\in\R^{D\times d}} \sum_{j=1}^L \|(\mId - \mV\mV^\top)\widetilde \vx_j\|_2 \ \st \  \mV^\T\mV = \mId,
\label{eq:GGD}\ee
which is then solved by using a geodesic gradient  descent (GGD) algorithm (which is a subgradient method on the Grassmannian in the strict sense). When initialized near the inlier subspace, GGD is theoretically guaranteed to converge to the inlier subspace, 
despite the nonconvexity of \eqref{eq:REAPER}. In particular, with a random Gaussian model for the dataset, this convergence holds with high-probability for any $d/D$, as long as the number of outliers $M$ scales as $O(\frac{\sqrt{D(D-d)}}{d} N)$.  Also as a first-order method in nature, the GGD can be implemented efficiently with each iteration mainly involving the computation of a subgradient,  its singular value decompotion (SVD) and update of the subspace algong the geodesic direction. However, only a sublinear convergence rate ($O(1/\sqrt{k})$ where $k$ indicaes the total number of iteration) was formally established for GGD in \citep{Maunu:arXiv17}, although a fast convergence was experimentally observed with a line search method for choosing the step sizes.

\paragraph{Dual Principal Component Pursuit (DPCP)} As a dual approach to  \eqref{eq:REAPER}, the recently proposed DPCP method~\citep{Tsakiris:DPCPICCV15,Tsakiris:DPCP-Arxiv17,Tsakiris:ICML17} seeks to learn recursively a basis for the orthogonal complement of the subspace by solving an $\ell_1$ minimization problem on the sphere. In particular, the main idea of DPCP  is to first compute a hyperplane $\calH_1$ that contains all the inliers $\bfcalX$. Such a hyperplane can be used to discard a potentially very large number of outliers, after which a method such as RANSAC may successfully be applied to the reduced dataset \footnote{Note that if the outliers are in general position, then $\calH_1$ will contain at most $D-d-1$ outliers.}. Alternatively, if $d$ is known, then one may proceed to recover $\calS$ as the intersection of $D-d$ orthogonal hyperplanes that contain $\bfcalX$. In any case, DPCP computes a normal vector $\vb_1$ to the first hyperplane $\calH_1$ as follows: 
\e
\min_{\vb \in \R^D} \, \|\widetilde {\mathbfcal X}^\top \vb\|_0 \ \st \ \vb \neq 0.
\label{eq:dpcp0}\ee
Notice that the function $\|\widetilde {\mathbfcal X}^\top \vb\|_0$ being minimized simply counts how many points in the dataset that are not contained in the hyperplane with normal vector $\vb$. Assuming that there are at least $d+1$ inliers and at least $D-d$ outliers (this is to avoid degenerate situations), and that all points are in general position\footnote{Every $d$-tuple of inliers is linearly independent, and every $D$-tuple of outliers is linearly independent.}, then every solution $\vb^*$ to \eqref{eq:dpcp0} must correspond to a hyperplane that contains $\bfcalX$, and hence $\vb^*$ is orthogonal to $\calS$. Since \eqref{eq:dpcp0} is computationally intractable, it is reasonable to replace it by\footnote{This optimization problem also appears in different contexts (e.g., \citep{Qu:NIPS14} and \citep{Spath:Numerische87}).}
\e
\min_{\vb \in \R^D}f(\vb):=\|\widetilde {\mathbfcal X}^\top \vb\|_1 \ \st \ \|\vb\|_2 = 1.
\label{eq:dpcp}\ee

In fact, the above optimization problem is equivalent to \eqref{eq:GGD} for the special case $d=D-1$. As shown in \citep{Tsakiris:DPCPICCV15,Tsakiris:DPCP-Arxiv17}, as long as the points are well distributed in a certain deterministic sense, any global minimizer of this non-convex problem is guaranteed to be a vector orthogonal to the subspace, regardless of the outlier/inlier ratio and the subspace dimension; a result that agrees with the earlier findings of \citep{Lerman:CA14}. Indeed, for synthetic data drawn from a hyperplane ($d=D-1$), DPCP has been shown to be the only method able to correctly recover the subspace with up to $70\%$ outliers ($D=30$).  Nevertheless, the analysis of \citep{Tsakiris:DPCPICCV15,Tsakiris:DPCP-Arxiv17} involves geometric quantities that are difficult to analyze in a probabilistic setting, and consequently it has been unclear how the number $M$ of outliers that can be tolerated scales as a function of the number $N$ of inliers. Moreover, even though \citep{Tsakiris:DPCPICCV15,Tsakiris:DPCP-Arxiv17} show that relaxing the non-convex problem to a sequence of linear programs (LPs) guarantees finite convergence to a vector orthogonal to the subspace, this approach is computationally expensive. See \Cref{sec:ALP} for the corresponding Alternating Linerization and Projection (ALP) method. Alternatively, while the IRLS scheme proposed in \citep{Tsakiris:DPCP-Arxiv17,Tsakiris:ICML17} is more efficient than the linear programming approach, it comes with no theoretical guarantees and scales poorly for high-dimensional data, since it involves an SVD at each iteration.

We note that the literature on this subject is vast and the above mentioned related work is far from exhaustive; other methods such as the Fast Median Subspace (FMS) \citep{Lerman:II17},  Geometric Median Subspace (GMS) \citep{Zhang:JMLR14}, Tyler M-Estimator (TME) \citep{Zhang:II16}, Thresholding-based Outlier Robust PCA (TORP) \citep{Cherapanamjeri:Arxiv17}, expressing each data point as a sparse linear combination of other data points \citep{Soltanolkotabi:AS14,You:CVPR17}, online subspace learning~\citep{Balzano:Allerton10}, etc. For many other related methods, we refer to a recent review article \citep{Lerman:ReviewArXiv18} that thoroughly summarizes the entire body of work on robust subspace recovery. 

\subsection{Main Contribution and Outline}
The focus of the present paper is to provide improved deterministic analysis as well as probability analysis of global optimialtiy and efficient algorithms for the DPCP approach.  We make the following specific contributions.

\begin{itemize}[leftmargin=*]
\item  \emph{Theory:} Although the conditions established in \citep{Tsakiris:DPCPICCV15,Tsakiris:DPCP-Arxiv17}  suggest that global minimizers of \eqref{eq:dpcp} are orthogonal to $\calS$ (if the outliers are well distributed on the unit sphere and the inliers are well distributed on the intersection of the unit sphere with the subspace $\calS$), they are deterministic in nature and difficult to interpret. In \Cref{sec:global optimality analysis}, we provide an improved analysis of global optimality for DPCP that  replaces the cumbersome geometric quantities in \citep{Tsakiris:DPCPICCV15,Tsakiris:DPCP-Arxiv17} with new quantities that are both tighter and easier to bound in probability. In order to do this, we first provide a geometric characterization of the critical points of \eqref{eq:dpcp} (see Lemma \ref{lem:critical-point}), revealing that any critical point of \eqref{eq:dpcp} is either orthogonal to the inlier subspace $\calS$, or very close to $\calS$, with its principal angle  from the inlier subspace being smaller for well distributed points and smaller outlier to inlier ratios $M/N$.  Employing a spherical random model, the improved global optimality condition suggests that DPCP can handle $M=O(\frac{1}{dD\log^2 D}N^2)$ outliers. This is in sharp contrast to existing provably correct state-of-the-art robust PCA methods, which as per the recent review \cite[Table I]{Lerman:ReviewArXiv18} can tolerate at best $M=O(N)$ outliers, when the subspace dimension $d$ and the ambient dimension $D$ are fixed. The comparison of the largest number of outliers can be tolarated by different approaches is listed in \Cref{table:review}, where the random models are descripted in \Cref{sec:comparison}.
	
\item \emph{Algorithms:} In \Cref{sec:algorithm}, we first establish conditions under which solving one linear program returns a normal vector to the inlier subspace. Specifically, given an appropriate $\vb_0$, the optimal solution of the following linear program
\[
\min_{\vb \in \R^D}f(\vb):=\|\widetilde {\mathbfcal X}^\top \vb\|_1 \ \st \ \vb_0^\top\vb = 1
\]
must be orthogonal to the inlier subspace. This improves upon the convergence guarantee in \citep{Tsakiris:DPCP-Arxiv17,Tsakiris:ICML17} where the ALP is only proved to converge in a finite number of iterations, but without any explict convergence rate. To further reduce the computational burden, we then provide a scalable \emph{Projected Sub-Gradient Method}  with piecewise geometrically diminishing step size (DPCP-PSGM), which is proven to solve the non-convex DPCP problem \eqref{eq:dpcp} with linear convergence and using only matrix-vector multiplications. This is in sharp contrast to classic results in the literature on PSGM methods, which usually require the problem to be convex in order to  establish sub-linear convergence~\citep{Boyd:2003subgradient}.
	DPCP-PSGM is orders of magnitude faster than the ALP and IRLS schemes proposed in~\citep{Tsakiris:DPCP-Arxiv17}, allowing us to extend the size of the datasets that we can handle from $10^3$ to $10^6$ data points. 
	
	\item \emph{Experiments:} Aside from experiments with synthetical data, we conduct experiments on road plane detection from 3D point cloud data using the KITTI dataset \citep{geiger2013vision}, which is an important computer vision task in autonomous car driving systems. The experiments show that for the same computational budget DPCP-PSGM outperforms RANSAC, which is one of the most popular methods for such computer vision applications.
\end{itemize}

\setlength{\tabcolsep}{6pt}
\begin{table*}[htp!]\caption{Comparison of the largest number of outliers that can be tolarated by different approaches under a random spherical model (for the FMS, CoP and DPCP) or a random Gaussian model (for all the other approaches).}\label{table:review}
	\begin{center}
		\small
		\renewcommand{\arraystretch}{1.6}
		\begin{tabular}{c|c}
			\hline  Method & random Gaussian model\\
			\hline \hline  
			GGD  & $M\lesssim  \frac{\sqrt{D(D-d)}}{d} N$\\
			\hline REAPER  & $M \lesssim  \frac{D}{d}N, \quad d\leq \frac{D-1}{2}$\\
			\hline GMS  & $M \lesssim  \frac{\sqrt{(D-d)D}}{d}N$ \\
			\hline $\ell_{2,1}$-RPCA & $M \lesssim \frac{1}{d\max(1,\log(L)/d)}N$\\
			\hline TME  & $M < \frac{D-d}{d}N$\\
			\hline TORP &  $M \lesssim \frac{1}{d\max(1,\log(L)/d)^2} N$\\
			\hline
		\end{tabular}
		\quad	
		\renewcommand{\arraystretch}{1.6}
		\begin{tabular}{c|c}
			\hline  Method & random spherical model\\
			\hline \hline  
			FMS  & 
			\begin{tabular}{c}
				$N/M \gtrapprox 0$, $N\rightarrow \infty$, i.e., \\
				any ratio of outliers when\footnotemark $N\rightarrow \infty$
			\end{tabular}\\
			\hline CoP & $M \lesssim \frac{D- d^2}{d}N, \quad d < \sqrt{D}$ \\
			\hline DPCP & $M \lesssim \frac{1}{dD\log^2D}N^2$ \\
			\hline
		\end{tabular}
	\end{center}
\end{table*}
\footnotetext{This asymptotic result assumes $d$ and $D$ fixed and thus these two parameters are omitted.}

\subsection{Notation}
We briefly introduce some of the notations used in this paper. Finite-dimensional vectors and matrices are indicated by bold characters. The symbols $\mId$ and $\mzero$ represent the identity matrix and zero matrix/vector, respectively. We denote the sign function by 
\e
\sign(a):=\left\{\begin{matrix}a/|a|, & a\neq 0,\\ 0, & a= 0.\end{matrix}\right.
\nonumber\ee
We also require the sub-differential $\Sign$ of the absolute value function $|a|$ defined as
\e
\Sign(a):=\left\{\begin{matrix}a/|a|, & a\neq 0,\\ [-1,1], & a= 0.\end{matrix}\right.
\label{eq:Sgn}\ee
Denote by $\sgn(a)$ an arbitrary point in $\Sign(a)$.
We also use $\sign(\va)$ to indicate that we apply the $\sign$ function element-wise to the vector $\va$ and similarly for $\Sign$ and $\sgn$. The unit sphere of $\R^D$ is denoted by $\setS^{D-1}:=\{\vb\in\R^D: \|\vb\|_2 =1\}$.

\section{Global Optimality Analysis for Dual Principal Component Pursuit}
\label{sec:global optimality analysis}
Although problem~\eqref{eq:dpcp} is non-convex (because of the constraint) and non-smooth (because of the $\ell_1$ norm), the work of \citep{Tsakiris:DPCPICCV15,Tsakiris:DPCP-Arxiv17} established conditions suggesting that if the outliers are well distributed on the unit sphere and the inliers are well distributed on the intersection of the unit sphere with the subspace $\calS$, then global minimizers of \eqref{eq:dpcp} are orthogonal to $\calS$. Nevertheless, these conditions are deterministic in nature and difficult to interpret. In this section, we give improved global optimality conditions that are i) tighter, ii) easier to interpret and iii) amenable to a probabilistic analysis.

\subsection{Geometry of the Critical Points} 
The heart of our analysis lies in a tight geometric characterization of the critical points of \eqref{eq:dpcp} (see Lemma \ref{lem:critical-point} below). Before stating the result, we need to introduce some further notation and definitions. Letting $\calP_{\calS}$ be the orthogonal projection onto $\calS$, we define the {\em principal angle}  of $\vb$ from $\calS$ as $\phi\in[0,\frac{\pi}{2}]$ such that $\cos(\phi) = \|\calP_{\calS}(\vb)\|_2/\|\vb\|_2$. Since we will consider the first-order optimality condition of \eqref{eq:dpcp}, we naturally need to compute the sub-differential of the objective function in \eqref{eq:dpcp}. Since $f$ is convex, its sub-differental at $\vb$ is defined as
\[
\partial f(\vb): = \left\{ \vd'\in\R^{D}: f(\va) \ge f(\vb) + \langle \vd', \va - \vb\rangle, \ \forall \va\in\R^D   \right\},
\]
where each $\vd'\in\partial f(\vb)$ is called a subgradient of $f$ at $\vb$. Note that the $\ell_1$ norm is subdifferetially regular. By the chain rule for subdifferentials of subdifferentially regular functions, we have
\e
\partial f(\vb) =  \widetilde\bfcalX\Sign(\widetilde\bfcalX^\top \vb).
\ee

Next, global minimizers of \eqref{eq:dpcp} are critical points in the following sense:
\begin{defi} \label{dfn:CriticalPoint} $\vb\in \setS^{D-1}$ is called a critical point of \eqref{eq:dpcp} if there is $\vd' \in \partial f(\vb)$ such that the Riemann gradient $ \vd=(\mId - \vb\vb^\top)\vd' = \vzero$.
\end{defi}
We now illustrate the key idea behind characterizing the geometry of the critical points. Let $\vb$ be a critical point that is not orthogonal to $\calS$. Then, under general position assumptions on the data, $\vb$ can be orthogonal to $K \le D-1$ columns of $\bfcalO$. It follows that any Riemann sub-gradient evaluated at $\vb$ has the form
\begin{align}
\vd=  (\mId - \vb\vb^\top)\bfcalX \sgn(\bfcalX^\top \vb) + (\mId - \vb\vb^\top)\bfcalO \sign( \bfcalO^\top \vb) + \vxi,
\label{eq:riemann subgrad}\end{align} where $\vxi = \sum_{i=1}^K \alpha_{j_i} \vo_{j_i}$ with $\vo_{j_1},\dots,\vo_{j_K}$ the columns of $\bfcalO$ orthogonal to $\vb$ and $\alpha_{j_1},\dots,\alpha_{j_K} \in [-1,1]$. Note that $\| \vxi \|_2 < D$. Since $\vb$ is a critical point, Definition \ref{dfn:CriticalPoint} implies a choice of $\alpha_{j_i}$ so that $\vd = \vzero$. Decompose $\vb = \cos(\phi) \vs + \sin(\phi) \vn$, where $\phi$ is the principal angle of $\vb$ from $\calS$, and $\vs = \calP_\calS(\vb)/\|\calP_\calS(\vb)\|_2$ and $\vn = \calP_{\calS^\perp}(\vb)/\|\calP_{\calS^\perp}(\vb)\|_2$ are the orthonormal projections of $\vb$ onto $\calS$ and $\calS^\perp$, respectively. Defining $\vg = -\sin(\phi) \vs + \cos(\phi) \vn$ and noting that $\vg \perp \vb$, it follows that
\begin{align*}
0 = \langle \vd,\vg \rangle &=  \vg^\top \bfcalO \sign( \bfcalO^\top \vb) +\vg^\top \bfcalX \sgn( \bfcalX^\top \vb)+ \vg^\top \vxi\\
&=  \vg^\top \bfcalO \sign( \bfcalO^\top \vb) -\sin(\phi)\vs^\top \bfcalX \sgn( \bfcalX^\top \vs)+ \vg^\top \vxi\\
&=  \vg^\top \bfcalO \sign( \bfcalO^\top \vb) - \sin(\phi) \left\|\bfcalX^\top
\vs \right\|_1+ \vg^\top \vxi,
\end{align*}
 which in particular implies that
\begin{align*}
\sin(\phi)&= \frac{\vg^\top \bfcalO \sign( \bfcalO^\top \vb) + \vg^\top \vxi}{ \left\|\bfcalX^\top \vs \right\|_1}\\
&\leq \frac{\left| \vg^\top \bfcalO \sign( \bfcalO^\top \vb)\right| + D}{ \left\|\bfcalX^\top \vs \right\|_1}
\end{align*} 
since $|\vg^\top \vxi|\leq \|\vg\|\|\vxi\| <D$.
Thus, we obtain Lemma \ref{lem:critical-point} after defining
\e
\eta_{\bfcalO}:=\frac{1}{M}\max_{\vb\in\setS^{D-1}}\    \left\| (\mId - \vb\vb^\top)\bfcalO \sign(\bfcalO^\top\vb)\right\|_2
\label{eq:eta O}\ee
and
\e
c_{\bfcalX,\min}:=\frac{1}{N}\min_{\vb\in\calS\cap\setS^{D-1}}\|\bfcalX^\top\vb\|_1.
\label{eq:cXmin}\ee

\begin{lem}
	Any critical point $\vb$ of \eqref{eq:dpcp} must either be a normal vector of $\calS$, or have a principal angle $\phi$ from $\calS$ smaller than or equal to $\arcsin\left(M \overline \eta_{\bfcalO}/Nc_{\bfcalX,\min}\right)$, where 
	\e
	\overline\eta_{\bfcalO}:=\eta_{\bfcalO} + D/M.
	\label{eq:etaO}\ee
	\label{lem:critical-point}
\end{lem} 


Towards interpreting Lemma \ref{lem:critical-point}, we first give some insight into the quantities $\eta_{\bfcalO}$ and $c_{\bfcalX,\min}$. First, we claim that $\eta_{\bfcalO}$ reflects how well distributed the outliers are, with smaller values corresponding to more uniform distributions. This can be seen by noting that as $M \rightarrow \infty$ and assuming that $\bfcalO$ remains well distributed, the quantity $\frac{1}{M}\bfcalO \sign\left(\bfcalO^\top \vb\right)$ tends to  the quantity $c_D \vb$, where $c_D$ is the average height of the unit hemi-sphere of $\R^D$ \citep{Tsakiris:DPCPICCV15,Tsakiris:DPCP-Arxiv17}
\e
c_D :=\frac{(D-2)!!}{(D-1)!!}\cdot\left\{\begin{matrix} \frac{2}{\pi},& D  \text{ is even} ,\\ 1, & D \text{ is old}, \end{matrix}  \right. \ \text{where} \  k!! :=\left\{\begin{matrix} k(k-2)(k-4)\cdots4\cdot2,& k \text{ is even} ,\\ k(k-2)(k-4)\cdots3 \cdot1, & k\text{ is old}. \end{matrix}  \right.
\label{eq:cD}
\ee
Since $\vg \perp \vb$, in the limit $\eta_{\bfcalO} \rightarrow 0$. Second, the quantity $c_{\bfcalX,\min}$ is the same as the {\em permeance statistic} defined in \citep{lerman2015robust}, and for well-distributed inliers is bounded away from small values, since there is no single direction in $\calS$ sufficiently orthogonal to $\bfcalX$. We thus see that according to Lemma \ref{lem:critical-point}, any critical point of \eqref{eq:dpcp} is either orthogonal to the inlier subspace $\calS$, or very close to $\calS$, with its principal angle $\phi$ from $\calS$ being smaller for well distributed points and smaller outlier to inlier ratios $M/N$. Interestingly, Lemma \ref{lem:critical-point} suggests that any algorithm can be utilized to find a normal vector to $\calS$ as long as the algorithm is guaranteed to find a critical point of \eqref{eq:dpcp} and this critical point is sufficiently far from the subspace $\calS$, i.e.,  it has principal angle larger than $\arcsin\left(M \overline \eta_{\bfcalO}/Nc_{\bfcalX,\min}\right)$. We will utilize this crucial observation in the next section to derive guarantees for convergence to the global optimum for a new scalable algorithm.

We now compare  \Cref{lem:critical-point} with the result in \cite[Theorem 1]{Maunu:arXiv17}.  In the case when the subspace is a hyperplane, \Cref{lem:critical-point} and \cite[Theorem 1]{Maunu:arXiv17} share similarities and differences. For comparison, we interpret the results in \cite[Theorem 1]{Maunu:arXiv17} for the DPCP problem. On one hand, both \Cref{lem:critical-point} and \cite[Theorem 1]{Maunu:arXiv17} attempt to characterize certian behaviors of the objective function when $\vb$ is away from the subspace by looking at the first-order information. On the other hand, we obtain \Cref{lem:critical-point}  by directly considering the Riemannian subdifferentional and proving that any Riemannian subgradient is not zero when $\vb$ is  away from the subspace  $\calS$ but not its normal vector. While \cite[Theorem 1]{Maunu:arXiv17} is obtained by checking a particular (directional) geodesic subderivatrive and showing it is negative\footnote{There is a subtle issue for the optimality condition by only checking a particular subderivative. This issue can be solved by checking either all the elements in the (directional) geodesic subdifferentional or (directional) geodesic directional derivative. In particular, under the general assumption of the data points as utilized in this paper, this issue can be mitigated by adding an additional term (such as the difference between $\eta_{\bfcalO}$ and $\overline \eta_{\bfcalO}$) into \cite[Theorem 1]{Maunu:arXiv17}.}. These two approaches also lead to different quantities utilized for 
capturing the region in which there is no critical point or local minimum. Particularly, with $
	a_\bfcalX = \lambda_d\left(\sum_{i=1}^N  \frac{\vx_i\vx_i^\T}{\|\vx_i\|} \right),$
let 
\begin{align*}
&\Omega_1 : = \left\{ \vb\in\setS^{D-1}: \arcsin\left(\overline\eta_{\bfcalO}/c_{\bfcalX,\min}\right)<\phi<\frac{\pi}{2} \right\}, \\
 &\Omega_2 : = \left\{ \vb\in\setS^{D-1}: \arcsin\left(\overline\eta_{\bfcalO}/a_{\bfcalX}\right)<\phi<\frac{\pi}{2} \right\},\end{align*}
which are the regions in which there is no critical point and no local minimum as claimed  in \Cref{lem:critical-point} and \cite[Theorem 1]{Maunu:arXiv17}, respectively\footnote{\cite[Theorem 1]{Maunu:arXiv17} utilizes a different quantity, which we prove is equivalent to $\eta_{\bfcalO}$ for the hyperplance case.}. We note that $c_{\bfcalX,\min}$ is larger than $a_{\bfcalX}$,\footnote{This can be seen as follows
\begin{align*}
	a_\bfcalX &= \lambda_d\left(\sum_{i=1}^N  \frac{\vx_i\vx_i^\T}{\|\vx_i\|} \right) = \min_{\vb\in\calS\cup \setS^{D-1}} \sum_{i=1}^N \vb^{\T} \frac{\vx_i\vx_i^\T}{\|\vx_i\|}\vb \\&= \min_{\vb\in\calS\cup \setS^{D-1}} \sum_{i=1}^N  \frac{(\vx_i^\top\vb)^2}{\|\vx_i\|} = \min_{\vb\in\calS\cup \setS^{D-1}} \sum_{i=1}^N  |\vx_i^\top\vb| \frac{|\vx_i^\top\vb|}{\|\vx_i\|}\leq \min_{\vb\in\calS\cup \setS^{D-1}} \sum_{i=1}^N  |\vx_i^\top\vb| = c_{\bfcalX,\min},
\end{align*}
where the second equality follows from the min-max theorem.} indicating that the region $\Omega_1$ is larger than $\Omega_2$. Also, under a probability setting, we provide a much tighter upper bound for $\overline\eta_{\bfcalO}$, i.e., $O(\sqrt{M})$ versus $\sqrt{M}\|\bfcalO\|_2$ (which roughtly scales as $O(M)$) in  \citep{Maunu:arXiv17}. Consequently our result leads to a much better bound on the number $M$ of outliers that can be tolerated scales as a function of the number $N$ of inliers.

We finally note that \cite[Theorem 1]{Maunu:arXiv17} also covers the case where the subspace has higher co-dimension. We leave the extension of \Cref{lem:critical-point} for multiple normal vectors as future work.

\subsection{Global Optimality} 
In order to characterize the global solutions of \eqref{eq:dpcp}, we define quantities similar to $c_{\bfcalX,\min}$ but associated with the outliers, namely
\e
c_{\bfcalO,\min}:=\frac{1}{M}\min_{\vb\in\setS^{D-1}}\|\bfcalO^\top\vb\|_1
\ \ \text{and} \ \ 
c_{\bfcalO,\max}:=\frac{1}{M}\max_{\vb\in\setS^{D-1}}\|\bfcalO^\top\vb\|_1.
\label{eq:cO terms}\ee
The next theorem, whose proof relies on \Cref{lem:critical-point}, provides new deterministic conditions under which any global solution to \eqref{eq:dpcp} must be a normal vector to $\calS$.
\begin{thm}\label{thm:DeterministicGlobal}
	Any global solution $\vb^{\star}$ to \eqref{eq:dpcp} must be orthogonal to the inlier subspace $\calS$ as long as
	\e
	\frac{M}{N} \cdot \frac{\sqrt{\overline\eta_{\bfcalO}^2 + \left( c_{\bfcalO,\max} - c_{\bfcalO,\min}  \right)^2   }}{c_{\bfcalX,\min}} <1.
	\label{eq:condition-for-global-as-a-normal}\ee
\end{thm}
The proof of \Cref{thm:DeterministicGlobal} is given in \Cref{sec:prf DeterministicGlobal}. Towards interpreting Theorem \ref{thm:DeterministicGlobal}, recall that for well distributed inliers and outliers $\overline\eta_{\bfcalO}$ is small, while 
the permeance statistics $c_{\bfcalO,\max},\, c_{\bfcalO,\min}$ are bounded away from small values. Now, the quantity $c_{\bfcalO,\max}$, thought of as a \emph{dual} permeance statistic, is bounded away from large values for the reason that there is not a single direction in $\R^D$ that can sufficiently capture the distribution of $\bfcalO$. In fact, as $M$ increases the two quantities $c_{\bfcalO,\max},\, c_{\bfcalO,\min}$ tend to each other and their difference goes to zero as $M \rightarrow \infty$. With these insights, Theorem \ref{thm:DeterministicGlobal} implies that regardless of the outlier/inlier ratio $M/N$, as we have more and more inliers and outliers while keeping $D$ and $M/N$ fixed, and assuming the points are well-distributed, condition \eqref{eq:condition-for-global-as-a-normal} will eventually be satisfied and any global minimizer must be orthogonal to the inlier subspace $\calS$.

\begin{figure}[htb!]
	\begin{subfigure}{0.32\linewidth}
	\centering	\includegraphics[width=1.8in]{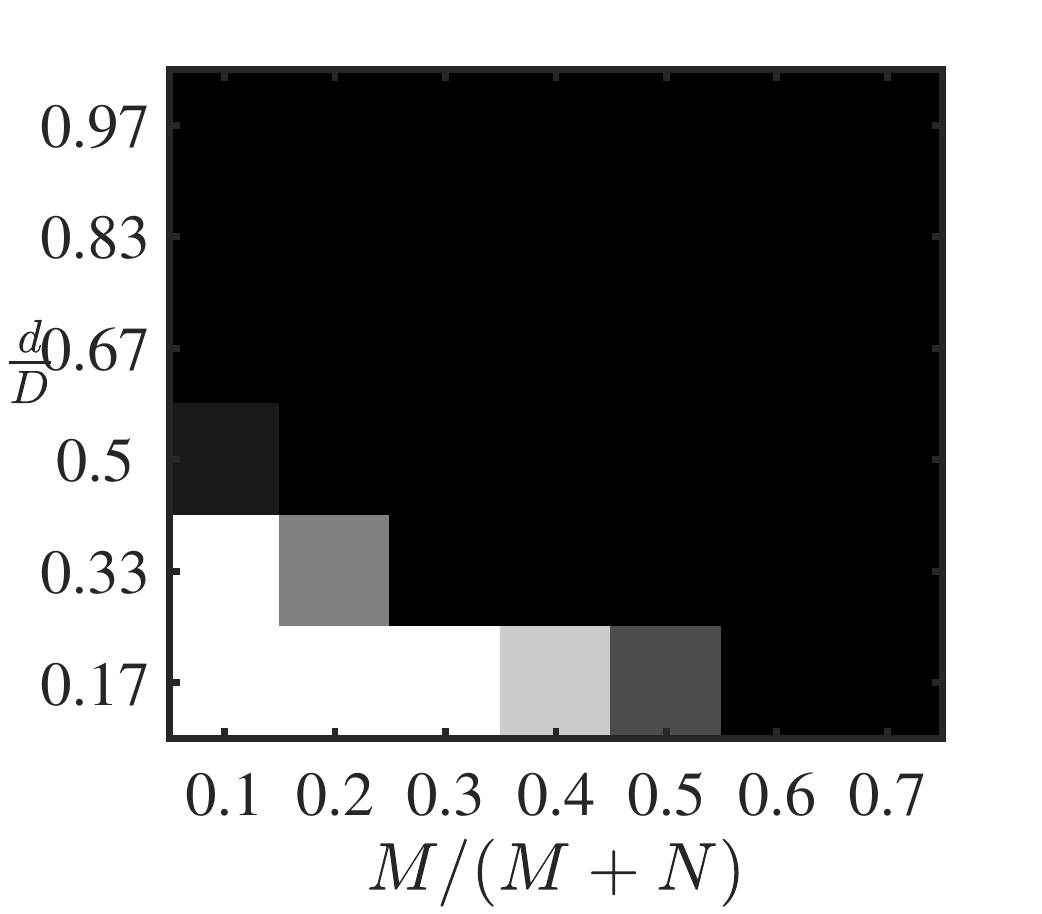}
		\caption{Check \cite[(24)]{Tsakiris:DPCPICCV15} with $N = 500$}
	\end{subfigure}
	\hfill
	\begin{subfigure}{0.32\linewidth}
	\centering	\includegraphics[width=1.8in]{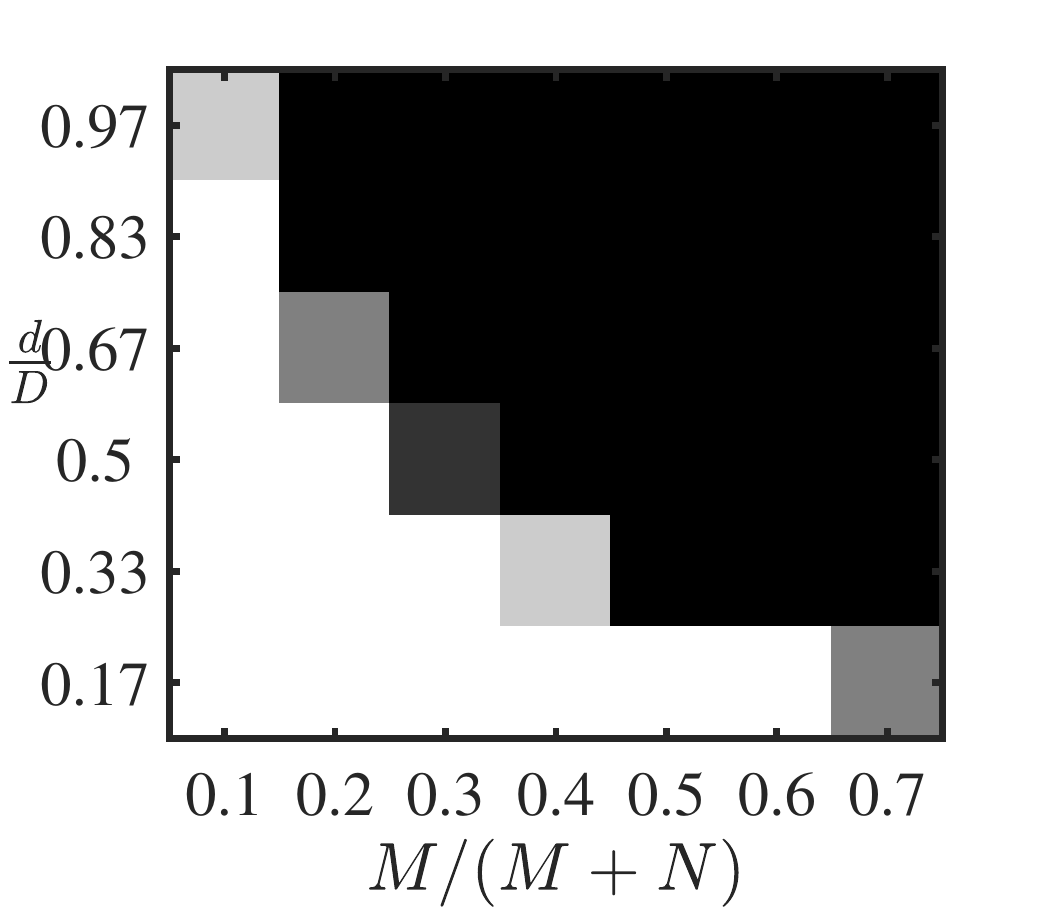}
		\caption{Check \cite[(24)]{Tsakiris:DPCPICCV15} with $N = 1000$}
	\end{subfigure}
	\hfill
	\begin{subfigure}{0.32\linewidth}
	\centering	\includegraphics[width=1.8in]{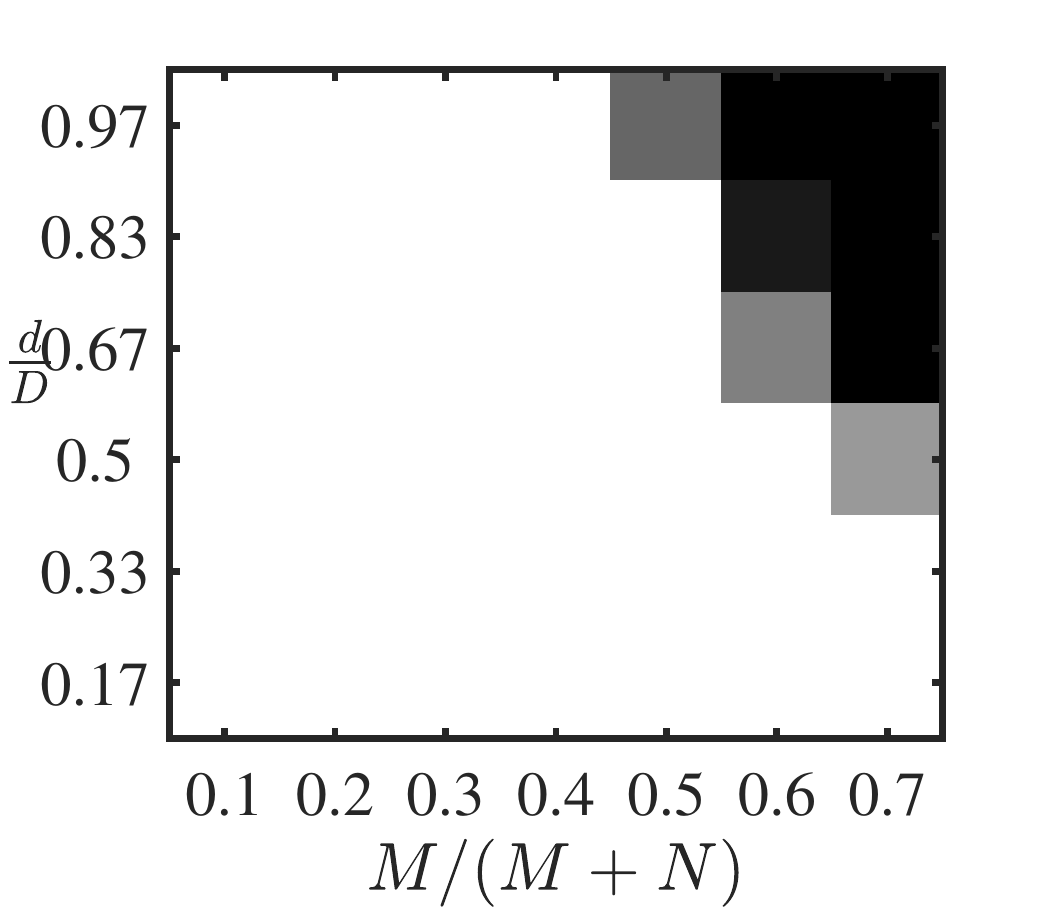}
		\caption{Check \cite[(24)]{Tsakiris:DPCPICCV15} with $N = 3000$}
	\end{subfigure}
	
	\begin{subfigure}{0.32\linewidth}
	\centering	\includegraphics[width=1.8in]{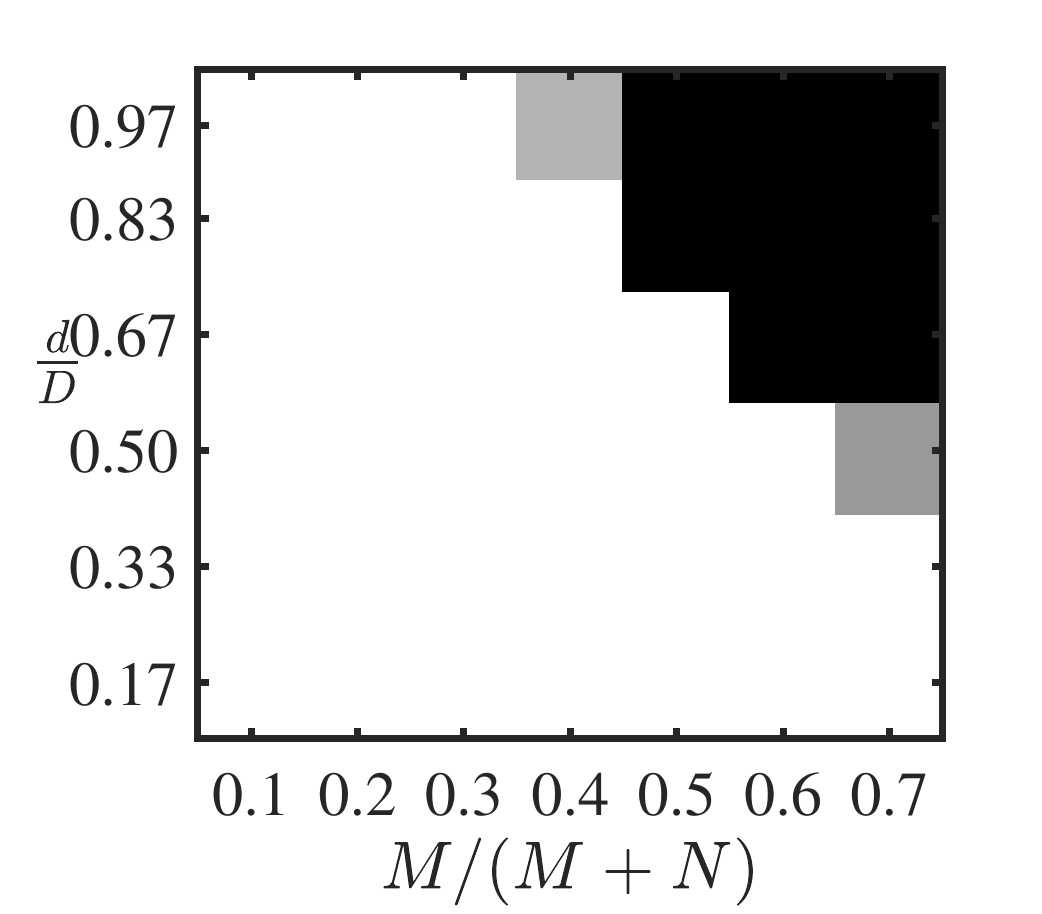}
		\caption{Check \eqref{eq:condition-for-global-as-a-normal} with $N = 500$}
	\end{subfigure}
	\hfill
	\begin{subfigure}{0.32\linewidth}
\centering		\includegraphics[width=1.8in]{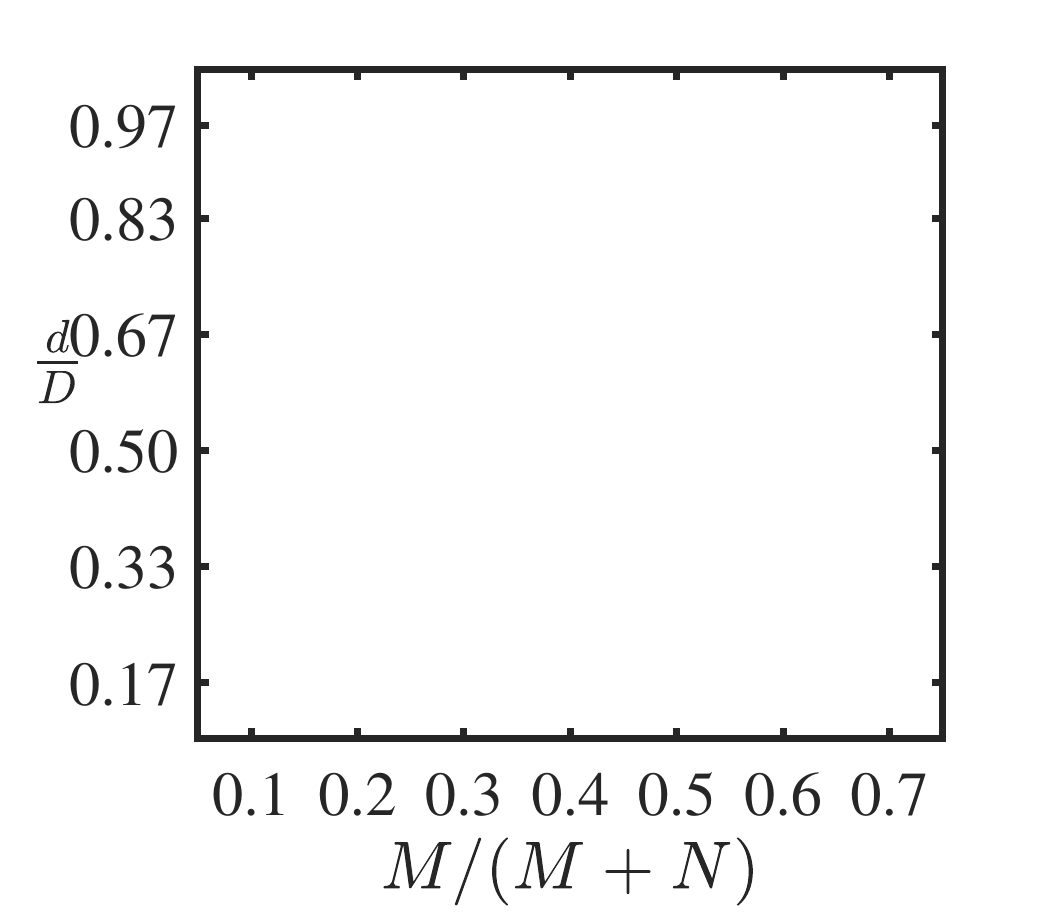}
		\caption{Check \eqref{eq:condition-for-global-as-a-normal} with $N = 1000$}
	\end{subfigure}
	\hfill
\begin{subfigure}{0.32\linewidth}
\centering	\includegraphics[width=1.8in]{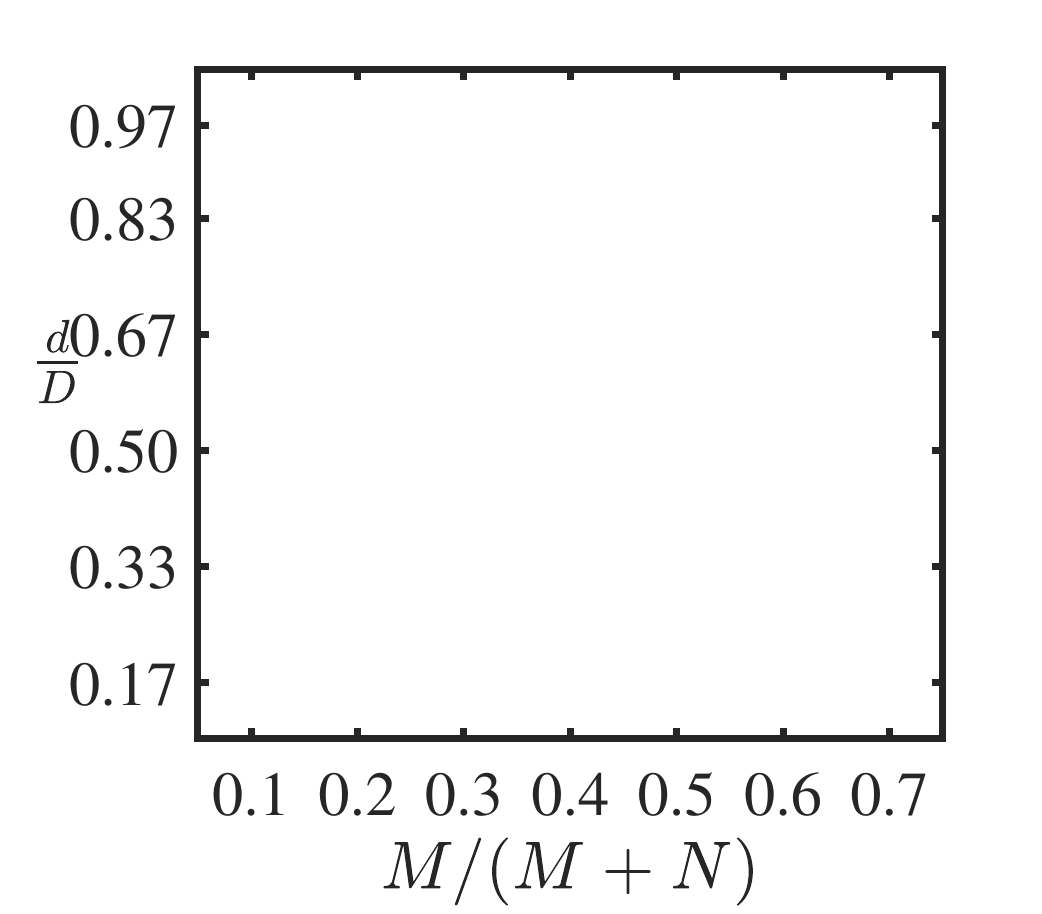}
	\caption{Check \eqref{eq:condition-for-global-as-a-normal} with $N = 3000$}
\end{subfigure}
	\caption{\small Check whether the condition \eqref{eq:condition-for-global-as-a-normal} and a similar condition in \cite[Theorem 2]{Tsakiris:DPCPICCV15} are satisfied (white) or not (black) for a fixed number $N$ of inliers while varying the outlier ration $M/(M+N)$ and the subspace relative dimension $d/D$. The inliers and outliers are generated in the same way as in \Cref{fig:check-ALP-N500}.}\label{fig:checkGlobalConditions}
\end{figure}
We note that a similar condition to \eqref{eq:condition-for-global-as-a-normal} is also given in \cite[Theorem 2]{Tsakiris:DPCPICCV15}. Although the proofs of the two theorems share some common elements, \cite[Theorem 2]{Tsakiris:DPCPICCV15} is derived by establishing discrepancy bounds between \eqref{eq:dpcp} and a  \emph{continuous} analogue of \eqref{eq:dpcp}, and involves
quantities difficult to handle such as \emph{spherical cap discrepancies} and circumradii of \emph{zonotopes}. In addition, as shown in \Cref{fig:checkGlobalConditions}, a numerical comparison of the conditions of the two theorems  reveals that condition \eqref{eq:condition-for-global-as-a-normal} is much tighter. We attribute this to the quantities in our new analysis better representing the function $\|\widetilde {\mathbfcal X}^\top \vb\|_1$ being minimized, namely $c_{\bfcalX,\min}$, $c_{\bfcalO,\min}$, $c_{\bfcalO,\max}$, and $\overline\eta_{\bfcalO}$,
when compared to the quantities used in the analysis of \citep{Tsakiris:DPCPICCV15,Tsakiris:DPCP-Arxiv17}. Moreover, our quantities are easier to bound under a probabilistic model, thus leading to the following characterization of the number of outliers that may be tolerated.

\begin{thm}\label{thm:global-opt-random-model} Consider a random spherical model where the columns of $\bfcalO$ and $\bfcalX$ are drawn independently and uniformly at random from the unit sphere $\setS^{D-1}$ and the intersection of the unit sphere with a subspace $\calS$ of dimension $d<D$, respectively. Then for any positive scalar $t<2(c_d\sqrt{N} -2)$, with probability at least $1-6e^{-{t^2}/{2}}$, any global solution of \eqref{eq:dpcp} is orthogonal to $\calS$ as long as
	\begin{align}
(4+t)^2M + C_0\left(\sqrt{D}\log  D + t\right)^2 M  \leq \Big(  c_d N - (2+ {t}/{2}) \sqrt{N} \Big)^2,\label{eq:global opt random model}\end{align}
	where $C_0$ is a universal constant that is indepedent of $N,M,D, d$ and $t$, and $c_d$ is defined in \eqref{eq:cD}.
\end{thm}
The proof of \Cref{thm:global-opt-random-model} is given in \Cref{sec:prf global-opt-random-model}. To interpret \eqref{eq:global opt random model}, first note that $ \sqrt{\frac{2}{\pi d}}\leq c_d \leq \sqrt{\frac{1}{d}}$ for any $d\in\N$.\footnote{We show it by induction. For the LHS, first note that $c_d \geq \sqrt{\frac{2}{\pi d}}$ holds for $d =1$ and $d =2$. Now suppose it holds for any $d \geq 2$ and we show it is also true for $d+1$. Towards that end, by the defintion of \eqref{eq:cD}, we have $
c_{d+1} = \frac{1}{d c_d}\frac{2}{\pi} \geq \frac{2}{d\pi}\sqrt{\frac{\pi d}{2}} = \sqrt{\frac{2}{\pi d}} \geq \sqrt{\frac{2}{\pi (d+1)}}. $
Thus, by induction, we have $c_d \geq \sqrt{\frac{2}{\pi d}}$ for any $d\in\N$. Similarly, for the RHS, $c_d\leq \sqrt{\frac{1}{d}}$ holds for any $d =1$ and $d=2$. Now suppose it holds for any $d \geq 2$ and we show it is also true for $d+1$. Towards that end, by the defintion of \eqref{eq:cD}, we have $
c_{d+1} = \frac{1}{d c_d}\frac{2}{\pi} \leq \sqrt{\frac{1}{d}}\frac{2}{\pi}  \leq \sqrt{\frac{1}{d+1}}. $
Thus, by induction, we have $c_d \geq \sqrt{\frac{2}{\pi d}}$ for any $d\in\N$.
} As a consequence, \eqref{eq:global opt random model} implies that at least $N>2\frac{1}{c_d^2}\geq 2d$.
More interestingly, according to Theorem \ref{thm:global-opt-random-model} DPCP can tolerate $M = O(\frac{1}{dD\log^2D}N^2)$ outliers, and particularly $M=O(N^2)$ for fixed $D$ and $d$.  We note that the universal constant $C_0$ comes from [~\citep{Van1998:asymptotic}, Cor. 19.35] which is utilized to bound the supreme of an {\em empirical process} related to our quantity $\eta_{\bfcalO}$. However, it is possible to get rid of this constant or obtain an explit expression of the constant by utilizing a different approach to interpret $\eta_{\bfcalO}$ in the random model. With a different approach for $\eta_{\bfcalO}$, we also believe it is possible to improve the bound with respect to $D$ for \eqref{eq:global opt random model}. In particular, we expect that an alternative bound for $\eta_{\bfcalO}$ improves the condition for the success of DPCP up to $M = O(\frac{D}{d}N^2)$. This topic will be the subject of future work.

As corollaries of \Cref{thm:global-opt-random-model}, the following two results further establish the global optimality condition for the random spherical model in the cases of high-dimensional subspace and large-scale data points.
\begin{cor}[high-dimensional subspace]\label{cor:global-opt-random-model} Consider the same random spherical model as in \Cref{thm:global-opt-random-model}. Then for any positive scalar $t< \frac{4(c_d\sqrt{N} -2)^2}{d}$,  with probability at least $1-6e^{-{td}/{2}}$, any global solution of \eqref{eq:dpcp} is orthogonal to $\calS$ as long as
	\begin{align}
	(4+\sqrt{td})^2M + C_0\left(\sqrt{D}\log D + \sqrt{td}\right)^2 M \leq \Big(  c_d N - (2+ \frac{\sqrt{td}}{2}) \sqrt{N} \Big)^2,\label{eq:global opt random model v2}\end{align}
	where $C_0$ is a universal constant in \Cref{thm:global-opt-random-model}.
\label{thm:global-opt-random-model v2}\end{cor}

\begin{cor}[large-scale data points] Consider the same random spherical model as in \Cref{thm:global-opt-random-model}. Then for any $\alpha\in(0,1)$ and positive scalar $t<\frac{4(c_d\sqrt{N} -2)^2}{N^\alpha}$,  with probability at least $1-6e^{-{tN^\alpha}/{2}}$, any global solution of \eqref{eq:dpcp} is orthogonal to $\calS$ as long as
\e	\begin{split}
	(4+\sqrt{tN^\alpha})^2M + C_0\left(\sqrt{D}\log D + \sqrt{tN^\alpha}\right)^2 M  &\leq \Big(  c_d N - (2+ \frac{\sqrt{tN^\alpha}}{2}) \sqrt{N} \Big)^2, \label{eq:global opt random model v3}\end{split}\ee
	where $C_0$ is a universal constant in \Cref{thm:global-opt-random-model}.
\label{thm:global-opt-random-model v3}\end{cor}
Unlike \Cref{thm:global-opt-random-model}, for fixed $t$, the failure probability in \Cref{thm:global-opt-random-model v2} or \Cref{thm:global-opt-random-model v3} decreases when $d$ or $N$ increases, respectively. On the other hand, \eqref{eq:global opt random model v2} gives a similar bound $M \lesssim \frac{1}{dD\log^2D}N^2$ while \eqref{eq:global opt random model v3} requires $M \lesssim \frac{N^2}{d\max\{D\log^2D,N^\alpha\}}$,  which maybe slightly stronger than $M \lesssim \frac{1}{dD\log^2D}N^2$ indicated by \eqref{eq:global opt random model}.

\subsection{Comparison with Existing Results}
\label{sec:comparison}

We now compare with the existing methods that are provably tolerable to the outliers. The methods are covered by a recent review in  \cite[Table I]{Lerman:ReviewArXiv18}, including the Geodesic Gradient Descent (GGD) \citep{Maunu:arXiv17}, Fast Median Subspace (FMS) \citep{Lerman:II17}, REAPER \citep{lerman2015robust}, Geometric Median Subspace (GMS) \citep{Zhang:JMLR14}, $\ell_{2,1}$-RPCA \citep{Xu:NIPS10} (which is called Outlier Pursuit (OP) in \cite[Table I]{Lerman:ReviewArXiv18}), Tyler M-Estimator (TME) \citep{Zhang:II16}, Thresholding-based Outlier Robust PCA (TORP) \citep{Cherapanamjeri:Arxiv17} and the Coherence Pursuit (CoP) \citep{Rahmani:arXiv16}. However, we note that the comparison maynot be very fair since the results summarized in \cite[Table I]{Lerman:ReviewArXiv18} are established for random Gaussian models where the columns of $\bfcalO$ and $\bfcalX$ are drawn independently and uniformly at random from the distribtuion $\calN(\vzero,\frac{1}{D}\mId)$ and $\calN(\vzero,\frac{1}{d}\mS\mS^\T)$ with $\mS\in\R^{D\times d}$ being an orthonormal basis of the inlier subspace $\calS$. Nevertheless, these two random models are closely related since each columns of $\bfcalO$ or $\bfcalX$ in the random Gaussian model is also concentrated around the sphere $\setS^{D-1}$, especially when $d$ is large.

That being said, we now review these results on the random Gaussian model. First, the global optimality condition  in \citep{lerman2015robust} indicates that with probability at least $1-3.5e^{-td}$, the inlier subpsace can be exactly recovered by solving the convex problem \eqref{eq:REAPER relaxed} (possibly with a final projection step) if 
\e
6\left(\frac{M}{D} + 1 +t\right) \leq \frac{1}{\sqrt{32\pi}}\left( \frac{N}{d} - \pi(4 + 2t) \right), \ d\leq (D-1)/2.
\label{eq:random-model-REAPER}\ee

Compared with \eqref{eq:random-model-REAPER} which requires $M\lesssim \frac{D}{d} N$,  \eqref{eq:global opt random model} requires $M\lesssim \frac{1}{dD\log^2D}N^2$.  On one hand, when the dimension $d$ and $D$ are fixed as constants, \eqref{eq:global opt random model} gives a better relationship between $M$ and $N$ than \eqref{eq:random-model-REAPER}. On the other hand,  the relationship between $M$ and $D$ given in \eqref{eq:global opt random model} is worser than the one in \eqref{eq:random-model-REAPER}.

The work of \citep{Maunu:arXiv17} establishes and interprets  a local optimality condition (which is similar to \Cref{lem:critical-point}) rather than a global one for \eqref{eq:REAPER}. Specifically, according to \cite[Theorem 5]{Maunu:arXiv17}, suppose that for some absolute constant $\widetilde C_1$, and other constants $0<\theta<\frac{\pi}{2}$, $0<a<1$ and $\widetilde  C_0>0$,
\e
\cos(\theta)\left((1-a)\sqrt{\frac{N}{d}} - \widetilde C_1\right)^2 > \left( \frac{M}{\sqrt{dD}} + \sqrt{\frac{M}{d}} + \sqrt{\frac{2M\widetilde C_0}{dD}}  \right) , \ N> \frac{\widetilde C_1^2}{(1-a)^2}d.
\label{eq:random-model-GGD}\ee
Then, with probability at least $1 - 2e^{-c_1 a^2 N} - e^{-\widetilde C_0}$ for some absolute constant $c_1$, the inlier subspace is the only local minimizer of \eqref{eq:REAPER} among all the subspaces (which is one-to-one correspondence to their orthogonal projection) that  have subspace angle at most $\theta$ to the inlier subspace $\calS$. Using this, one may establish a similar global optimality condition with the approach used in \Cref{thm:DeterministicGlobal}. Here, we instead interpret our \Cref{lem:critical-point} in the random spherical model, implying that for any positive constants $0<\theta<\frac{\pi}{2}$, if
\e
\cos(\theta)\left( c_d N  - (2+ \frac{t}{2}) \sqrt{N}\right) >  C_0\left(\sqrt{D}\log \left(\sqrt{c_D D}\right) + t\right)\sqrt{M},\ t<2(c_d\sqrt{N} -2),
\label{eq:random-model-DPCP local}\ee
then with probability at least $1-4e^{-{t^2}/{2}}$, any critical point of \eqref{eq:dpcp} that has principal angle from $\calS^\perp$ small than $\theta$ must be orthogonal to $\calS$. Here $C_0$ is a universal constant in \Cref{thm:global-opt-random-model}. Now comparing \eqref{eq:random-model-GGD} and \eqref{eq:random-model-DPCP local},  \eqref{eq:random-model-GGD} requires $M\lesssim \cos(\theta)\sqrt{\frac{D}{d}}N$, while \eqref{eq:random-model-DPCP local} needs $M \lesssim \cos^2(\theta)\frac{1}{dD\log^2D}N^2$. As we stated before, this difference mostly owes to a much tighter upper bound for $\overline\eta_{\bfcalO}$, i.e., $O(\sqrt{M})$ versus $\sqrt{M}\|\bfcalO\|_2$ (which roughtly scales as $O(M)$) in  \citep{Maunu:arXiv17}.

Finally, we summarize the exact recovery condition or global optimality condition for the existing methods that are provably  tolerable to the outliers in \Cref{table:review} which imitates  \cite[Table I]{Lerman:ReviewArXiv18}. One one hand, for fixed $d$ and $D$, \Cref{table:review} indicates that DPCP in the only method that can tolerate up to $O(N^2)$ outliers. On the other hand, the result on DPCP has suboptimal bound with respect to $D$.
 
\section{Efficient Algorithms for Dual Principal Component Pursuit}
\label{sec:algorithm}
In this section, we first review the linear programming approach proposed in \citep{Tsakiris:DPCPICCV15,Tsakiris:DPCP-Arxiv17} for solving the DPCP problem \eqref{eq:dpcp} and provide new convergence guarantee for this approach. We then provide a projected sub-gradient method which has guaranteed convergence performance and is scalable to the problem sizes as it only uses marix-vector multiplications in each iteration.

\subsection{Alternating Linerization and Projection Method}
\label{sec:ALP}
Note that the DPCP problem \eqref{eq:dpcp} involves a convex objective function and a non-convex feasible region, which nevertheless is easy to project onto. This structure was exploited in~\citep{Qu:NIPS14,Tsakiris:DPCPICCV15}, where in the second case the authors proposed an \emph{Alternating Linearization and Projection (ALP)} method that solves a sequence of linear programs (LP) with a linearization of the non-convex constraint and then projection onto the sphere. Specifically, if we denote the constraint function associated with \eqref{eq:dpcp} as $g(\vb) := (\vb^\top\vb -1)/2$, then the first order Tayor approximation of $g(\vb)$ at  any $\vw\in \setS^{D-1}$ is $g(\vb)\approx g(\vw) + \vb^\top\nabla g(\vw) = \vb^\top\vw$. With an initial guess of $\vb_0$, we compute a sequence of iterates $\{\vb_k\}_{k\geq 1}$ via the update~\citep{Tsakiris:DPCP-Arxiv17}
\e
\vb_{k}  = \argmin_{\vb\in\R^d} \|\widetilde {\mathbfcal X}^\top\vb\|_1, \ \st \ \vb^\top\widehat \vb_{k-1} = 1 \quad \text{and} \quad \widehat \vb_k = {\vb_k}/{\|\vb_k\|_2},
\label{eq:dpcp lp}\ee
where the optimization subproblem can be written as a linear program (LP) rewritting the $\ell_1$ norm in an equivalent linear form with auxiliary variables.
An alternatively view of the constraint $\vb^\top\widehat \vb_{k-1} = 1$ in \eqref{eq:dpcp lp} is that it defines an affine hyperplane which excludes the original point $\vzero$ and has $\widehat \vb_{k-1}$ as its normal vector.

The following result establishes conditions under which $\vb_1$ is orthogonal to the subspace $\calS$ and new conditions to guarantee that the sequence $\{\widehat\vb_k\}_{k\geq 0}$ converges to a normal vector of $\calS$ in a finite number of iterations.
\begin{thm}
	Consider the sequence $\{\widehat\vb_k\}_{k\geq 0}$ generated by the recursion \eqref{eq:dpcp lp}. Let $\phi_0$ be the principal angle of $\widehat \vb_0$ from $\calS$. Then,
	\begin{enumerate}[label=(\roman*)]
		\item $\vb_1$ is orthogonal to the subspace $\calS$ if
		\e
		\phi_0 > \phi_0^\natural:=\arctan\left( \frac{Mc_{\bfcalO,\max}}{Nc_{\bfcalX,\min}-M\overline\eta_{\bfcalO}}\right);
		\label{eq:angle lp succed}\ee
		\item the sequence $\{\widehat\vb_k\}_{k\geq 0}$ converges to a normal vector of $\calS$ in a finite number of iterations if
		\e
		\phi_0> \phi_0^\star:=\arccos \left(\frac{\sqrt{N^2c^2_{\bfcalX,\min} - M^2\overline\eta_{\bfcalO}^2} - M(c_{\bfcalO,\max} - c_{\bfcalO,\min} )}{Nc_{\bfcalX,\max}}\right),
		\label{eq:angle sequential lp succed}\ee
		where $
		c_{\bfcalX,\max}:=\frac{1}{N}\max_{\vb\in\calS\cap\setS^{D-1}}\|\bfcalX^\top\vb\|_1$.
	\end{enumerate}
	\label{thm:lp to global}\end{thm}

\begin{figure}[htp!]
	\centerline{
		\includegraphics[height=2in]{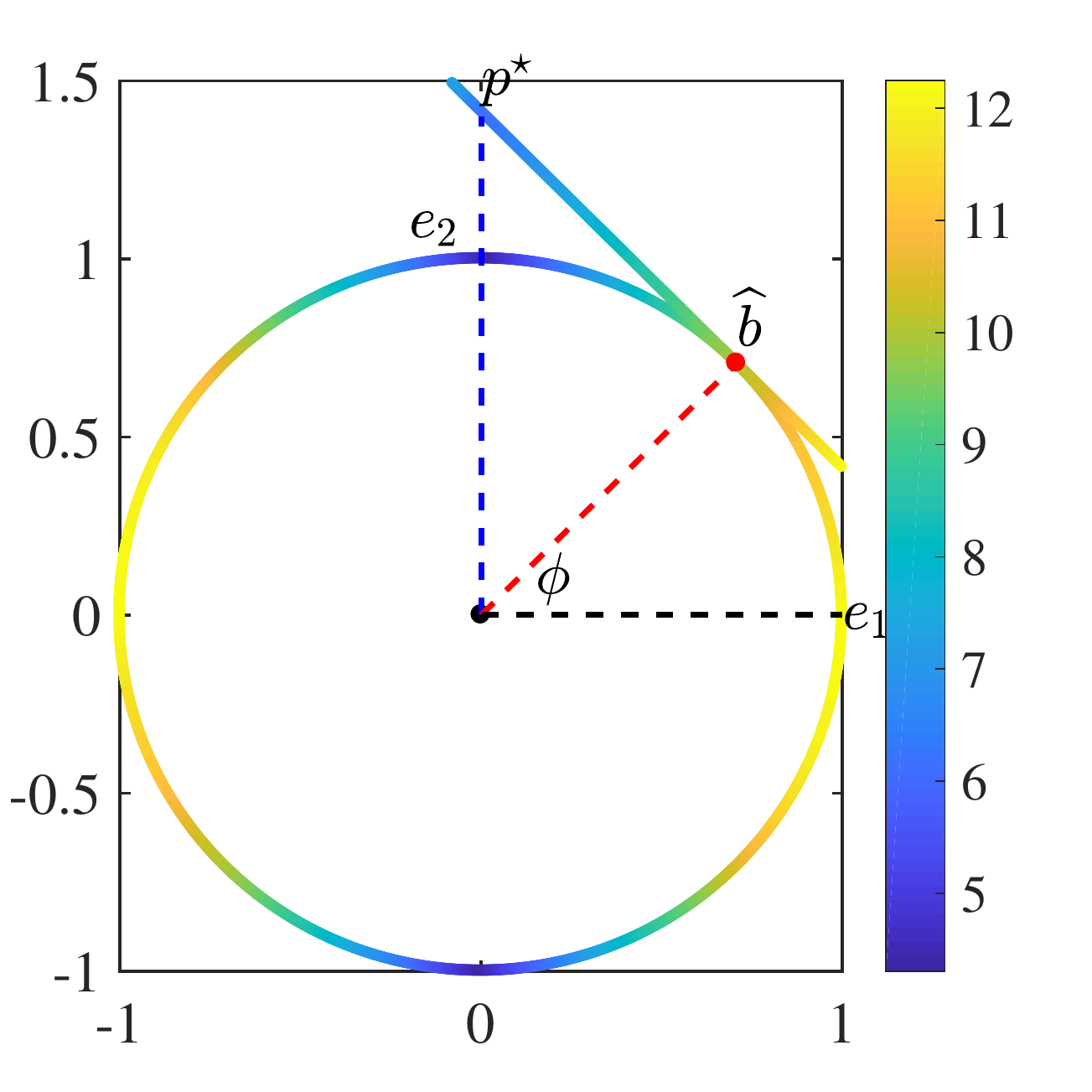}}
	\caption{ llustration of \Cref{thm:lp to global} ($i$). Here $\calS = \Span(\ve_1)$, $\ve_2$ is a normal vector and $\widehat \vb_0$ is the initialization for the proceduere \eqref{eq:dpcp lp}, where the minimization is equivalent to first projecting the circle onto the tangent line passing through $\widehat \vb$ and then finding the smallest re-scaled objective value on this line. }
	\label{fig:describe Thm-LP}
\end{figure}

The proof of \Cref{thm:lp to global} is given in \Cref{sec:prf lp to global}. First note that the expressions in \eqref{eq:angle lp succed} and \eqref{eq:angle sequential lp succed} always define angles between $0$ and $\pi/2$ when the condition \eqref{eq:condition-for-global-as-a-normal} is satisfied.
As illustrated in \Cref{fig:describe Thm-LP}, when the initialization is not very close to the subspace $\calS$, \Cref{thm:lp to global} ($i$) indicates that one procedure of \eqref{eq:dpcp lp} gives a vector that is orthogonal $\calS$, i.e., finds a hyperplance that contains all the inliers $\bfcalX$. \Cref{thm:lp to global} ($ii$) suggests that  the requirement on the principal angle of the initialization $\widehat \vb_0$ can be weakened if we are allowed to implement multiple alternating procedures in \eqref{eq:dpcp lp}. To see this, for well distributed points (both inliers and outliers),  the angle $\phi_0^\star$ defined in \eqref{eq:angle sequential lp succed} tends to zero as $M,N$ go to infitiny with $M/N$ constant; while  $\tan(\phi_0^\natural)$ defined in \eqref{eq:angle lp succed} goes to $\Omega(M/N)$. We finally note that a similar condition to \eqref{eq:angle sequential lp succed} is also given in \cite[Theorem 12]{Tsakiris:DPCP-Arxiv17}. Compared with the condition in \cite[Theorem 12]{Tsakiris:DPCP-Arxiv17}, \eqref{eq:angle sequential lp succed} defines a slightly smaller angle $\phi_0^\star$ and again more importantly, the quantities in \eqref{eq:condition-for-global-as-a-normal} is slightly tighter and amenable to probabilistic analysis.

\subsection{Projected Sub-Gradient Method for Dual Principal Component Pursuit}
Although efficient LP solvers (such as Gurobi~\citep{gurobi}) may be used to solve each LP in the ALP approach, these methods do not scale well with the problem size (i.e., $D,N$ and $M$). Inspired by \Cref{lem:critical-point}, which states that any critical point that has principal angle larger than $\arcsin\left(M{\overline \eta_{\bfcalO}}/N{c_{\bfcalX,\min}}\right)$ must be a normal vector of $\calS$, we now consider solving \eqref{eq:dpcp} with a first-order method, specifically Projected Sub-Gradient Method (DPCP-PSGM), which is stated in \Cref{alg:PSGM}.

\begin{algorithm}[htp!]
	\caption{(DPCP-PSGM) Projected Sub-gradient Method for Solving~\eqref{eq:dpcp}}
	{\bf Input:}  data $\widetilde \bfcalX\in\R^{D\times L}$ and initial step size $\mu_0$;\\
	{\bf Initialization:} choose $\widehat\vb_0$; a typical way is to set $\widehat\vb_0 = \argmin_{\vb}\|\widetilde \bfcalX^\top \vb\|_2, \ \st \ \vb\in\setS^{D-1}$;
	
	\begin{algorithmic}
		\FOR{$k=1,2,\ldots$}
		\STATE 1. update the step size $\mu_k$ according to a certain rule;
		\STATE 2. perform sub-gradient method: $\vb_k = \widehat\vb_{k-1} - \mu_k \widetilde \bfcalX \sign(\widetilde\bfcalX^\top \widehat\vb_{k-1})$; 
		\STATE 3. project onto the sphere $\setS^{D-1}$: $\widehat\vb_k = \calP_{\setS^{D-1}}\left(\vb_{k}\right) = \vb_k/\|\vb_k\|$;
		\ENDFOR
	\end{algorithmic}
	\label{alg:PSGM}
\end{algorithm}

To see why it is possible that DPCP-PSGM finds a normal vector, recall that at the $k$-th step:
\e\begin{split}
	\vb_k &= \widehat\vb_{k-1} - \mu_k \widetilde \bfcalX \sign(\widetilde\bfcalX^\top\widehat\vb_{k-1})\\
	&= \widehat\vb_{k-1} - \mu_k\bfcalX\sign(\bfcalX^\top\widehat \vb_{k-1}) - \mu_k\bfcalO\sign(\bfcalO^\top\widehat \vb_{k-1}).
\end{split}\label{eq:pgd step k}\ee
For the rest of this section, it is more convenient to use the principal angle $\theta\in[0,\frac{\pi}{2}]$ between $\vb$ and the orthogonal subspace $\calS^\perp$; thus $\vb$ is a normal vector of $\calS$ if and only if $\theta = 0$. 
Similarly let $\theta_{k}$ be the principal angle between $\widehat \vb_{k}$ and the complement $\calS^\perp$.  Suppose $\theta_{k-1}>0$, i.e., $\widehat \vb_{k-1}$ is not a normal vector to $\calS$. We rewrite $\widehat\vb_{k-1}$ as 
\e
\widehat\vb_{k-1} = \sin(\theta_{k-1})\vs_{k-1} + \cos(\theta_{k-1})\vn_{k-1},
\label{eq:decomposition-b-k}\ee
where  $\vs_{k-1} = \calP_\calS(\widehat \vb_{k-1})/\|\calP_\calS(\widehat \vb_{k-1})\|_2$ and $\vn_{k-1} = \calP_{\calS^\perp}(\widehat\vb_{k-1})/\|\calP_{\calS^\perp}(\widehat\vb_{k-1})\|_2$ are the orthonormal projections of $\widehat \vb$ onto $\calS$ and $\calS^\perp$, respectively. Since $\widehat \vb_k$ is the normalized version of $\vb_k$, they have the same principal angle to $\calS^\perp$. We now consider how the principal angle changes after we apply sub-gradient method to $\widehat \vb_{k-1}$ as in  \eqref{eq:pgd step k}. Towards that end, with \eqref{eq:decomposition-b-k},  we first  decompose the term  $\bfcalO\sign(\bfcalO^\top\widehat\vb_{k-1})$ (appeared in \eqref{eq:pgd step k}) into two parts
\e\begin{split}
	\bfcalO\sign(\bfcalO^\top\widehat \vb_{k-1}) &= \widehat\vb_{k-1}\widehat\vb_{k-1}^\top\bfcalO\sign(\bfcalO^\top\widehat \vb_{k-1}) + \left(\mId - \widehat\vb_{k-1}\widehat\vb_{k-1}^\top\right) \bfcalO\sign(\bfcalO^\top\widehat \vb_{k-1})\\
	& = \|\bfcalO^\top\widehat \vb_{k-1}\|_1\widehat\vb_{k-1} + \vq_{k-1},\\
\end{split}\label{eq:O bk term}\ee
where we define
\e
\vq_{k-1} = \left(\mId - \widehat\vb_{k-1}\widehat\vb_{k-1}^\top\right) \bfcalO\sign(\bfcalO^\top\widehat \vb_{k-1}),
\label{eq:term q}
\nonumber\ee
which is the Riemanian subgradient of $\|\bfcalO^\top \vb\|_1$ at $\vb_{k-1}$.
Note that $\vq_{k-1}$ is expected to be small as it is bounded above by $\eta_{\bfcalO}$ (defined in \eqref{eq:eta O}).

Similarly, for the other term $\bfcalX\sign(\bfcalX^\top\widehat\vb_{k-1})$, we decompose it as
\e\begin{split}
\bfcalX\sign(\bfcalX^\top\widehat\vb_{k-1})  &= \bfcalX\sign(\bfcalX^\top\vs_{k-1}) \\
&=  \vs_{k-1}\vs_{k-1}^\top\left(\bfcalX\sign(\bfcalX^\top\vs_{k-1})\right) + (\calP_{\calS} - \vs_{k-1}\vs_{k-1}^\T)\left(\bfcalX\sign(\bfcalX^\top\vs_{k-1})\right)\\
&= \|\bfcalX^\top\vs_{k-1}\|_1\vs_{k-1} + \vp_{k-1},
\end{split}\label{eq:X bk term}
\ee
where the first equality follows because $\vn_{k-1}$ is orthogonal to the inliers $\bfcalX$, and in the last equality we define
\e
\vp_{k-1}=\left(\calP_{\calS} - \vs_{k-1}\vs_{k-1}^\T\right)\left(\bfcalX\sign(\bfcalX^\top\vs_{k-1})\right).
\label{eq:term p}
\nonumber\ee
To bound $\vp_{k-1}$, we also need a quantity similar to $\eta_{\bfcalO}$ that quantifies how well the inliers are distributed within the subspace $\calS$:
\e \eta_{\bfcalX}:=\frac{1}{N}\max_{\vb\in\calS\cap\setS^{D-1}}  \left\| \left(\calP_{\calS} - \vb\vb^\T\right)\left(\bfcalX \sign(\bfcalX^\top\vb\right)\right\|.
\label{eq:eta X}\ee
Similar to the disccusion for $\eta_{\bfcalO}$ after \Cref{lem:critical-point},  if we let $N \rightarrow \infty$ and assume that $\bfcalX$ remains well distributed, then the quantity $\frac{1}{N}\bfcalX \sign\left(\bfcalX^\top \vb\right)$ tends to the quantity $c_d \vb$ (where $c_d$ is the average height of the unit hemi-sphere of $\R^d$ defined in \eqref{eq:cD}) \citep{Tsakiris:DPCPICCV15,Tsakiris:DPCP-Arxiv17}. Since $\vg \perp \vb$, in the limit $\eta_{\bfcalX} \rightarrow 0$. By the definition of $\eta_{\bfcalX}$ in \eqref{eq:eta X}, we  have $\frac{1}{N}\|\vp_{k-1}\|\leq \eta_{\bfcalX}$.

Now plugging \eqref{eq:O bk term} and \eqref{eq:X bk term} into \eqref{eq:pgd step k} gives
\e
\begin{split}
	\vb_k & = \widehat\vb_{k-1} - \mu_k(\bfcalX\sign(\bfcalX^\top\widehat\vb_{k-1}) + \bfcalO\sign(\bfcalO^\top\widehat \vb_{k-1}))\\
	& = \widehat \vb_{k-1} -\mu_k (\|\bfcalX^\top\vs_{k-1}\|_1\vs_{k-1} + \vp_{k-1} + \|\bfcalO^\top\widehat \vb_{k-1}\|_1\widehat\vb_{k-1} + \vq_{k-1})\\
	& = (1 - \mu_k\|\bfcalO^\top\widehat \vb_{k-1}\|_1)\widehat \vb_{k-1} - \mu_k \|\bfcalX^\top\vs_{k-1}\|_1 \vs_{k-1} - \mu_k (\vp_{k-1} + \vq_{k-1}).
\end{split}
\label{eq:bk useful}\ee
First suppose $\vp_{k-1} = \vq_{k-1} = \vzero$ and consider the simple case $\vb_k = (1 - \mu_k\|\bfcalO^\top\widehat \vb_{k-1}\|_1)\widehat \vb_{k-1} - \mu_k \|\bfcalX^\top\vs_{k-1}\|_1 \vs_{k-1}$. As illustrated in \Cref{fig:PSGM_angle_decay 1}, in this case, $(1 - \mu_k\|\bfcalO^\top\widehat \vb_{k-1}\|_1)\widehat \vb_{k-1}$ and $\mu_k \|\bfcalX^\top\vs_{k-1}\|_1 \vs_{k-1}$ are the scaled version of $\widehat \vb_{k-1}$ and $\vs_{k-1}$ (which is the projection of $\widehat \vb_{k-1}$ onto the subspace $\calS$), respectively. Thus, as long as $\mu_k$ is not too large in the sense that $(1 - \mu_k \|\bfcalO^\top\widehat \vb_{k-1}\|_1)\sin(\theta_{k-1}) >\mu_k \|\bfcalX^\top\vs_{k-1}\|_1$,  the principal angle of $\vb_{k}$ satisfies 
\e\begin{split}
	\tan(\theta_k) &= \frac{(1 - \mu_k \|\bfcalO^\top\widehat \vb_{k-1}\|_1)\sin(\theta_{k-1}) -\mu_k \|\bfcalX^\top\vs_{k-1}\|_1}{(1 - \mu_k \|\bfcalO^\top\widehat \vb_{k-1}\|_1)\cos(\theta_{k-1})} \\
	&<  \frac{(1 - \mu_k \|\bfcalO^\top\widehat \vb_{k-1}\|_1)\sin(\theta_{k-1}) }{(1 - \mu_k \|\bfcalO^\top\widehat \vb_{k-1}\|_1)\cos(\theta_{k-1})} \\
	&=  \tan(\theta_{k-1}).
\end{split}
\label{eq:angle decay simple case}\ee

When $\vp_{k-1}$ and $\vq_{k-1}$ are not zero but small,  the principal angle $\theta_k$ is also expected to be smaller than $\theta_{k-1}$ as long as the step size $\mu_k$ is not too large. On the other hand, for both cases, when the step size $\mu_k$ is relatively large compared to $\theta_{k-1}$ in the sense that $(1 - \mu_k \|\bfcalO^\top\widehat \vb_{k-1}\|_1)\sin(\theta_{k-1}) <\mu_k \|\bfcalX^\top\vs_{k-1}\|_1$, it is difficult (or impossible) to show the decay of the principal angle $\theta_k$. Instead, we will provide upper bound for the principal angle $\theta_k$.

\begin{figure}
		\centerline{
			\includegraphics[width=3.0in]{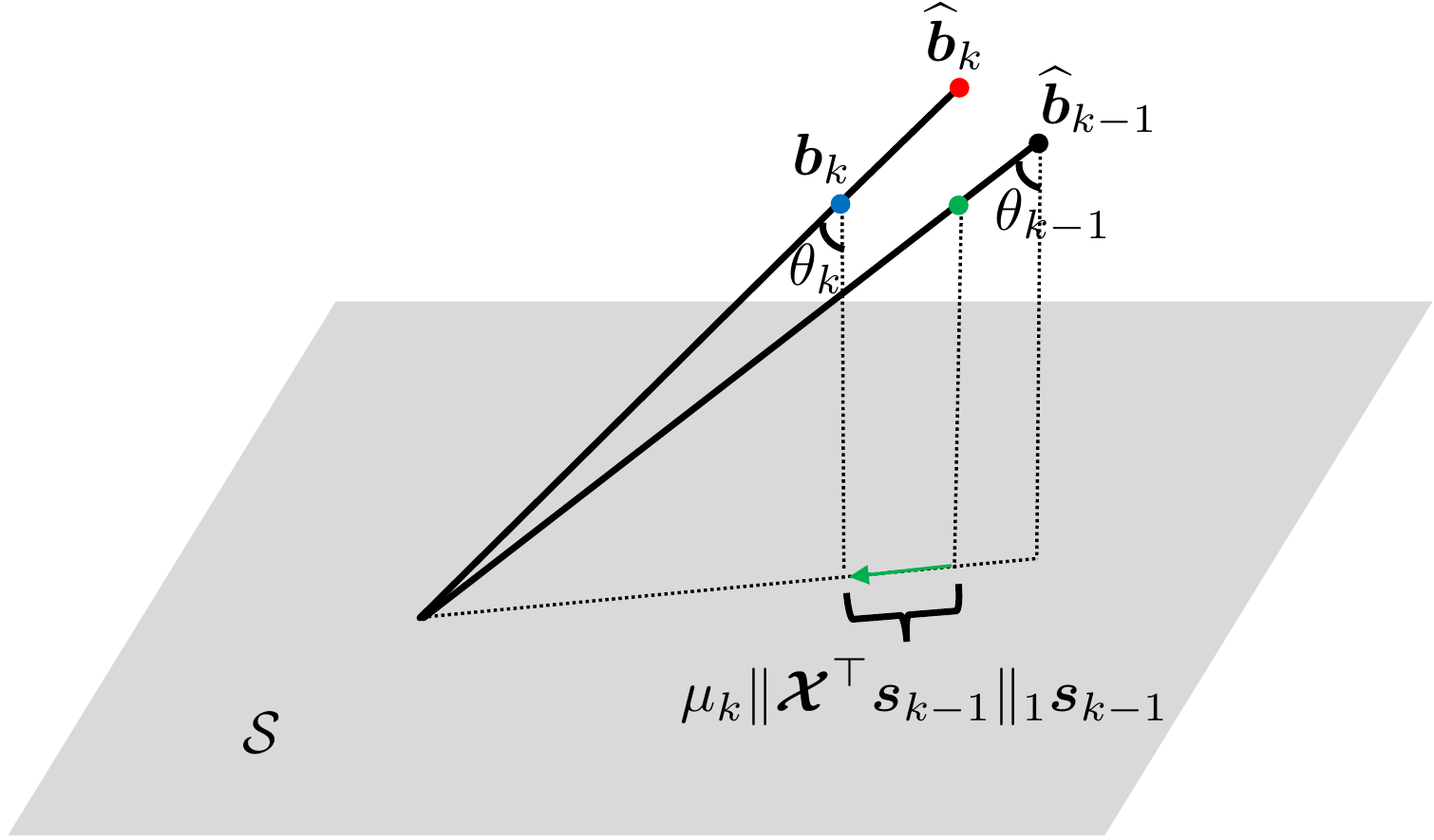}}
	\caption{Illustration of \eqref{eq:angle decay simple case}  in the ideal case when $\vp_{k-1} = \vzero$ and $\vq_{k-1} = \vzero$. Here $(1 - \mu_k\|\bfcalO^\top\widehat \vb_{k-1}\|_1)\widehat \vb_{k-1}$ (the green dot) is the shrinkaged version of $\widehat \vb_{k-1}$ (the black dot), while $\mu_k \|\bfcalX^\top\vs_{k-1}\|_1 \vs_{k-1}$ (green line) is a scaled version $\vs_{k-1}$ (which is the projection of $\widehat \vb_{k-1}$ onto the subspace $\calS$). As long as $\mu_k$ is not too large, $\vb_k = (1 - \mu_k\|\bfcalO^\top\widehat \vb_{k-1}\|_1)\widehat \vb_{k-1} - \mu_k \|\bfcalX^\top\vs_{k-1}\|_1 \vs_{k-1}$ (the blue dot) has smaller principal angle than $\widehat \vb_{k-1}$ (i.e., $\theta_{k-1}<\theta_k$) and $\widehat \vb_k$ (the red dot) is the normalized version of $\vb_k$. 
	}
	\label{fig:PSGM_angle_decay 1}
\end{figure}

Th above analysis also reveals one fact that unlike gradient descent  for smooth problems, the choice of step size for PSGM is more complicated since a constant step size in general can not guarantee the convergence of PSGM even to a critical point, though such a choice is often used in practice. For the purpose of illustration, consider a simple example $h(x) = |x|$ without any constraint, and suppose that $\mu_k = 0.08$ for all $k$ and that an initialization of $x_0 = 0.1$ is used. Then, the iterates $\{x_k\}$ will jump between two points $0.02$ and $-0.06$ and never converge to the global minimum $0$. Thus, a widely adopted strategy is to use diminishing step sizes, including those that are not summable  (such as $\mu_k=O(1/k)$ or $\mu_k = O(1/\sqrt{k})$) \cite{Boyd:2003subgradient}, or geometrically diminishing (such as $\mu_k =O(\rho^k), \rho<1$)~\citep{Goffin1977:convergence,Davis:2018subgradient,Li:RMSarXiv18}. However, for such choices,  most of the literature establishes convergence guarantees for PSGM in the context of convex feasible regions~\citep{Boyd:2003subgradient,Goffin1977:convergence,Davis:2018subgradient}, and thus can not be directly applied to  \Cref{alg:PSGM}.

Our next result provides performance guarantees for \Cref{alg:PSGM} for various choices of step sizes ranging from constant to geometrically diminishing step sizes, the latter one giving an  {\em R-linear convergence} of the sequence of principal angles to zero.

\begin{thm}[Convergence guarantee for PSGM]
	Let $\{\widehat \vb_k\}$ be the sequence generated by \Cref{alg:PSGM} with initialization $\widehat \vb_0$, whose principal angle $\theta_0$ to $\calS^\perp$ is assumed to satisfy
	\e
	\theta_0  < \arctan\left(\frac{Nc_{\bfcalX,\min}}{{N\eta_{\bfcalX} + M\eta_{\bfcalO}}}\right).
	\label{eq:initialization of PSGM v0}\ee
	Also assume that
	\e
	Nc_{\bfcalX,\min} \geq N\eta_{\bfcalX} + M\eta_{\bfcalO}
	\label{eq:condition for PSGM v0}.
	\nonumber	\ee
	Let 
	\[\mu' :=\frac{1}{4\cdot\max\{Nc_{\bfcalX,\min},Mc_{\bfcalO,\max} \}}.
	\]
	Then the angle $\theta_k$ between $\widehat\vb_k$ and $\calS^\perp$ satisfies the following properties in accordance with various choices of step sizes.
	
	\begin{enumerate}[leftmargin=0.5cm,label=(\roman*)]
		\item (constant step size) With $\mu_k=\mu \leq \mu', \ \forall k\geq 0$, we have
		\e
		\theta_k \leq \left\{\begin{matrix}\max\{\theta_0,\theta^\diamond(\mu)\}, & k< K^\diamond(\mu), \\ \theta^\diamond(\mu), & k\geq  K^\diamond(\mu),\end{matrix} \right.
		\label{eq:theta small constant stepsize}\ee
		
		where 
		\e
		K^\diamond(\mu) :=\frac{\tan(\theta_0)}{ \mu \left(Nc_{\bfcalX,\min} - \max\left\{1,\tan(\theta_0)\right\}(N\eta_{\bfcalX} + M\eta_{\bfcalO})\right)}
		\label{eq:K diamond}	\ee and 
		\e
		\theta^\diamond(\mu):=\arctan\left(\frac{\mu}{\sqrt{2}\mu'}\right). 
		\nonumber	\ee
		
		\item (diminishing step size) With $\mu_k  \leq \mu',  \mu_k \rightarrow0, \sum_{k=1} ^\infty \mu_k = \infty$, we have $\theta_k \rightarrow 0$.
		
		\item (diminishing step size of $O(1/k)$) With $\mu_0  \leq \mu',  \mu_k = \frac{\mu_0}{k}, \forall k\geq 1$, we have $ \tan(\theta_k)  = O(\frac{1}{k}).$
		\item (piecewise geometrically diminishing  step size) With $
		\mu_0 \leq \mu' $ and 
		\e
		\mu_k  = \left\{\begin{matrix}\mu_0, & k< K_0, \\ \mu_0 \beta^{\lfloor (k-K_0)/K \rfloor+1},  & \ k \geq K_0,\end{matrix}\right.
		\label{eq:diminishing step size} \ee
		where $\beta<1$, $\lfloor \cdot \rfloor$ is the floor function,
		and $K_0,K\in \N$ are chosen such that
		\e\begin{split}
			&K_0\geq K^\diamond(\mu_0), \\
			&K  \geq \Big(\sqrt{2}\beta\mu'\Big(Nc_{\bfcalX,\min} - (N\eta_{\bfcalX} + M\eta_{\bfcalO})\Big)\Big)^{-1},
		\end{split}\label{eq:set K PSGM piecewise constant}\ee where $K^\diamond(\mu)$ is defined in \eqref{eq:K diamond}, we have
		\e
		\tan(\theta_k) \leq \left\{\begin{matrix}\max\{\tan(\theta_0),\frac{\mu_0}{\sqrt{2}\mu'}\}, & k< K_0, \\ \frac{\mu_0}{\sqrt{2}\mu'}\beta^{\lfloor(k-K_0)/K\rfloor}, & k\geq  K_0.\end{matrix} \right.
		\label{eq:convergence PSGM piecewise constant}\ee
	\end{enumerate}
	\label{thm:convergence PSGM all}\end{thm}

\begin{figure}
	\centerline{
		\includegraphics[height=2.5in]{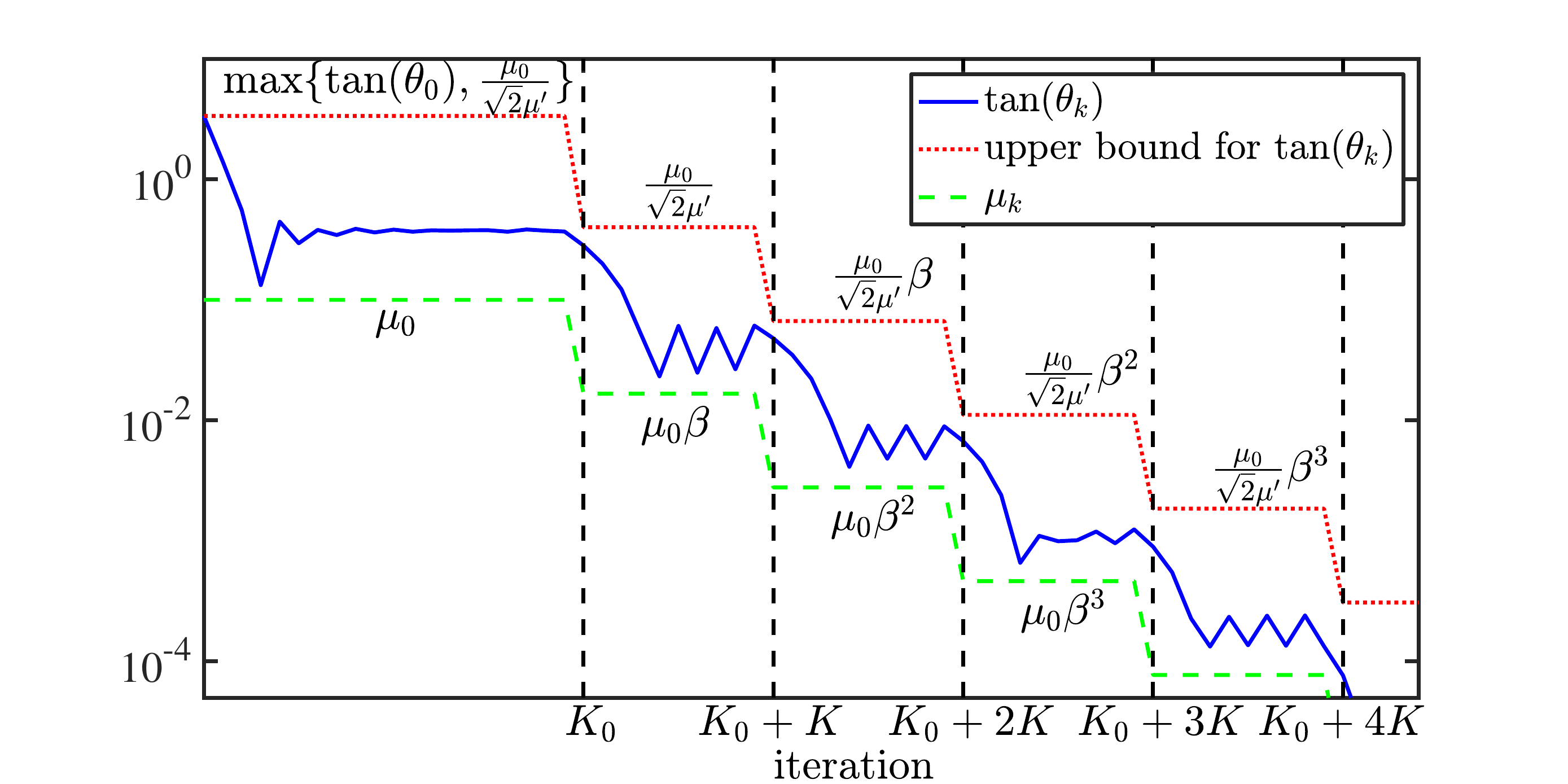}}
	\caption{Illustration of \Cref{thm:convergence PSGM all}($iv$): $\theta_k$ is the principal angle between $\vb_k$  and $\calS^\perp$ generated by the PSGM \Cref{alg:PSGM} with piecewise geometrically diminishing step size. The red dotted line represents the upper bound on $\tan(\theta_k)$ given by \eqref{eq:convergence PSGM piecewise constant}, while the green dashed line indicates the choice of the step size \eqref{eq:diminishing step size}.}
	\label{fig:describe Thm PSGM}
\end{figure}

The proof of \Cref{thm:convergence PSGM all} is given in \Cref{sec:prf convergence PSGM all}.
First note that with the choice of constant step size $\mu$, although PSGM is not guaranteed to find a normal vector, \eqref{eq:theta small constant stepsize} ensures that after $K^\diamond(\mu)$ iterations, $\widehat\vb_k$ is close to $\calS^\perp$ in the sense that $\theta_k \leq \theta^\diamond(\mu)$, which can be much smaller than $\theta_0$ for a sufficiently small $\mu$. The expressions for $K^\diamond(\mu)$ and $\theta^\diamond(\mu)$ indicate that there is a tradeoff in selecting the step size $\mu$. By choosing a larger step size $\mu$, we have a smaller $K^\diamond(\mu)$ but a larger upper bound $\theta^\diamond(\mu)$. We can balance this tradeoff according to the requirements of specific applications. For example, in applications where the accuracy of $\theta$ (to zero) is not as important as the convergence speed, it is appropriate to choose a larger step size. An alternative and more efficient way to balance this tradeoff is to change the step sizes as the iterations proceed. For the classical diminishing step sizes that are not summable, \Cref{thm:convergence PSGM all}($ii$) guarantees convergence of $\theta_k$ to zero (i.e., the iterates converge to a normal vector), though the convergence rate depends on the specific choice of step sizes. For example, \Cref{thm:convergence PSGM all}($iii$) guarantees a sub-linear convergence of $\tan(\theta_k)$ for step sizes diminishing as $1/k$.

The approach of piecewise geometrically diminishing step size (see~\Cref{thm:convergence PSGM all}($iv$)) takes advantage of the tradeoff in \Cref{thm:convergence PSGM all}($i$) by first using a relatively large initial step size $\mu_0$ so that $K^\diamond(\mu_0)$ is small (although $\theta^\diamond(\mu_0)$ is large), and then decreasing the step size in a piecewise fashion. As illustrated in \Cref{fig:describe Thm PSGM}, with such a piecewise geometrically diminishing step size, \eqref{eq:convergence PSGM piecewise constant} establishes a piecewise geometrically decaying bound for the principal angles. {Note that the curve $\tan(\theta_k)$ is not monotone because, as noted earlier, PSGM is not a descent method.} Perhaps the most surprising aspect in \Cref{thm:convergence PSGM all}($iv$) is that with the diminishing step size~\eqref{eq:diminishing step size}, we obtain a $K$-step $R$-linear convergence rate for $\tan(\theta_k)$. 
{ This linear convergence rate relies on both the choice of the step size and certain beneficial geometric structure in the problem. As characterized by \Cref{lem:critical-point}, one such structure is that all critical points in a neighborhood of $\calS^\perp$ are global solutions.  Aside from this, other properties (e.g., the negative direction of the Riemannian subgradient points toward $\calS^\perp$) are used to show the decaying rate of the principal angle. This is different from the recent work \citep{Davis:2018subgradient} in which linear convergence for PSGM is obtained for sharp and weakly convex objective functions and convex  constraint sets. Very recently, for the same optimization problem \eqref{eq:dpcp}  in the context of orthonormal dictionary learning, \cite{bai:arXiv2018DL} also utilized a subgradient method but with step size $\mu_k = O(1/k^{3/8})$, giving a sublinear convergence which is slower than the piecewise linear convergence. Thus, we believe the choice of piecewise geometrically diminishing step size is of independent interest and can be useful for other nonsmooth problems.\footnote{While smoothing allows one to use gradient-based algorithms with guaranteed convergence, the obtained solution is a perturbed version of the targeted one and thus a rounding step (such as solving a linear program~\cite{Qu:NIPS14}) is required. However, as illustrated in \Cref{fig:PSGMvsIRLS}, solving one linear program is more expensive than the PSGM for \eqref{eq:dpcp} when the data set is relatively large, thus indicating that using a smooth surrogate is not always beneficial.}}

\subsection{Initialization}
\label{sec:initialization}

\paragraph{Random Initialization}
We first consider $\widehat\vb_0$ drawn randomly from the unit sphere $\setS^{D-1}$. For such initialization, we analyze its principal angle to $\calS^\perp$ in expectation.
\begin{lem} Let  $\widehat\vb_0$ be drawn randomly from the unit sphere $\setS^{D-1}$. Then
	\[	
	\sqrt{\frac{d-1/2}{D}}\leq	\E[\sin(\theta_0)] = \frac{B(\frac{d}{2}+\frac{1}{2},\frac{D-d}{2})}{B(\frac{d}{2},\frac{D-d}{2} )}  \leq \sqrt{\frac{d}{D-1/2}},
	\]
	where $\theta_0$ is the principal angle of $\widehat\vb_0$ from $\calS^\perp$ and $B(\cdot,\cdot)$ is the beta function.
	\label{lem:random initialization}\end{lem}
The proof of \Cref{lem:random initialization} is given in \Cref{sec:prf random initialization}.
\Cref{lem:random initialization} generalizes the result in \cite[Lemma 5]{Goldstein:TIT18} which  analyzes the inner product between two  random vectors on the unit sphere. \Cref{lem:random initialization} indicates that when the subspace dimension $d$ is small, then in expectation a random initialization has a small principal angle from $\calS^\perp$. However, if $d$ is large (say $d = D-1$), then it is very possible that a random vector is very close to $\calS$, irrespectively  of the data matrix. This is because generating a random vector does not utilize the data matrix $\widetilde\bfcalX$, although it is among the easiest ones for initialization. Thus, it is possible to obtain a better initialization from more sophisticated methods using $\widetilde\bfcalX$.

\paragraph{Spectral Initialization} Another commonly used strategy is to use a spectral method generating an initialization which has much better guaranteed performance than a random one~ \citep{Lu:arXiv17}. For our problem, we use the smallest eigenvector of $\widetilde \bfcalX\widetilde \bfcalX^\T$, i.e.,  the classical PCA approach for finding a normal vector from the data $\widetilde\bfcalX$.

\begin{lem}
	Consider a spectral initialization
	$\widehat\vb_0$ by taking the bottom eigenvector of $\widetilde \bfcalX\widetilde \bfcalX^\T$. Then, $\theta_0$ (the principal angle of $\widehat\vb_0$ from $\calS^\perp$) satisfies
	\e
	\sin(\theta_0) < \sqrt{\sigma_1^2(\bfcalO) - \sigma_D^2(\bfcalO)}/{\sigma_d(\bfcalX)},
	\label{eq:angle by spectral method}
	\ee
	where $\sigma_\ell$ denotes the $\ell$-th largest singular value.
	\label{lem:spectral initialization}\end{lem}

The proof of \Cref{lem:spectral initialization} is given in \Cref{sec:prf spectral initialization}. This result together with \Cref{thm:convergence PSGM all} gives a formally guarantee of the PSGM with a spectral initialization $\vb_0$.

\begin{cor}
	Suppose the inliers and outliers satisfy
	\e\begin{split}
	\frac{\sigma_1^2(\bfcalO) - \sigma_D^2(\bfcalO)}{\sigma_d^2(\bfcalX) - \left(\sigma_1^2(\bfcalO) - \sigma_D^2(\bfcalO)\right)}	  &< \left(\frac{Nc_{\bfcalX,\min}}{{N\eta_{\bfcalX} + M\eta_{\bfcalO}}}\right)^2,\\
	Nc_{\bfcalX,\min} &\geq N\eta_{\bfcalX} + M\eta_{\bfcalO}.
\end{split}	\label{eq:PSGM and Spectral initia}\ee
	Then, the PSGM (see \Cref{alg:PSGM}) with a spectral initialization $\widehat \vb_0$ (i.e., the bottom eigenvector of $\widetilde \bfcalX\widetilde \bfcalX^\T$) and piecewise geometricallly deminishing stepsizes as in \Cref{thm:convergence PSGM all} converges to a normal vector in a linear convergence rate.
	\label{cor:PSGM and Spectral initia}\end{cor}
\Cref{cor:PSGM and Spectral initia} follows directly by letting the upper bound of $\theta_0$ specified in \eqref{eq:angle by spectral method} satisfy the requirement in \eqref{eq:initialization of PSGM v0}. We now interpret the above results in a random spherical model as used in \Cref{thm:global-opt-random-model}.

\begin{cor}
	Consider the same random spherical model as in \Cref{thm:global-opt-random-model}. Then for any positive number $t<\min\{\sqrt{N}-C_2 \sqrt{d},  2c_d\sqrt{N}-4\}$, with probability at least $1-3e^{-{t^2}/{2}}-3e^{-C_1t^2}$,  the PSGM  with a spectral initialization $\widehat \vb_0$ and piecewise geometricallly deminishing stepsizes converges to a normal vector in a linear convergence rate provided that
	\e\begin{split}
		\frac{ \frac{\sqrt{D}}{d} \left(\sqrt{N} - C_2\sqrt{d} -t \right)^2}{ 4 C_2 \sqrt M + 4C_2  t} &> 1 + \left(  \frac{ C_0\left( \sqrt{d}\log d\sqrt{N} +  \sqrt{D}\log D\sqrt{M} + t(\sqrt{N} + \sqrt{M}) \right)  }{  c_d N  - (2+ \frac{t}{2})\sqrt{N}    }    \right)^2,\\
		 \Big(  c_d N - (2+ {t}/{2}) \sqrt{N} \Big)&\ge C_0\left(\left(\sqrt{D}\log  D + t\right) \sqrt{M} + \left(\sqrt{d}\log  d + t\right) \sqrt{N}\right),
	\end{split}
	\label{eq:PSGM and Spectral initia-random model}\ee
	where $c_d$ and $c_D$ are defined in \eqref{eq:cD}, and $C_0,C_1$ and $C_2$ are universal constants indepedent of $N,M,D, d$ and $t$.	
	\label{cor:PSGM and Spectral initia-random model}\end{cor}

The proof of \Cref{cor:PSGM and Spectral initia-random model} is given in \Cref{prf:PSGM and Spectral initia-random model}. Note that when $\bfcalO$ is a Gaussian random matrix whose entries are independent normal random variables of mean $0$ and variance $\frac{1}{D}$,  $C_1 = C_2 =1$. Thus we suspect that $C_1$ and $C_2$ are  also close to $1$ 
in \Cref{cor:PSGM and Spectral initia-random model}. Note that the LHS of the first line in \eqref{eq:PSGM and Spectral initia-random model} is $O(\frac{N\sqrt{D}}{\sqrt{M}d})$, while the RHS of the first line is $1 + O(\frac{dD\log^2 DM}{N^2}+ \frac{d^2\log^2 d}{N})$. This together with the second line suggests that \eqref{eq:PSGM and Spectral initia-random model} is satisfied when $M\lesssim \frac{1}{dD\log^2D}N^2$ and $N\gtrsim d^2\log^2d$, implying that the DPCP-PSGM with a spectral initialization can tolerate $M = O(\frac{1}{dD\log^2D}N^2)$ outliers, matching the bound given in \Cref{thm:global-opt-random-model}. 

\section{Proofs}

\subsection{Proof of  \Cref{thm:DeterministicGlobal}}
\label{sec:prf DeterministicGlobal}

	Let $\vb^\star$ be an optimal solution of~\eqref{eq:dpcp}. For the sake of contradiction, suppose that $\vb^\star\not\perp \calS$, i.e., its principal angle $\phi^\star$ to the subspace $\calS$ satisfies $\phi^\star<\frac{\pi}{2}$. It then follows from \Cref{lem:critical-point} that
	\[
	\sin(\phi^\star) \leq \frac{M\overline\eta_{\bfcalO}}{N c_{\bfcalX,\min}}.
	\]
	On the other hand, utilizing the fact that $\vb^\star$ is a global minimum, we have
	\[
	\cos(\phi^\star)N c_{\bfcalX,\min} + Mc_{\bfcalO,\min} \leq g(\vb^\star) \leq \min_{\vb\in\setS^{D-1},\vb\perp \calS}\|\widetilde\bfcalX^\top\vb\|_1 = \min_{\vb\in\setS^{D-1},\vb\perp \calS}\|\bfcalO^\top\vb\|_1\leq M c_{\bfcalO,\max},
	\]
	which gives
	\[
	\cos(\phi^\star) \leq \frac{ M(c_{\bfcalO,\max} - c_{\bfcalO,\min})}{N c_{\bfcalX,\min}}.
	\]
	Combining the above inequalities on $\phi^\star$ yields
	\[
	1=\sin^2(\phi^\star) + \cos^2(\phi^\star) \leq \frac{M^2(\overline\eta_{\bfcalO}^2 + \left( c_{\bfcalO,\max} -c_{\bfcalO,\min}  \right)^2 )  }{N^2c^2_{\bfcalX,\min}}.
	\]
	But this contradicts to \eqref{eq:condition-for-global-as-a-normal}.

\subsection{Proof of Theorem \ref{thm:global-opt-random-model}}
\label{sec:prf global-opt-random-model}
The proof of \Cref{thm:global-opt-random-model} follows directly from \Cref{thm:DeterministicGlobal} and the following results concerning  different quantities in a random spherical model.
\begin{lem}
	Consider a random spherical model where the columns of $\bfcalO$ and $\bfcalX$ are drawn independently and uniformly at random from the unit sphere $\setS^{D-1}$ and the intersection of the unit sphere and a subspace $\calS$ of dimension $d<D$, respectively. Fix a number $t>0$. Then
	\begin{align*}
&\P{c_{\bfcalO,\min} \geq c_D - (2+ \frac{t}{2}) /\sqrt{M}}\geq 1-2e^{-\frac{t^2}{2}},\\
& \P{c_{\bfcalO,\max} \leq c_D + (2+ \frac{t}{2}) /\sqrt{M}}\geq 1-2e^{-\frac{t^2}{2}}, \\
	&\P{\eta_{\bfcalO} \lesssim \left(\sqrt{D}\log \left(\sqrt{c_D D}\right) + t\right)/\sqrt{M}}\geq 1-2e^{-\frac{t^2}{2}} \\
& \P{c_{\bfcalX,\min} \geq c_d  - (2+ \frac{t}{2}) /\sqrt{N}},\geq 1-2e^{-\frac{t^2}{2}},	
	\end{align*}
where $\frac{1}{D}<c_D\approx \sqrt{\frac{2}{\pi D }}<1$ is defined in \eqref{eq:cD}.
	\label{thm:quantities-ranom-model}\end{lem}
The proof of \Cref{thm:quantities-ranom-model} is given in \Cref{sec:prf-quantities-ranom-model}. We note that the above results are not optimized and thus it is possible to have much tighter results by more sophisticated analysis or a different random model. For example, for a random Gaussian model, a slightly tighter bound for $c_{\bfcalO,\min}$ is given in \citep{Lerman:FCM15} as follows:
\e
\P{c_{\bfcalO,\min} \geq \frac{2}{\pi \sqrt{D}} - 2\sqrt{\frac{1}{M}} -t\frac{1}{\sqrt{MD}}}\geq 1-e^{-\frac{t^2}{2}}.
\ee

\subsection{Proof of \Cref{thm:lp to global}}
\label{sec:prf lp to global}
We individually prove the two arguments in \Cref{thm:lp to global}.
	
\paragraph{Proof of part $(i)$:} We first rewrite the initialization $\widehat\vb_0$ by $\widehat\vb_0 = \cos(\phi_0)\vs_0 + \sin(\phi_0)\vn_0$, where $\phi_0$ is the principal angle of $\widehat\vb_0$ from $\calS$, and $\vs_0 = \calP_\calS(\widehat\vb_0)/\|\calP_\calS(\widehat\vb_0)\|_2$ and $\vn = \calP_{\calS^\perp}(\widehat\vb_0)/\|\calP_{\calS^\perp}(\widehat\vb_0)\|_2$ are the orthonormal projections of $\widehat\vb_0$ onto $\calS$ and $\calS^\perp$, respectively. For any variable $\vb$ in \eqref{eq:dpcp lp} that is not required on the unit sphere , we similarly decompose it by $\vb = \alpha \vs + \beta \vn$, where $\alpha = \|\calP_\calS(\vb)\|_2$, $\vs = \calP_\calS(\vb)/\alpha$, $\beta = \|\calP_{\calS^\perp}(\vb)\|_2$, $\vn = \calP_{\calS^\perp}(\vb)/\beta$, and $\alpha \vs$ and $\beta \vn$ are the orthonormal projections of $\vb$ onto $\calS$ and $\calS^\perp$, respectively. Now we rewrite the the $\ell_1$ minimization in \eqref{eq:dpcp lp} as
	\e\begin{split}
		&\min_{\alpha,\beta,\vs,\vn} f_1(\alpha,\beta,\vs,\vn):=\|\widetilde {\mathbfcal X}^\T(\alpha \vs + \beta \vn)\|_1, \\ & \st \ \alpha\cos(\phi_0)\vs^\top\vs_0 + \beta \sin(\phi_0)\vn^\top\vn_0= 1.
	\end{split}\label{eq:dpcp lp 1}\ee
Recall that the objective function consisits of two parts corresponding to inliers and outliers:
	\e
	\|\widetilde {\mathbfcal X}^\T(\alpha \vs + \beta \vn)\|_1 = \|{\mathbfcal X}^\top\alpha \vs\|_1  + \|{\mathbfcal O}^\T(\alpha \vs + \beta \vn)\|_1.
	\label{eq:separate bfcalX}\ee

To show that \eqref{eq:dpcp lp 1} achieves its global minimum only at $\alpha = 0$ (i.e., $\vb$ is a normal vector of $\calS$), we first separate $\alpha$ and $\beta$ in $\|{\mathbfcal O}^\T(\alpha \vs + \beta \vn)\|_1$ with a surrogate function which is not greater than $\|{\mathbfcal O}^\T(\alpha \vs + \beta \vn)\|_1$. Specifically, we have
	\e
	\|{\mathbfcal O}^\T(\alpha \vs + \beta \vn)\|_1 \geq \beta\|{\mathbfcal O}^\top \vn\|_1 - \alpha M \overline\eta_{\bfcalO},
	\label{eq:lower for O p}\ee
where $\overline\eta_{\bfcalO}$ is defined in \eqref{eq:etaO}. 
Before proving~\eqref{eq:lower for O p}, we note that by the triangle inequality of the norm, an alternative version of \eqref{eq:lower for O p} is
	\e
	\|{\mathbfcal O}^\T(\alpha \vs + \beta \vn)\|_1 \geq \beta\|{\mathbfcal O}^\top \vn\|_1 - \alpha\|\bfcalO^\top\vs\|_1.
	\label{eq:lower for O p v2}\ee
	However, the bound in \eqref{eq:lower for O p v2} is too loose in that when we plug \eqref{eq:lower for O p v2} into \eqref{eq:separate bfcalX}, we arrive at $\|\widetilde {\mathbfcal X}^\T(\alpha \vs + \beta \vn)\|_1 \geq \alpha(\|{\mathbfcal X}^\top\vs\|_1 - \|{\mathbfcal O}^\top\vs\|_1)  + \|{\mathbfcal O}^\T(\alpha \vs + \beta \vn)\|_1$, which is useful only when $\|{\mathbfcal X}^\top\vs\|_1 - \|{\mathbfcal O}^\top\vs\|_1>0$ (which requires the number of outliers $M\lesssim N$).

	On the other hand, intuitively, as $M \rightarrow \infty$ and assuming that $\bfcalO$ remains well distributed, the quantity $\frac{1}{M}\|{\mathbfcal O}^\T(\alpha \vs + \beta \vn)\|_1 \rightarrow c_D \sqrt{\alpha^2 + \beta^2}$ and $\frac{1}{M}\beta\|{\mathbfcal O}^\top \vn\|_1 \rightarrow c_D\beta$, which suggests that $\|{\mathbfcal O}^\T(\alpha \vs + \beta \vn)\|_1\geq \beta\|{\mathbfcal O}^\top\|_1$ is expected. We now turn to prove \eqref{eq:lower for O p}. To that end, define
	\e
	\chi(\alpha) = \|{\mathbfcal O}^\T(\alpha \vs + \beta \vn)\|_1 - \beta\|{\mathbfcal O}^\top\vn\|_1 + \alpha M\overline\eta_{\bfcalO}.
	\label{eq:chi alpha}\ee
	In what follows, we show $\chi(\alpha)$ is an increasing function since together with the fact $\chi(0) = 0$, it is a sufficient condition for \eqref{eq:lower for O p}. Towards that end, we let $\widehat \vb = \cos(\phi)\vs + \sin(\phi)\vn$ (where $\phi = \arccos\left(\frac{\alpha}{\sqrt{\alpha^2 + \beta^2}}\right)$) be the projection of $\vb$ onto the sphere $\setS^{D-1}$  and compute the subdifferential of $\chi(\alpha)$ as
	\e
	\partial\chi(\alpha) = \sum_{i=1}^M \Sign(\vo_i^\top(\alpha \vs + \beta \vn))\vo_i^\top\vs + M\overline\eta_{\bfcalO} = \sum_{i=1}^M \Sign(\vo_i^\top\widehat \vb)\vo_i^\top\vs + M\overline\eta_{\bfcalO}.
	\ee
	
	Now for any $a\in \partial\chi(\alpha)$, we can write it as $a = \sum_{i=1}^M \sgn(\vo_i^\top\widehat \vb)\vo_i^\top\vs + M\overline\eta_{\bfcalO}$. It follows that
	\begin{align*}
	&a = \sum_{i=1}^M \sgn(\vo_i^\top\widehat \vb)\vo_i^\top\vs + M\overline\eta_{\bfcalO}\\& = \sum_{i=1}^M \sgn(\vo_i^\top\widehat \vb)\vo_i^\top\left(\cos(\theta)(\cos(\theta)\vs + \sin(\theta)\vn) - \sin(\theta)(-\sin(\theta)\vs + \cos(\theta)\vn)\right) + M\overline\eta_{\bfcalO}\\
	& = \cos(\theta)\|\bfcalO^\top\widehat\vb\|_1 - \sin(\theta) \sum_{i=1}^M \sgn(\vo_i^\top\widehat \vb)\vo_i^\top\vxi_{\widehat\vb} + M\overline\eta_{\bfcalO}\\
	& \geq  - \sum_{i=1}^M \sgn(\vo_i^\top\widehat \vb)\vo_i^\top\vxi_{\widehat\vb} + M\overline\eta_{\bfcalO}
	\end{align*}
	where $\vxi_{\widehat\vb} = (-\sin(\theta)\vs + \cos(\theta)\vn)$. By rewriting $\sgn$ with $\sign$ as in \eqref{eq:riemann subgrad} and using the general assumption of outliers, we have
	\[
	\sum_{i=1}^M \left|\sgn(\vo_i^\top\widehat \vb)\vo_i^\top\vxi_{\widehat\vb}\right| \leq  M\overline\eta_{\bfcalO},
	\]
Thus, we have $a\geq 0$ for any $a\in \partial\chi(\alpha)$ and therefore \eqref{eq:lower for O p} follows.
	
	Now plugging \eqref{eq:lower for O p} into \eqref{eq:separate bfcalX}, we have
	\begin{align*}
	\|\widetilde {\mathbfcal X}^\T(\alpha \vs + \beta \vn)\|_1 &= \|{\mathbfcal X}^\top\alpha \vs\|_1  + \|{\mathbfcal O}^\T(\alpha \vs + \beta \vn)\|_1\geq \alpha \|{\mathbfcal X}^\top\vs\|_1 + \beta\|{\mathbfcal O}^\top\vn\|_1 - \alpha M\overline\eta_{\bfcalO}\\
	& = \alpha \left(\|{\mathbfcal X}^\top\vs\|_1 - M\overline\eta_{\bfcalO}\right) + \beta\|{\mathbfcal O}^\top\vn\|_1,
	\end{align*}
	where the inequality achieves the equality when $\alpha = 0$. Noting the assumption that $\|{\mathbfcal X}^\top\vs\|_1 - M\overline\eta_{\bfcalO} > 0$, we now consider the following problem
	\e\begin{split}
		&\min_{\alpha,\beta,\vs,\vn} f_2(\alpha,\beta,\vs,\vn):=\alpha \left(\|{\mathbfcal X}^\top\vs\|_1 - M\overline\eta_{\bfcalO}\right) + \beta\|{\mathbfcal O}^\top\vn\|_1, \\ & \st \ \alpha\cos(\phi_0)\vs_0^\top\vs + \beta \sin(\phi_0)\vn_0^\top\vn= 1.
	\end{split}\label{eq:dpcp lp 2}\ee
	Recall that $f_2(\alpha,\beta,\vs,\vn)\leq f_1(\alpha,\beta,\vs,\vn)$ and $f_2(\alpha,\beta,\vs,\vn)= f_1(\alpha,\beta,\vs,\vn)$ when $\alpha = 0$. Thus, if we show that the optimal solution for \eqref{eq:dpcp lp 2} is obtained only when $\alpha = 0$, we conclude that the optimal solution for \eqref{eq:dpcp lp 1} is also obtained only when $\alpha = 0$. The remaining part is to consider the global solution of \eqref{eq:dpcp lp 2}.
	
	Suppose that $(\alpha^\star,\beta^\star,\vs^\star,\vn^\star)$ is an optimal solution of \eqref{eq:dpcp lp 2} with $\alpha^\star>0$. Noting that the $l_1$-norm is absolutely scalable, it is clear that $\vs_0^\top\vs^\star\geq 0$ and $\vn_0^\top\vn^\star\geq 0$. To obtain the contradiction, we construct $\alpha' = 0,\vs' = \vs^\star,\beta'\vn' = \beta^\star\vn^\star + \gamma\vn_0$ where $\gamma$ is determined such that $\vp' = \beta'\vn' $ satisfies the condition $\widehat\vb_0^\top\vp' = 1$, i.e.,
	\[
	\sin(\phi_0)\beta^\star\vn_0^\top\vn^\star + \sin(\phi_0)\gamma = 1,
	\]
	which implies that
	\[
	\gamma = \frac{1}{\sin(\phi_0)} - \beta^\star\vn_0^\top\vn^\star.
	\]
	Since $(\alpha^\star,\beta^\star,\vs^\star,\vn^\star)$ is an optimal solution of \eqref{eq:dpcp lp 2}, it also satisfies the constraint
	\[
	\alpha^\star\cos(\phi_0)\vs_0^\top\vs^\star + \beta^\star \sin(\phi_0)\vn_0^\top\vn^\star= 1,
	\]
	which together with the above equation gives
	\[
	\gamma = \frac{1}{\tan(\phi_0)}\alpha^\star\vs_0^\top\vs^\star.
	\]
	Now we have
	\begin{align*}
	&f_2(\alpha^\star,\beta^\star,\vs^\star,\vn^\star) - f_2(\alpha',\beta',\vs',\vn'):=\alpha^\star \left(\|{\mathbfcal X}^\top\vs^\star\|_1 - M\overline\eta_{\bfcalO}\right) + \beta^\star\|{\mathbfcal O}^\top\vn^\star\|_1 - \|{\mathbfcal O}^\T(\beta'\vn')\|_1\\
	&=\alpha^\star \left(\|{\mathbfcal X}^\top\vs^\star\|_1 - M\overline\eta_{\bfcalO}\right) + \beta^\star\|{\mathbfcal O}^\top\vn^\star\|_1 - \|{\mathbfcal O}^\T(\beta^\star\vn^\star + \gamma\vn_0)\|_1\\
	& \geq \alpha^\star \left(\|{\mathbfcal X}^\top\vs^\star\|_1 - M\overline\eta_{\bfcalO}\right) + \beta^\star\|{\mathbfcal O}^\top\vn^\star\|_1 - \beta^\star\|{\mathbfcal O}^\top\vn^\star\|_1 - \gamma\|{\mathbfcal O}^\top\vn_0\|_1\\
	& = \alpha^\star\left( \|{\mathbfcal X}^\top\vs^\star\|_1 -M\overline\eta_{\bfcalO} \right) - \frac{1}{\tan(\phi)}\alpha^\star\vs_0^\top\vs^\star\|{\mathbfcal O}^\top\vn_0\|_1\\
	& \geq \alpha^\star\left( \|{\mathbfcal X}^\top\vs^\star\|_1 - M\overline\eta_{\bfcalO} - \frac{1}{\tan(\phi_0)}\|{\mathbfcal O}^\top\vn_0\|_1 \right)\\
	& \geq \alpha^\star\left(N c_{\bfcalX,\min} - M\overline\eta_{\bfcalO}- \frac{1}{\tan(\phi_0)}Mc_{\bfcalO,\max} \right)>0,
	\end{align*}
	where the last inequality follows because of \eqref{eq:angle lp succed} that
	\[
	\phi_0>\arctan\left(\frac{Mc_{\bfcalO,\max}}{N c_{\bfcalX,\min}-M\overline\eta_{\bfcalO}}\right).
	\]
	This contradicts to the assumption that $(\alpha^\star,\beta^\star,\vs^\star,\vn^\star)$ is an optimal solution of \eqref{eq:dpcp lp 2} and  thus we conclude that the optimal solution for \eqref{eq:dpcp lp 2} is obtained only when $\alpha = 0$. And so does \eqref{eq:dpcp lp 1}, implying that the optimal solution to \eqref{eq:dpcp lp 1} must be orthogonal to $\calS$.

\paragraph{Proof of part $(ii)$:}
Due to the constraint $\vb_k^\top\widehat \vb_{k-1} = 1$, we have $\|\vb_k\|_2 \geq 1$. It follows that
	\e
	\cdots \leq g(\widehat \vb_k)\leq g(\vb_k)\leq  g(\widehat\vb_{k-1})\leq \cdots \leq g(\widehat\vb_0).
	\label{eq:decay obj LP}\ee
	Invoke \cite[Proposition 16]{Tsakiris:DPCP-Arxiv17} which states that the sequence $\{\vb_k\}$ converges to a critical point $\vb^\star$ of problem~\eqref{eq:dpcp}. For the sake of contradiction, suppose that $\vb^\star\not\perp \calS$, i.e., its principal angle $\phi^\star<\frac{\pi}{2}$. Utilizing \eqref{eq:decay obj LP}, we have
	\[
	g(\vb^\star) = \|\bfcalX^\top\vb^\star\|_1 + \|\bfcalO^\top\vb^\star\|_1\leq g(\widehat\vb_0) = \|\bfcalX^\top\widehat\vb_0\|_1 + \|\bfcalO^\top\widehat\vb_0\|_1.
	\]
	Plugging the inequalities $\|\bfcalX^\top\vb^\star\|_1 \geq N\cos(\phi^\star) c_{\bfcalX,\min}$, $\|\bfcalO^\top\vb^\star\|_1 \geq Mc_{\bfcalO,\min}$, $\|\bfcalX^\top\widehat\vb_0\|_1 \leq \cos(\phi_0) Nc_{\bfcalX,\max}$ and $\|\bfcalO^\top\widehat\vb_0\|_1 \leq Mc_{\bfcalO,\max}$ into the above equation gives
	\e
	\cos(\phi^\star) \leq \frac{\cos(\phi_0) Nc_{\bfcalX,\max} + M(c_{\bfcalO,\max} - c_{\bfcalO,\min})}{N c_{\bfcalX,\min}}.
	\nonumber\ee
	On the other hand, since $\vb^\star$ is a critical point of \eqref{eq:dpcp}, it follows from the first order optimality (see \Cref{lem:critical-point}) that
	\e
	\sin(\phi^\star) \leq \frac{M\overline\eta_{\bfcalO}}{N c_{\bfcalX,\min}}.
	\nonumber	\ee
	Combining the above equation together gives
	\e
	1 = \cos^2(\phi^\star) + \sin^2(\phi^\star) \leq \frac{\sqrt{\left(\cos(\phi_0)N c_{\bfcalX,\max} + M(c_{\bfcalO,\max} - c_{\bfcalO,\min})\right)^2 + M^2\overline\eta^2_{\bfcalO} }}{N c_{\bfcalX,\min}},
	\nonumber	\ee
	which contradicts \eqref{eq:angle sequential lp succed}.

\subsection{Proof of \Cref{thm:convergence PSGM all}}
\label{sec:prf convergence PSGM all}
The proof of \Cref{thm:convergence PSGM all} builds heavily on the following result characterizing the behaviors of the iterates generated by \Cref{alg:PSGM}.
\begin{lem}[Analysis of iterates for the PSGM]
	Let $\{\widehat \vb_k\}$ be the sequence generated by \Cref{alg:PSGM}  with initialization $\widehat \vb_0$ whose principal angle to the normal subspace $\calS^\perp$ satisfies
	\e
	\theta_0  < \arctan\left(\frac{N c_{\bfcalX,\min}}{N\eta_{\bfcalX} + M\eta_{\bfcalO}}\right)
	\label{eq:initialization of GPD}\ee
	and step size satisfying
	\begin{align}
	\mu_k \leq \mu' :=\frac{1}{4\cdot\max\{N c_{\bfcalX,\min},M c_{\bfcalO,\max} \}}, \ \forall k\geq 0.
	\label{eq:step size}\end{align}
	Given
	\e
	c_{\bfcalX,\min} \geq \eta_{\bfcalX} + \frac{M}{N}\eta_{\bfcalO},
	\label{eq:condition for PSGM}\ee
	the angle $\theta_k$ of $\widehat\vb_k$ to $\calS^\perp$ satisfies the following properties.
	\begin{enumerate}[label=(\roman*)]
		\item (decay of $\theta_{k}$ when $\theta_{k-1}$ is relatively large compared with $\mu_k$) In the case
		\e
		\sin(\theta_{k-1}) > \frac{\mu_k Nc_{\bfcalX,\min}}{(1 - \mu_k Mc_{\bfcalO,\max})},
		\label{eq:region I}\ee
		we have
		\e\begin{split}
			&\tan(\theta_{k-1}) - \tan(\theta_k) \\&\geq \mu_k\min\left\{N c_{\bfcalX,\min} - (N\eta_{\bfcalX} + M\eta_{\bfcalO}), Nc_{\bfcalX,\min} - \tan(\theta_{k-1})(N\eta_{\bfcalX} + M\eta_{\bfcalO})\right\}.
		\end{split}\label{eq:decay region I}\ee
		\item (upper bound for $\theta_k$ when $\theta_{k-1}$ is relatively small compared with $\mu_k$) In the case
		\e
		\sin(\theta_{k-1}) \leq \frac{\mu_k Nc_{\bfcalX,\min}}{(1 - \mu_k Mc_{\bfcalO,\max})},
		\label{eq:region II}\ee
		we have
		\e\begin{split}
			\tan(\theta_k) &\leq \mu_k\max\bigg\{  \frac{c_{\bfcalX,\max} + (N\eta_{\bfcalX} + M\eta_{\bfcalO})}{ (1 - \mu_k Mc_{\bfcalO,\max})\cos(\theta_{k-1})},\\ &\qquad \qquad \qquad \frac{ c_{\bfcalX,\max} }{ (1 - \mu_k Mc_{\bfcalO,\max})\cos(\theta_{k-1}) - \mu_k(N\eta_{\bfcalX} + M\eta_{\bfcalO}) } \bigg\}\\
			&\leq \frac{1}{\sqrt{2}}\frac{\mu_k}{\mu'}.
		\end{split}\label{eq:region II angle small}\ee
	\end{enumerate}
	\label{thm:convergence PSGM}\end{lem}
The proof of \Cref{thm:convergence PSGM} is given in \Cref{sec:prf-convergence PSGM}.  With \Cref{thm:convergence PSGM}, we now prove the four arguments in \Cref{thm:convergence PSGM all} in the following four subsections.

\subsubsection{Proof of \Cref{thm:convergence PSGM all}($i$)}

The argument about $k\leq K^\diamond$ follows directly from \Cref{thm:convergence PSGM} which implies that either $\theta_k < \theta_{k-1}$ or $\theta_k$ satisfies \eqref{eq:theta small constant stepsize}.

We now trun to prove the argument about $k\geq K^\diamond$.  First assume that $\theta_{K^\diamond}>0$ and \eqref{eq:region I} holds for all $k\leq K^\diamond $, i.e.,
\e
\sin(\theta_{k-1}) >   \frac{\mu \|\bfcalX^\top\vs_{k-1}\|_1}{(1 - \mu \|\bfcalO^\top\widehat \vb_{k-1}\|_1)}  \geq \frac{\mu Nc_{\bfcalX,\min}}{(1 - \mu Mc_{\bfcalO,\min})}, \ \forall k\leq K^\diamond.
\ee
It follows from \eqref{eq:decay region I} that
\begin{align*}
\tan(\theta_{K^\diamond}) &\leq \tan(\theta_0) - K^\diamond \mu (Nc_{\bfcalX,\min} - \max\{1,\tan(\theta_0)\}(N\eta_{\bfcalX} + M\eta_{\bfcalO}))\\
&\leq 0,
\end{align*}
which contradicts to the fact that $\theta_{K^\diamond}> 0$. Thus, either of the following case must hold:
\begin{enumerate}[label=($\roman*$)]
	\item  $\theta_{K^\diamond} = 0$; 
	\item  there exists $K_1 \leq K^\diamond$ such that
	\e
	\sin(\theta_{K_1-1}) \leq \frac{\mu Nc_{\bfcalX,\min}}{(1 - \mu c_{\bfcalO,\max})}.
	\label{eq:theta sin theta K1}\ee
\end{enumerate}

For case $(i)$, due to \Cref{thm:convergence PSGM} which implies that either $\theta_k < \theta_{k - 1}$  or \eqref{eq:theta small constant stepsize} holds for all $k\geq K^\diamond+1$, and by induction we conclude that \eqref{eq:theta small constant stepsize} holds for all $k\geq K^\diamond$. 

In case $(ii)$,	invoking \eqref{eq:region II angle small}, we have
\begin{align*}
\tan(\theta_{K_1}) &\leq \frac{1}{\sqrt{2}}\frac{\mu}{\mu'}.
\end{align*}
Again \Cref{thm:convergence PSGM}  implies that either $\theta_k < \theta_{k-1}$ or \eqref{eq:theta small constant stepsize} holds for all $k\geq K_1+1$. By induction, we have that \eqref{eq:theta small constant stepsize} holds for all $k\geq K_1$ and by assumtion that $K_1 \leq K^\diamond$. This completes the proof.

\subsubsection{Proof of \Cref{thm:convergence PSGM all}($ii$)}
We first show that for any $K_1 \geq 1$, there exists $K_2 \geq K_1$ such that \eqref{eq:region II} is true for $k= K_2$. We prove it by contradiction. Suppose \eqref{eq:region I} holds for all $k\geq K_1$, which implies that
\e
\tan(\theta_{k-1}) - \tan(\theta_k) \geq \mu_k (Nc_{\bfcalX,\min} - \max\{1,\tan(\theta_0)\}(N\eta_{\bfcalX} + M\eta_{\bfcalO}))
\ee
for all $k\geq K_1$. Repeating the above equation for all $k\geq K_1$ and summing them up give
\[
\tan(\theta_{K_1-1})  \geq \tan(\theta_{K_1-1}) - \lim_{k\rightarrow \infty} \tan(\theta_{k}) \geq (Nc_{\bfcalX,\min} - (N\eta_{\bfcalX} + M\eta_{\bfcalO}))\sum_{k = K_1}^\infty \mu_k,
\]
where we utilize the fact that $\theta_k\in[0,\frac{\pi}{2})$ for all $k\geq 0$. The above equation implies that
\[
\sum_{k = K_1}^\infty \mu_k <\infty,
\]
which contradicts to \eqref{eq:diminishing step size}. Thus, there exists $K_2 \geq K_1$ such that \eqref{eq:region II} is true for $k= K_2$. Now invoking \eqref{eq:region II angle small}, we have
\e
\tan(\theta_{K_2}) \leq \frac{1}{\sqrt{2}}\frac{\mu_{K_2}}{\mu'}.
\label{eq:angle small K2}\ee

For the angle $\theta_{k}$, $k> K_2 $, invoking \Cref{thm:convergence PSGM}, we have
\[
\theta_k \leq \tan(\theta_{K_2}) \sup_{i\geq K_2} \frac{1}{\sqrt{2}}\frac{\mu_{i}}{\mu'}
\]
for all $k\geq K_2$. Now letting $K_1 \rightarrow \infty$, which implies that $K_2 \rightarrow \infty$. Then, the above equation along with \eqref{eq:diminishing step size} that $\lim_{i\rightarrow \infty}\mu_i =0$ gives that
\[
\lim_{k\rightarrow \infty} \theta_k \leq 0.
\]
Sine $\theta_k \geq 0$, we finally obtain
\[
\lim_{k\rightarrow \infty} \theta_k = 0.
\]

\subsubsection{Proof of \Cref{thm:convergence PSGM all}($iii$)}
It follows from \Cref{thm:convergence PSGM} that at the $k$-th iteration, we have either
\e
\tan(\theta_k) \leq \tan(\theta_{k-1}) -  \frac{\mu_0}{k}\min\{N c_{\bfcalX,\min} - (N\eta_{\bfcalX} + M\eta_{\bfcalO}), Nc_{\bfcalX,\min} - \tan(\theta_{k-1})(N\eta_{\bfcalX} + M\eta_{\bfcalO})\}
\label{eq:theta decay prove 1/k}\ee
or
\e
\tan(\theta_k) \leq \frac{1}{\sqrt{2}}\frac{\mu_k}{\mu'} =
\frac{\mu_0}{\sqrt{2}\mu'}\frac{1}{k}.
\label{eq:theta small prove 1/k}\ee
Therefore, by induction, we have
\e
\tan(\theta_k) \leq \max\left\{\tan (\theta_0), \frac{\mu_0}{\sqrt{2}\mu'} \right\}
\ee
for all $k\geq 0$.

To further proceed, we first assume that there exists
\e
K^\star \geq \frac{1}{ \sqrt{2}\mu'\left(Nc_{\bfcalX,\min} - (N\eta_{\bfcalX} + M\eta_{\bfcalO})\right) }
\nonumber	\ee
such that
\e
\tan(\theta_{K^\star}) \leq
\frac{\mu_0}{\sqrt{2}\mu'}\frac{1}{K^\star}.
\label{eq:tan Kstar}\ee

With this assumption, in what follows, we prove \eqref{eq:theta small prove 1/k} holds for all $k\geq K^\star$ by induction. To that ends, first note that \eqref{eq:theta small prove 1/k} holds for $k = K^\star$. Now suppose \eqref{eq:theta small prove 1/k} holds for some $k\geq K^\star$, which implies that $\tan(\theta_k)\leq \frac{1}{\sqrt{2}}$. Then we know either \eqref{eq:theta small prove 1/k} holds for $k+1$ or \eqref{eq:theta decay prove 1/k} for $k+1$. For the later case, we have
\begin{align*}
\tan(\theta_{k+1}) & \leq \tan(\theta_{k}) -  \mu_0\left(Nc_{\bfcalX,\min} - (N\eta_{\bfcalX} + M\eta_{\bfcalO})\right)\frac{1}{k+1}\\
& \leq \frac{\mu_0}{\sqrt{2}\mu'}\frac{1}{k} - \mu_0\left(Nc_{\bfcalX,\min} - (N\eta_{\bfcalX} + M\eta_{\bfcalO})\right)\frac{1}{k+1}\\
& \leq \frac{\mu_0}{\sqrt{2}\mu'}\frac{1}{k+1} - \frac{1}{k+1}\mu_0 \left(  c_{\bfcalX,\min} - (N\eta_{\bfcalX} + M\eta_{\bfcalO}) - \frac{1}{k\sqrt{2}\mu'} \right)\\
& \leq \frac{\mu_0}{\sqrt{2}\mu'}\frac{1}{k+1},
\end{align*}
where the first inequlaity follows because of \eqref{eq:theta decay prove 1/k}  and $\tan(\theta_k)\leq \frac{1}{\sqrt{2}}$, and the last line utilizes the fact that $k\geq K^\star \geq \frac{1}{ \sqrt{2}\mu'\left(Nc_{\bfcalX,\min} -(N\eta_{\bfcalX} + M\eta_{\bfcalO})\right) }$. Thus, by induction, \eqref{eq:theta small prove 1/k} holds for all $k\geq K^\star$.

The rest of the proof is to show the existence of such $K^\star$. Denote by
\e
K_3 = \left\lceil   \frac{1}{ \sqrt{2}\mu'\left(Nc_{\bfcalX,\min} - (N\eta_{\bfcalX} + M\eta_{\bfcalO})\right) }\right\rceil.
\nonumber	\ee
Now suppose that for all $k\geq K_3$, \eqref{eq:region II} is not true, which implies that \eqref{eq:theta decay prove 1/k} must hold. Thus, we have
\begin{align*}
\tan(\theta_k) &\leq \max\left\{ \tan(\theta_0), \frac{\mu_0}{\sqrt{2}\mu'} \right\}\\
& \quad - \mu_0\min\left\{N c_{\bfcalX,\min} - (N\eta_{\bfcalX} + M\eta_{\bfcalO}), Nc_{\bfcalX,\min} - \tan(\theta_{k-1})(N\eta_{\bfcalX} + M\eta_{\bfcalO})\right\}\left(\sum_{i=K^\diamond}^k\frac{1}{i}\right)\\
& \leq \max\left\{ \tan(\theta_0), \frac{\mu_0}{\sqrt{2}\mu'} \right\} \\
&\quad - \mu_0\min\{N c_{\bfcalX,\min} - (N\eta_{\bfcalX} + M\eta_{\bfcalO}), Nc_{\bfcalX,\min} - \tan(\theta_{k-1})(N\eta_{\bfcalX} + M\eta_{\bfcalO})\}\log(k/K^\diamond)
\end{align*}
for all $k\geq K_3$. The above equation implies that $\tan(\theta_k)\leq 0$ for all $k\geq K'$ where
\[
K' = K_3 e^{\frac{  \max\{ \tan(\theta_0), \frac{\mu_0}{\sqrt{2}\mu'} \}   }{ \mu_0\min\{N c_{\bfcalX,\min} - (N\eta_{\bfcalX} + M\eta_{\bfcalO}), Nc_{\bfcalX,\min} - \tan(\theta_{k-1})(N\eta_{\bfcalX} + M\eta_{\bfcalO})\} }}.
\]
This contradicts to the assumption that \eqref{eq:region I} always holds. Therefore, there exists at least one $K^\star\in [K_3,K']$ such that \eqref{eq:tan Kstar} holds. Thus \eqref{eq:theta small prove 1/k} holds for all $k\geq K^\star$. This together with the fact that  $\tan(\theta_k) \leq \max\{\tan (\theta_0), \frac{\mu_0}{\sqrt{2}\mu'} \}$ for all $k\leq K^\star$ completes the proof of \Cref{thm:convergence PSGM all}($iii$).

\subsubsection{Proof of \Cref{thm:convergence PSGM all}($iv$)}

As illustrated in \Cref{fig:describe Thm PSGM}, our main idea is to bound the iterates in each piece or block with \Cref{thm:convergence PSGM all}($i$). To that end, we first use \Cref{thm:convergence PSGM all}($i$) with $\mu = \mu_0$ to get
\e
\theta_k \leq \left\{\begin{matrix}\max\{\theta_0,\theta_0^\diamond\}, & k< K_0^\diamond, \\ \theta_0^\diamond, & k\geq  K_0^\diamond,\end{matrix} \right.
\label{eq:theta small proof exp decay 0}\ee	
where
\e
K_0^\diamond =\left\lceil\frac{\tan(\theta_0)}{ \mu_0\min\left\{N c_{\bfcalX,\min} - (N\eta_{\bfcalX} + M\eta_{\bfcalO}),c_{\bfcalX,\min} - \tan(\theta_0)(N\eta_{\bfcalX} + M\eta_{\bfcalO})\right\}}\right\rceil \leq K_0
\label{eq:K diamond 0}\ee
and
\e
\tan(\theta_0^\diamond) \leq \frac{1}{\sqrt{2}}\frac{\mu_0}{\mu'}.
\label{eq:theta diamond 0}\ee

Plugging \eqref{eq:K diamond 0} and \eqref{eq:theta diamond 0} back to \eqref{eq:theta small proof exp decay 0} gives
\e
\tan(\theta_k) \leq \left\{\begin{matrix}\max\{\tan(\theta_0),\frac{\mu_0}{\sqrt{2}\mu'}\}, & k< K_0 \\ \frac{\mu_0}{\sqrt{2}\mu'}, & k\geq  K_0.\end{matrix} \right.
\label{eq:theta small proof exp decay 0 v2}
\nonumber\ee	

At $K_0$-th step, the step size becomes $\beta\mu_0$. We can now veiw the following steps as they are initialized at $\widehat \vb_{K_0}$ with $\theta_{K_0}$ satisfying the above equation. Also, as presented through the proof of \Cref{thm:convergence PSGM}, \eqref{eq:initialization of GPD} holds for all $k\geq 0$. Thus,  applying \Cref{thm:convergence PSGM all}($i$) with $\theta_0 = \theta_{K_0}$ and $\mu = \mu_0\beta$, we have
\e
\theta_k \leq \left\{\begin{matrix}\max\{\theta_{K_0},\theta_1^\diamond\}, & K_0 \leq k< K_0 + K_1^\diamond, \\ \theta_1^\diamond, & k\geq  K_0 + K_1^\diamond,\end{matrix} \right.
\label{eq:theta small proof exp decay 1}\nonumber\ee	
where
\e
K_1^\diamond =\left\lceil\frac{\tan(\theta_{K_0})}{ \mu_0\beta\min\left\{N c_{\bfcalX,\min} - (N\eta_{\bfcalX} + M\eta_{\bfcalO}),c_{\bfcalX,\min} - \tan(\theta_{K_0})(N\eta_{\bfcalX} + M\eta_{\bfcalO})\right\}}\right\rceil
\label{eq:K diamond 1}\ee
and
\e
\tan(\theta_1^\diamond) \leq \frac{1}{\sqrt{2}}\frac{\mu_0\beta}{\mu'}.
\label{eq:theta diamond 1}\ee

It follows from \eqref{eq:theta small proof exp decay 0}-\eqref{eq:theta diamond 0} that
\e
\tan(\theta_{K_0}) \leq \frac{1}{\sqrt{2}}\frac{\mu_0}{\mu'}<1,
\label{eq:theta K 0}\ee
which plugged into \eqref{eq:K diamond 1} gives
\begin{align*}
K_1^\diamond &= \left\lceil\frac{\tan(\theta_{K_0})}{ \mu_0\beta\min\left\{N c_{\bfcalX,\min} - (N\eta_{\bfcalX} + M\eta_{\bfcalO}),c_{\bfcalX,\min} - \tan(\theta_{K_0})(N\eta_{\bfcalX} + M\eta_{\bfcalO})\right\}}\right\rceil\\
& = \left\lceil\frac{\tan(\theta_{K_0})}{ \mu_0\beta \left(Nc_{\bfcalX,\min} - (N\eta_{\bfcalX} + M\eta_{\bfcalO})\right)}\right\rceil \\
&\leq \left\lceil\frac{\frac{1}{\sqrt{2}}\frac{\mu_0\beta}{\mu'}}{ \mu_0\beta \left(Nc_{\bfcalX,\min} - (N\eta_{\bfcalX} + M\eta_{\bfcalO})\right)}\right\rceil \\
&= \left\lceil\frac{ 1}{ \sqrt{2}\beta\mu'\left(Nc_{\bfcalX,\min} - (N\eta_{\bfcalX} + M\eta_{\bfcalO})\right)}\right\rceil \\&\leq K.
\end{align*}
Plugging the above equation, \eqref{eq:theta diamond 1} and \eqref{eq:theta K 0}  into \eqref{eq:theta small proof exp decay 1}  gives
\e
\tan(\theta_k) \leq \left\{\begin{matrix} \frac{1}{\sqrt{2}}\frac{\mu_0}{\mu'}, & K_0\leq k< K + K, \\ \frac{1}{\sqrt{2}}\frac{\mu_0}{\mu'} \beta, & k\geq  K_0 + K.\end{matrix} \right.
\label{eq:theta small proof exp decay 1 v2}\ee	

We now complete the proof of \eqref{eq:convergence PSGM piecewise constant} by induction. Suppose for some $\ell\geq 1$, the following holds
\e
\tan(\theta_k) \leq \left\{\begin{matrix} \frac{1}{\sqrt{2}}\frac{\mu_0}{\mu'}\beta^{\ell-1}, & K_0 + (\ell - 1)K\leq k< K_0 + \ell K, \\ \frac{1}{\sqrt{2}}\frac{\mu_0}{\mu'} \beta^\ell, & k\geq  K_0 + \ell K.\end{matrix} \right.
\label{eq:theta small proof exp decay ell}\ee	

Similarly, at $(K_0 + \ell K)$-th step, the step size changes to $\mu_0\beta ^{\ell +1}$. We veiw the following steps as they are initialized at $\widehat \vb_{K_0+\ell K}$ with $\theta_{K_0+\ell K}$ satisfying the above equation. Thus,  applying \Cref{thm:convergence PSGM all}($i$) with $\theta_0 = \theta_{K_0+\ell K}$ and $\mu = \mu_0\beta^{\ell + 1}$, we have
\e
\theta_k \leq \left\{\begin{matrix}\max\{\theta_{K_0},\theta_{\ell + 1}^\diamond\}, & K_0 + \ell K\leq k< K_0 +\ell K+ K_{\ell + 1}^\diamond, \\ \theta_{\ell + 1}^\diamond, & k\geq  K_0 + \ell K+ K_{\ell + 1}^\diamond,\end{matrix} \right.
\label{eq:theta small proof exp decay l + 1}\ee	
where
\e
K_{\ell + 1}^\diamond =\left\lceil\frac{\tan(\theta_{K_0+\ell K})}{ \mu_0\beta^{\ell +1}\min\left\{N c_{\bfcalX,\min} - (N\eta_{\bfcalX} + M\eta_{\bfcalO}), Nc_{\bfcalX,\min} - \tan(\theta_{K_0 + \ell K})(N\eta_{\bfcalX} + M\eta_{\bfcalO})\right\}}\right\rceil
\label{eq:K diamond l + 1}\ee
and
\e
\tan(\theta_{\ell + 1}^\diamond) \leq \frac{1}{\sqrt{2}}\frac{\mu_0}{\mu'}\beta^{\ell + 1}.
\label{eq:theta diamond l + 1}\ee

Plugging $\tan(\theta_{K_0 + \ell K}) \leq \frac{1}{\sqrt{2}}\frac{\mu_0}{\mu'} \beta^\ell$ into \eqref{eq:K diamond l + 1}, we have

\begin{align*}
K_{\ell + 1}^\diamond &= \left\lceil\frac{\tan(\theta_{K_0+\ell K})}{ \mu_0\beta^{\ell +1}\min\left\{N c_{\bfcalX,\min} - (N\eta_{\bfcalX} + M\eta_{\bfcalO}), Nc_{\bfcalX,\min} - \tan(\theta_{K_0 + \ell K})(N\eta_{\bfcalX} + M\eta_{\bfcalO})\right\}}\right\rceil\\
& = \left\lceil\frac{\tan(\theta_{K_0+\ell K})}{ \mu_0\beta^{\ell + 1} \left(Nc_{\bfcalX,\min} - (N\eta_{\bfcalX} + M\eta_{\bfcalO})\right)}\right\rceil \\
&\leq \left\lceil\frac{\frac{1}{\sqrt{2}}\frac{\mu_0\beta^{\ell}}{\mu'}}{ \mu_0\beta^{\ell + 1} \left(Nc_{\bfcalX,\min} - (N\eta_{\bfcalX} + M\eta_{\bfcalO})\right)}\right\rceil \\
&= \left\lceil\frac{ 1}{ \sqrt{2}\beta\mu'\left(Nc_{\bfcalX,\min} - (N\eta_{\bfcalX} + M\eta_{\bfcalO})\right)}\right\rceil \\ &\leq K,
\end{align*}
which together with \eqref{eq:theta diamond l + 1}, $\tan(\theta_{K_0 + \ell K}) \leq \frac{1}{\sqrt{2}}\frac{\mu_0}{\mu'} \beta^\ell$  and \eqref{eq:theta small proof exp decay l + 1}  gives
\e
\tan(\theta_k) \leq \left\{\begin{matrix} \frac{1}{\sqrt{2}}\frac{\mu_0}{\mu'}\beta^\ell, & K_0 + \ell K\leq k< K_0 + (\ell+1) K, \\ \frac{1}{\sqrt{2}}\frac{\mu_0}{\mu'} \beta^{(\ell+1)}, & k\geq  K_0 + (\ell+1) K.\end{matrix} \right.
\label{eq:theta small proof exp decay ell v2}\ee

By induction, this completes the proof of \eqref{eq:convergence PSGM piecewise constant} and hence \Cref{thm:convergence PSGM all}($iv$).

\subsection{Proofs for \Cref{sec:initialization}}

\subsubsection{Proof of \Cref{lem:random initialization}}
\label{sec:prf random initialization}
It is equivalent to consider a random Gaussian vector $\vb_0\in\R^{D}$ whose elements are independently generated form a normal distribution. Because the Gaussian distribution is invariant under the orthogonal group of transformations, without loss of generality, we suppose $\calS = \Span\{\ve_1,\ldots,\ve_d\}$ where $\ve_1,\ldots,\ve_D$ form a canonical basis of $\R^D$. Now the quantity $\sin^2(\theta_0)$ is simply $\|\vb_0(1:d)\|^2/\|\vb_0\|^2$ whose distribution function is given by \cite[Theorem 3.3.4]{Muirhead:JW2009}:
	\[
	g(z) = \frac{1}{B(\frac{d}{2},\frac{D-d}{2} )}z^{\frac{d}{2} -1}(1-z)^{\frac{D-d}{2}-1},
	\]	
	where $B(\cdot,\cdot)$ is the beta function. Hence, we compute the expectation of the quantity $\sin(\theta_0)$
	\[
	\E[\sin(\theta_0)] = \frac{1}{B(\frac{d}{2},\frac{D-d}{2} )}\int_0^1 z^{1/2}z^{\frac{d}{2} -1}(1-z)^{\frac{D-d}{2}-1}\dif z = \frac{B(\frac{d}{2}+\frac{1}{2},\frac{D-d}{2})}{B(\frac{d}{2},\frac{D-d}{2} )} = \frac{\gamma(\frac{D}{2})}{\gamma(\frac{D+1}{2})} \frac{\gamma(\frac{d+1}{2})}{\gamma(\frac{d}{2})}.
	\]
	Now plugging the bound $\sqrt{(D-1/2)/2}\leq \frac{\gamma(\frac{D+1}{2})}{\gamma(\frac{D}{2})} \leq \sqrt{D/2}$ into the above equation gives
	\[
	\sqrt{\frac{d-1/2}{D}}\leq	\E[\sin(\theta_0)] \leq \sqrt{\frac{d}{D-1/2}}.
	\]

\subsubsection{Proof of \Cref{lem:spectral initialization}}
\label{sec:prf spectral initialization}
	Note that for any $\vb\perp \calS$, $\|\widetilde \bfcalX^\top\vb\|^2 = \|\bfcalO^\top\vb\|^2 \leq \sigma_1^2(\bfcalO)$. Thus, since $\widehat\vb_0$ is the optimal solution to $\argmin_{\vb\in\setS^{D-1}}\|\widetilde \bfcalX^\top\vb\|^2$, we have
	\[
	\|\widetilde \bfcalX^\top\widehat\vb_0\|^2\leq  \sigma_1^2(\bfcalO).
	\]
	On the other hand, we have
	\[
	\|\widetilde \bfcalX^\top\widehat\vb_0\|^2 = \|\bfcalX^\top\widehat\vb_0\|^2 + \|\widetilde \bfcalO^\top\widehat\vb_0\|^2 = \sin^2(\theta_0)\|\bfcalX^\top\vs_0\|^2 + \|\widetilde \bfcalO^\top\widehat\vb_0\|^2 \geq \sin^2(\theta_0)\sigma_d^2(\bfcalX)  + \sigma_D^2(\bfcalO),
	\]
	where $\vs_0 = \calP_\calS(\widehat\vb_0)/\|\calP_\calS(\widehat\vb_0)\|_2$ is the orthonormal projections of $\widehat\vb_0$ onto $\calS$. Here $\sigma_D(\bfcalO) = 0$ if $M< D$.
	Combining the above two equations gives
	\[
	\sin^2(\theta_0) \leq \frac{\sigma_1^2(\bfcalO) - \sigma_D^2(\bfcalO)}{\sigma_d^2(\bfcalX)}.
	\]

\subsubsection{Proof of  \Cref{cor:PSGM and Spectral initia-random model}}
\label{prf:PSGM and Spectral initia-random model}
The following results provide concentration inequalities for the singular values appeared in \eqref{eq:angle by spectral method} when the inliners and outliers are generated from a random spherical model.
\begin{lem}\cite[Theorem 5.39]{Vershynin2010:introduction}
	Let the columns of $\bfcalO\in\R^{D\times M}$  and $\bfcalX$ be drawn independently and uniformly at random from the unit sphere $\setS^{D-1}$ and the intersection of the unit sphere with a subspace $\calS$ of dimension $d<D$, respectively. Then for every $t>0$, there exist constants $C_1,C_2$ such that 
	\e\begin{split}
		&\P{\sigma_1(\bfcalO) \geq \frac{\sqrt{M} + C_2\sqrt{D}+t}{\sqrt{D}}} \leq e^{-C_1t^2},\\
		&\P{\sigma_D(\bfcalO) \leq \frac{\sqrt{M} - C_2\sqrt{D} -t}{\sqrt{D}}} \leq e^{-C_1 t^2},\\
		& \P{\sigma_d(\bfcalX) \leq \frac{\sqrt{N} - C_2\sqrt{d} -t}{\sqrt{d}}} \leq e^{-C_1 t^2}.
	\end{split}
	\label{eq:singular values random spherical model}\ee	
	\label{lem:singular values random spherical model}\end{lem}

Note that when $\bfcalO$ is a Gaussian random matrix whose entries are independent normal random variables of mean $0$ and variance $\frac{1}{D}$, according to \cite[Theorem 5.35]{Vershynin2010:introduction}, similar concentration inequalities as \eqref{eq:singular values random spherical model} hold but with constants $C_1 = C_2 =1$. We suspect that $C_1$ and $C_2$ are  also close to $1$ 
in \Cref{lem:singular values random spherical model}. 

We also require the following result from \Cref{thm:quantities-ranom-model}:
\begin{align*}
&\P{\eta_{\bfcalX} \leq C_0\left(1 + \sqrt{d}\log \left(\sqrt{c_d d}\right) + t\right)/\sqrt{N}}\geq 1-2e^{-\frac{t^2}{2}}, \\
&\P{\eta_{\bfcalO} \leq C_0\left(1 + \sqrt{D}\log \left(\sqrt{c_D D}\right) + t\right)/\sqrt{M}}\geq 1-2e^{-\frac{t^2}{2}}, \\
& \P{c_{\bfcalX,\min} \geq c_d  - (2+ \frac{t}{2}) /\sqrt{N}} \geq 1-2e^{-\frac{t^2}{2}},
\end{align*}
where the proof for $\eta_{\bfcalX}$ follows from similary argument for $\eta_{\bfcalO}$.
The proof is finished by utilizing the above quantities into \Cref{cor:PSGM and Spectral initia}.

\section{Experiments on Synthetic and Real $3$D Point Cloud Road Data}

\subsection{Numerical Evaluation of the Theoretical Conditions of \Cref{thm:lp to global} for the Alternating Linerization and Rrojection Method}

We begin with synthetic experiments to evaluate the theoretical conditions \eqref{eq:angle lp succed} and \eqref{eq:angle sequential lp succed} of \Cref{thm:lp to global}, under which the Alternating Linerization and Projection (ALP) method (see procedure in \eqref{eq:dpcp lp}) generates a normal vector that is orthogonal to the inlier subspace $\calS$ either in one iteration or a finite number of iterations. We also evaluate the procedure for ALP in \eqref{eq:dpcp lp} until the iteration is orthogonal to the inlier subspace $\calS$. As illustrated in \Cref{sec:initialization}, the spectral method provides a much better initialization, especailly when the inlier subpsace $\calS$ has high dimension. Thus, we use the spectral initialization (i.e., $\widehat \vb_0$ is the bottom eigenvector of $\widetilde \bfcalX \widetilde \bfcalX^\T$) throughout the experiments and check whether conditions \eqref{eq:angle lp succed} and \eqref{eq:angle sequential lp succed} are satisfied. Towards that end, we fix the ambient dimension $D = 30$ and randomly sample a subspace $\calS$ of of varying dimension $d$ from $5$ to $29$.  We uniformly at random sample $N = 500$ inliers from $\calS\cap \setS^{D-1}$  and $M$ samples from $\setS^{D-1}$ where $M$ is choosen so that the percentage of outliers $R$ varies from $0.1$ to $0.7$. 

\Cref{fig:check-ALP-N500-a} shows the angle $\phi_0$ between the spectral intialization $\widehat \vb_0$ and the inlier subspace $\calS$. We numerically estimate the parameters $c_{\bfcalX,\min}$ $c_{\bfcalO,\max}$, $c_{\bfcalO,\min}$ and $\eta_{\bfcalO}$ and then display the two angles $\phi^\star$ (defined in~\eqref{eq:angle lp succed}) and $\phi^\natural$ (defined in~\eqref{eq:angle sequential lp succed}) in \Cref{fig:check-ALP-N500-b} and \Cref{fig:check-ALP-N500-c}, respectively. In this figures, $0^\circ$ corresponds to black while $90^\circ$ corresponds to white. Now \Cref{fig:check-ALP-N500-d}   shows whether the condition \eqref{eq:angle lp succed}  is true (white if $\phi_0>\phi^\natural$) or not (black if $\phi_0\leq \phi^\natural$); similar result for  \eqref{eq:angle sequential lp succed} is plotted in \Cref{fig:check-ALP-N500-e}. We observe that despite the upper right corners corresponding to large subspace dimension and high outlier ratio, most part of \eqref{eq:angle lp succed} is white indicating that one procedure of \eqref{eq:dpcp lp} returns a normal vector. This is demonstrated in \Cref{fig:check-ALP-N500-f} which displays the angle $\phi_1$ between $\widehat \vb_1$ and $\calS$. As we observed, for most cases in \Cref{fig:check-ALP-N500-f}, $\widehat\vb_1$ is orthogonal to $\calS$, in agreement with \Cref{thm:lp to global}. We continue the procedure of \eqref{eq:dpcp lp} and plot the angles of $\widehat \vb_2$ and $\widehat \vb_3$ in \Cref{fig:check-ALP-N500-g} and \Cref{fig:check-ALP-N500-h}, respectively. Althouth there are few cases that \eqref{eq:angle sequential lp succed} is not satisfied in \Cref{fig:check-ALP-N500-f}, \Cref{fig:check-ALP-N500-h} indicates that ALP finds a normal vector in three iterations, suggesting that the condition \eqref{eq:angle sequential lp succed} is slightly stronger than necessary and leaving room for future theoretical improvements.

As explained in the  discussion right after \Cref{thm:lp to global}, when the inliers and outliers are well distributed, for any fixed outlier ratio both $\phi^\natural$ and $\phi^\star$ are expected to decrease as $N$ increases, regardless of the relative subspace dimension $\frac{d}{D}$. We now conduct similar experiments by increasing $N$ up to $1500$ and show the results in \Cref{fig:check-ALP-N500}. It is interestig to note that as guaranteed by \Cref{fig:check-ALP-N1500-d}, ALP successfully returns a normal vector only in one iteration, as shown in \Cref{fig:check-ALP-N1500-f}.

\begin{figure}[htb!]
	\begin{subfigure}{0.32\linewidth}
	\centerline{
		\includegraphics[width=2in]{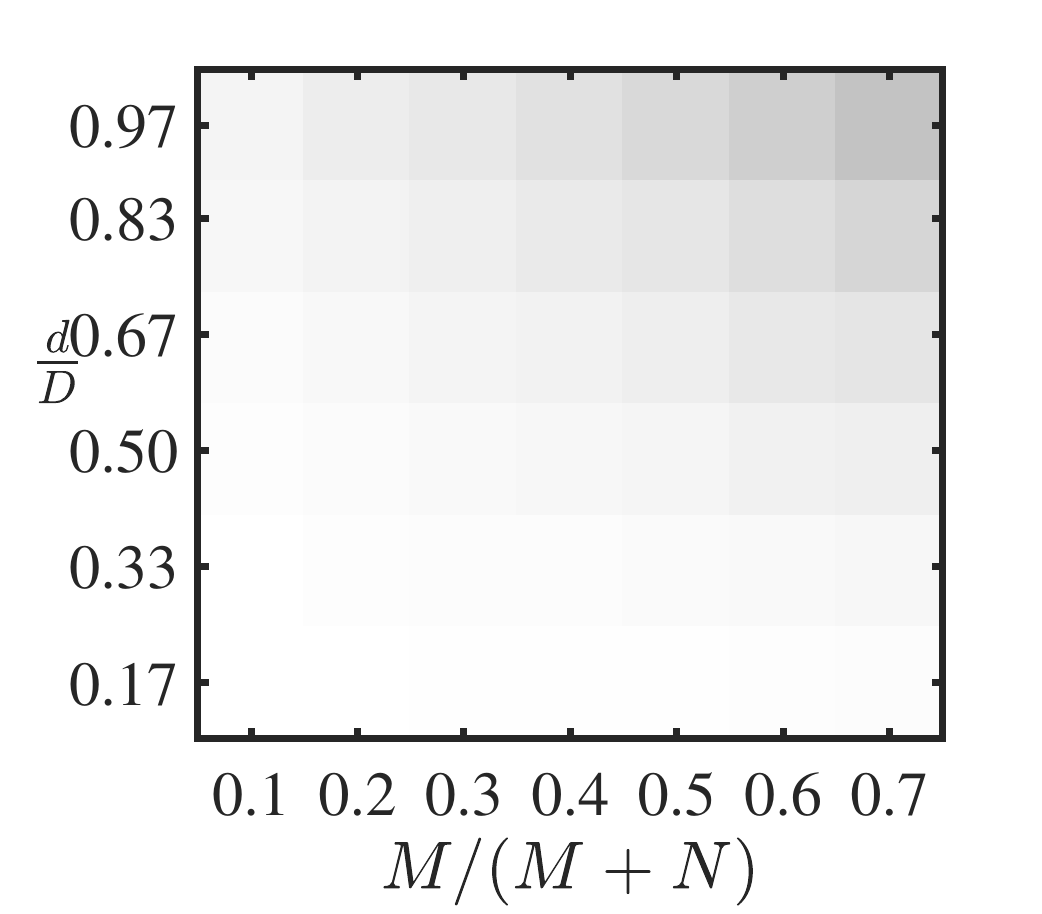}}
	\caption{$\phi_0$}
	\label{fig:check-ALP-N500-a}
\end{subfigure}	
	\hfill
		\begin{subfigure}{0.32\linewidth}
		\centerline{
			\includegraphics[width=2in]{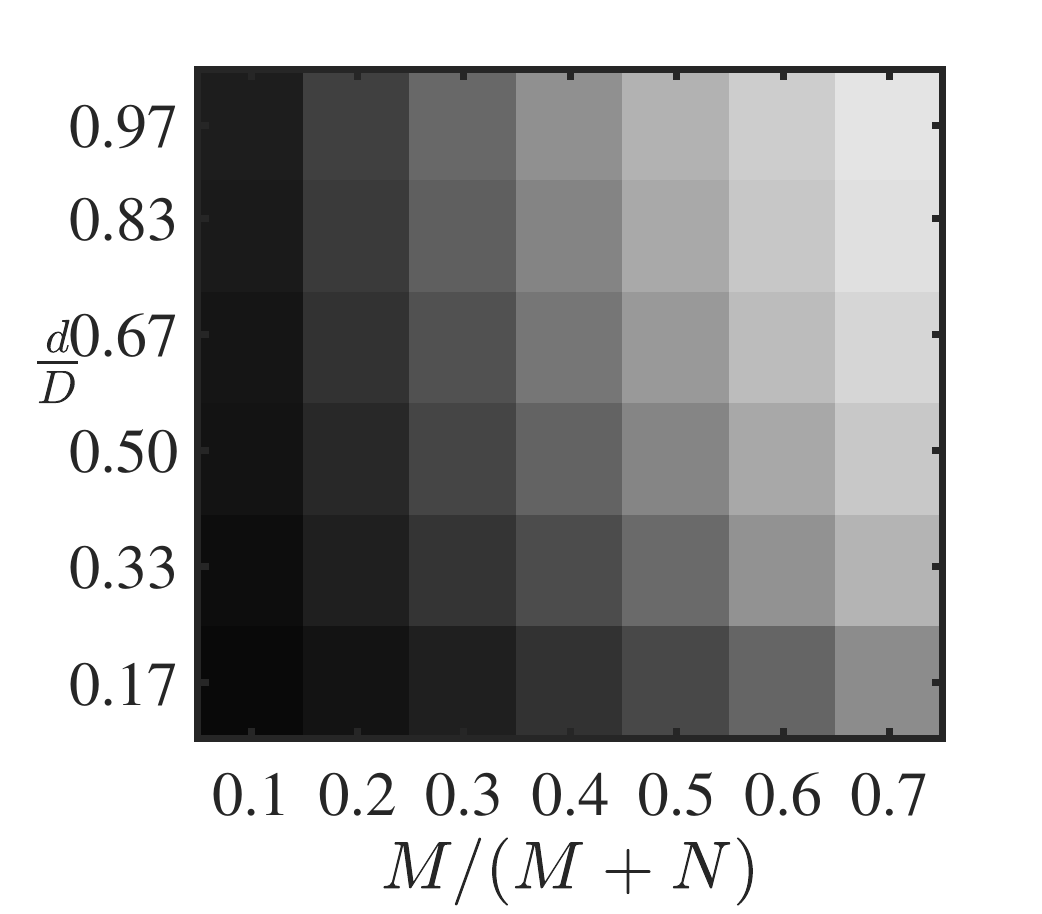}}
		\caption{$\phi^\natural$ in \eqref{eq:angle lp succed}}
			\label{fig:check-ALP-N500-b}
	\end{subfigure}	
\hfill
		\begin{subfigure}{0.32\linewidth}
	\centerline{
		\includegraphics[width=2in]{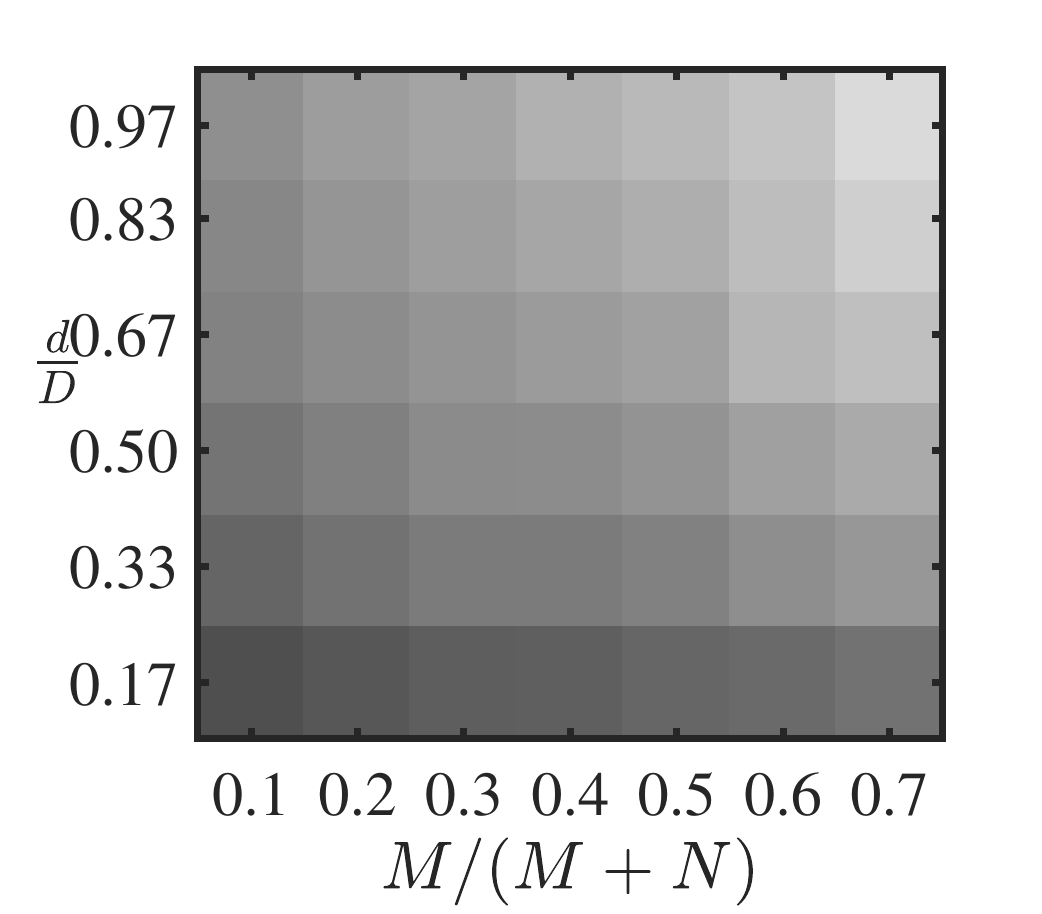}}
	\caption{$\phi^\star$ in \eqref{eq:angle sequential lp succed}}
		\label{fig:check-ALP-N500-c}
\end{subfigure}	
\vfill
\begin{subfigure}{0.32\linewidth}
	\centerline{
		\includegraphics[width=2in]{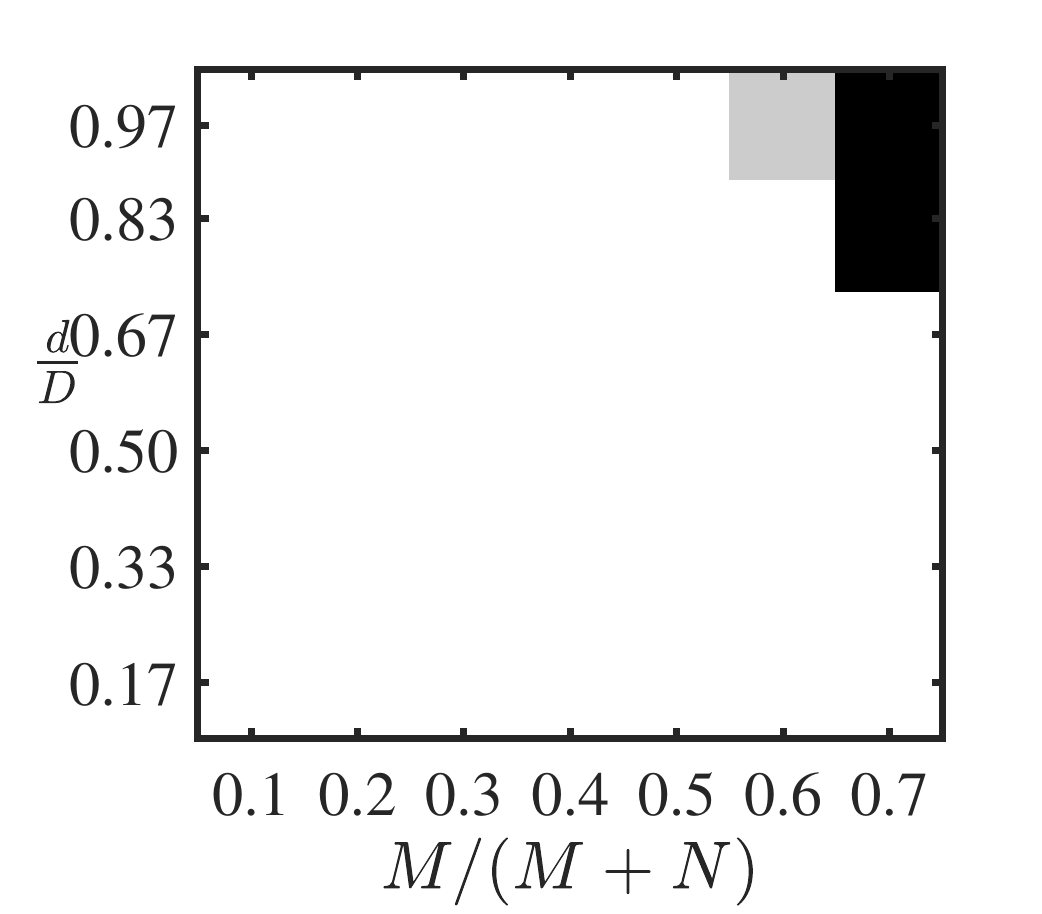}}
	\caption{check \eqref{eq:angle lp succed}}
		\label{fig:check-ALP-N500-d}
\end{subfigure}	
\hfill
\begin{subfigure}{0.32\linewidth}
	\centerline{
		\includegraphics[width=2in]{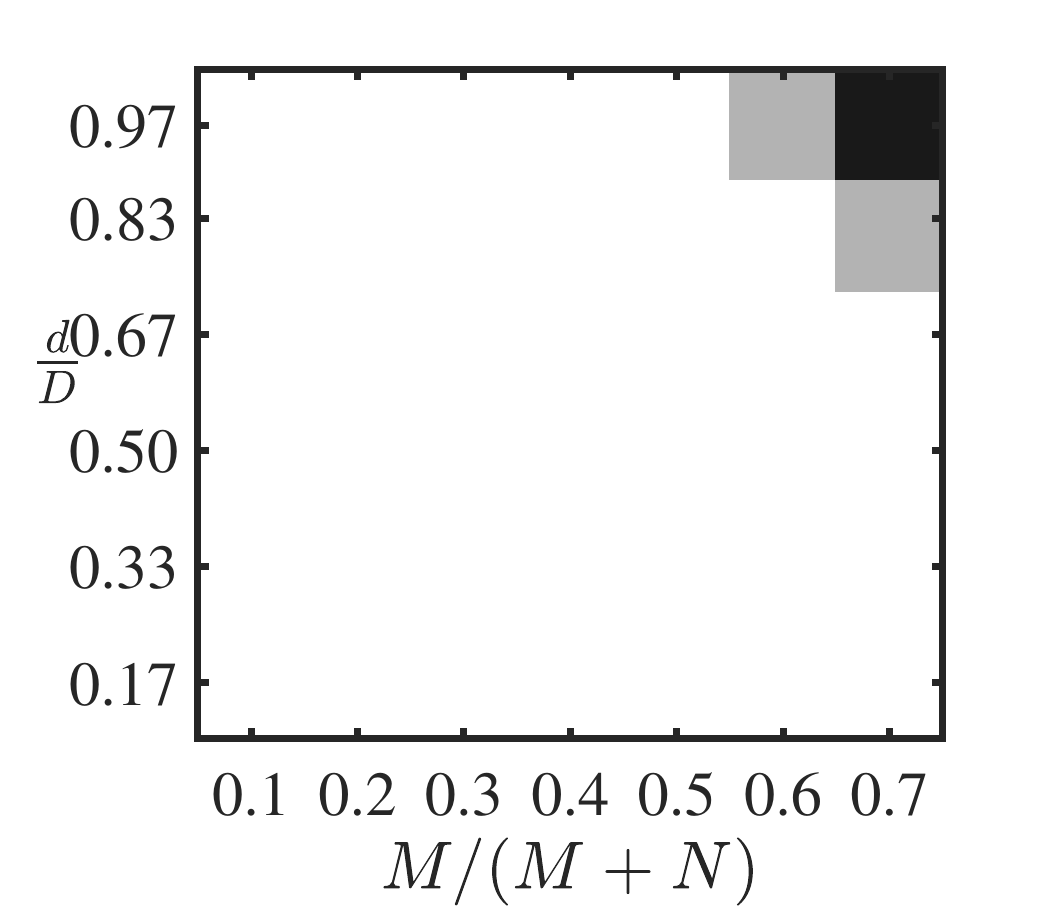}}
	\caption{check \eqref{eq:angle sequential lp succed}}
		\label{fig:check-ALP-N500-e}
\end{subfigure}	
\hfill
	\begin{subfigure}{0.32\linewidth}
	\centerline{
		\includegraphics[width=2in]{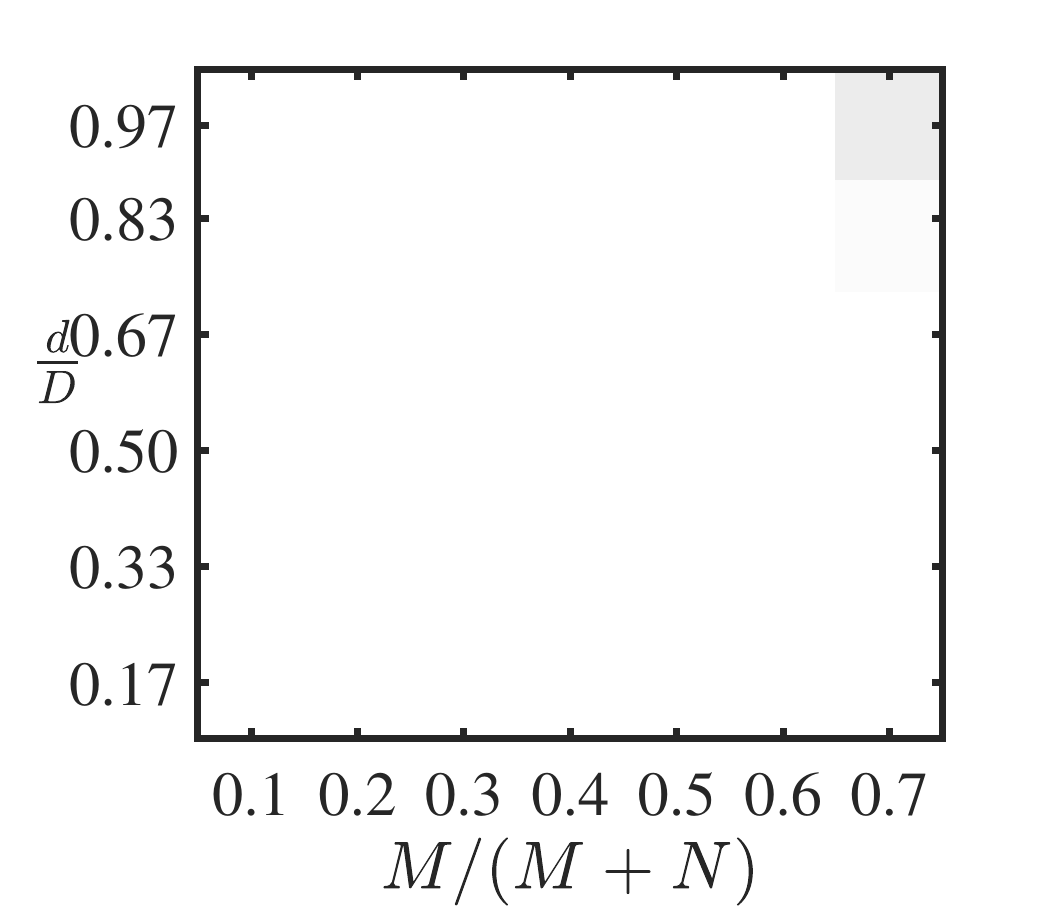}}
	\caption{$\phi_1$}
	\label{fig:check-ALP-N500-f}
\end{subfigure}	
\vfill
	\begin{subfigure}{0.32\linewidth}
	\centerline{
		\includegraphics[width=2in]{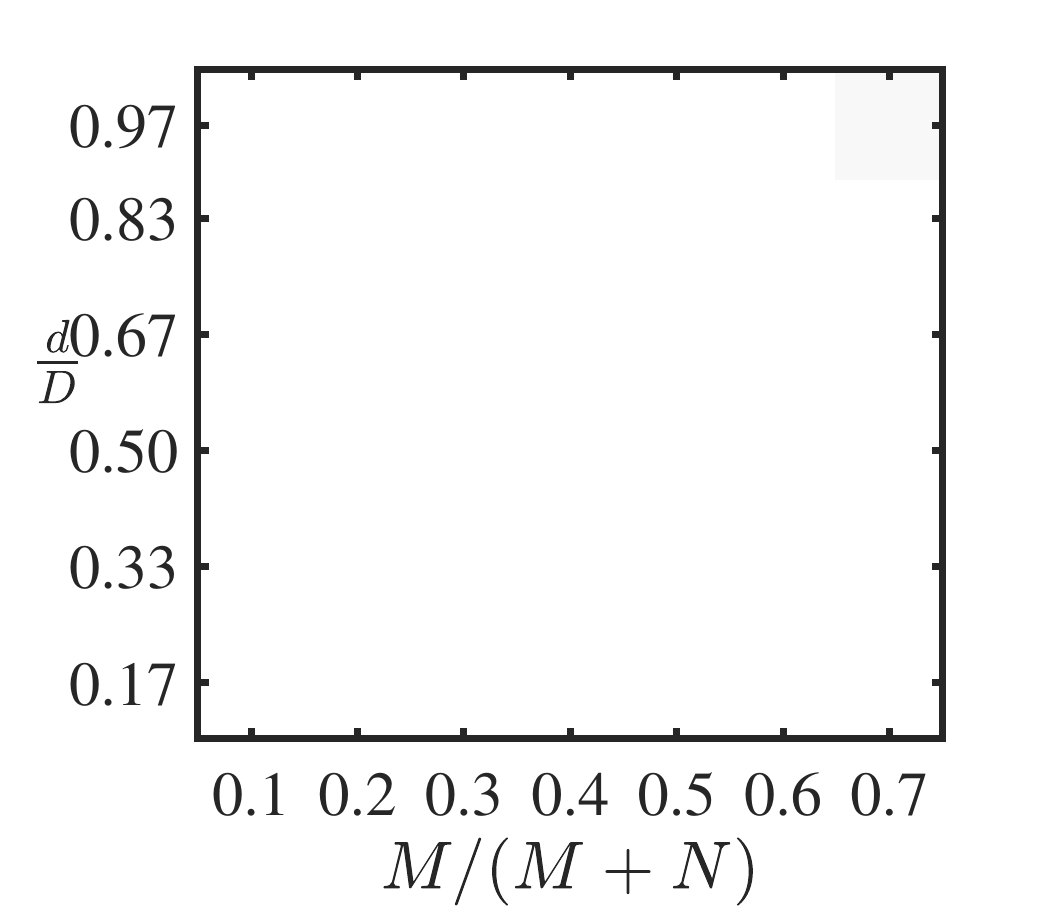}}
	\caption{$\phi_2$}
		\label{fig:check-ALP-N500-g}
\end{subfigure}	
	\begin{subfigure}{0.32\linewidth}
	\centerline{
		\includegraphics[width=2in]{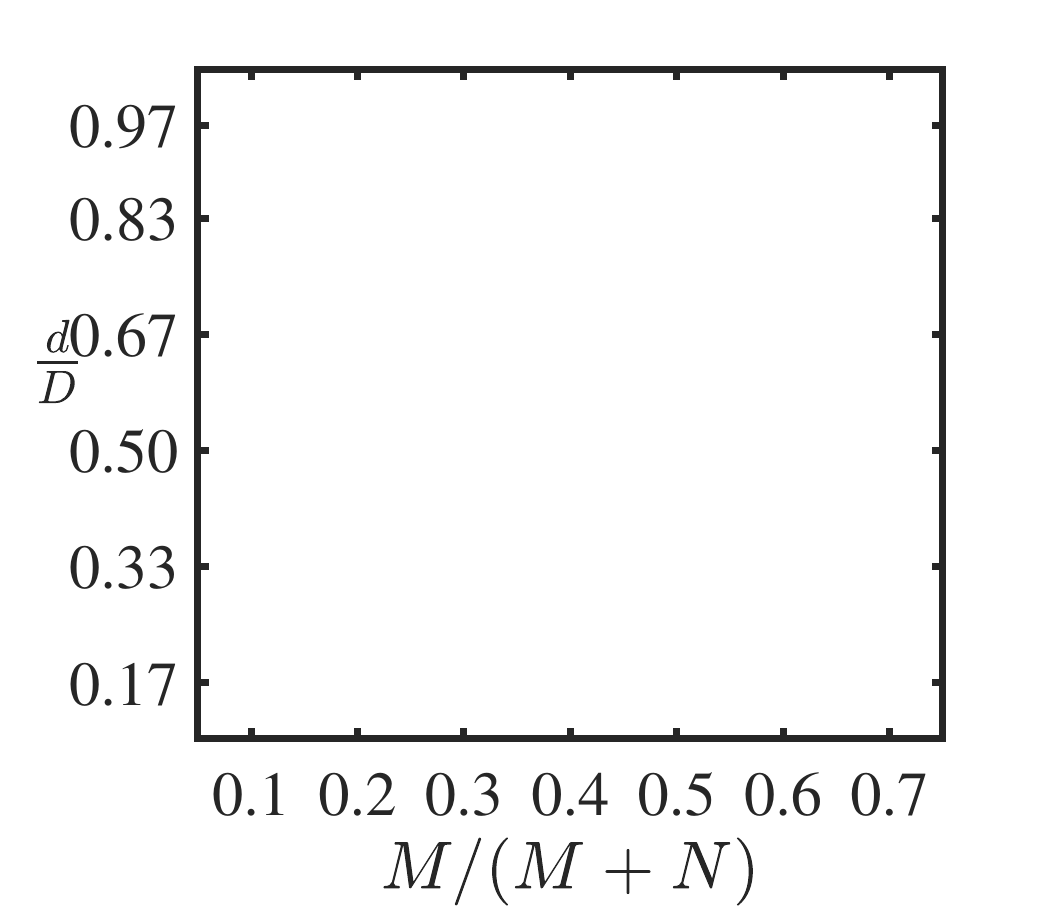}}
	\caption{$\phi_3$}
		\label{fig:check-ALP-N500-h}
\end{subfigure}	
\caption{(a) the angle $\phi_0$ ($0^\circ$ corresponds to black and $90^\circ$ corresponds to white) of $\widehat\vb_0$ from the subspace $\calS$ where $\vb_0$ is obtained via the spectral method in \Cref{lem:spectral initialization}, (b) the angle $\phi^\star$ in~\eqref{eq:angle lp succed}, (c) the angle $\phi^\natural$ in~\eqref{eq:angle sequential lp succed},  (d) whether the sufficient condition \eqref{eq:angle lp succed} is satisfied (white if $\phi_0>\phi^\natural$) or not (black if $\phi_0\leq \phi^\natural$), (e) whether the sufficient condition  \eqref{eq:angle sequential lp succed}  is satisfied (white if $\phi_0>\phi^\star$) or not (black if $\phi_0\leq \phi^\star$), (f) the angle $\phi_1$ of $\widehat \vb_1$ in \eqref{eq:dpcp lp}, (g) the angle $\phi_2$ of $\widehat \vb_2$ in \eqref{eq:dpcp lp}, (h) the angle $\phi_3$  of $\widehat \vb_3$ in \eqref{eq:dpcp lp}. Here, the vertical axis is the relative inlier subspace dimension $\frac{d}{D}$, while the horizontal axis represents the outlier ratio. The dimension of the ambient space is $D = 30$ and the number of inliers is $N = 500$. Results are averaged over 10 independent trails.
	}\label{fig:check-ALP-N500}
\end{figure}

\begin{figure}[htb!]
	\begin{subfigure}{0.32\linewidth}
	\centerline{
		\includegraphics[width=2in]{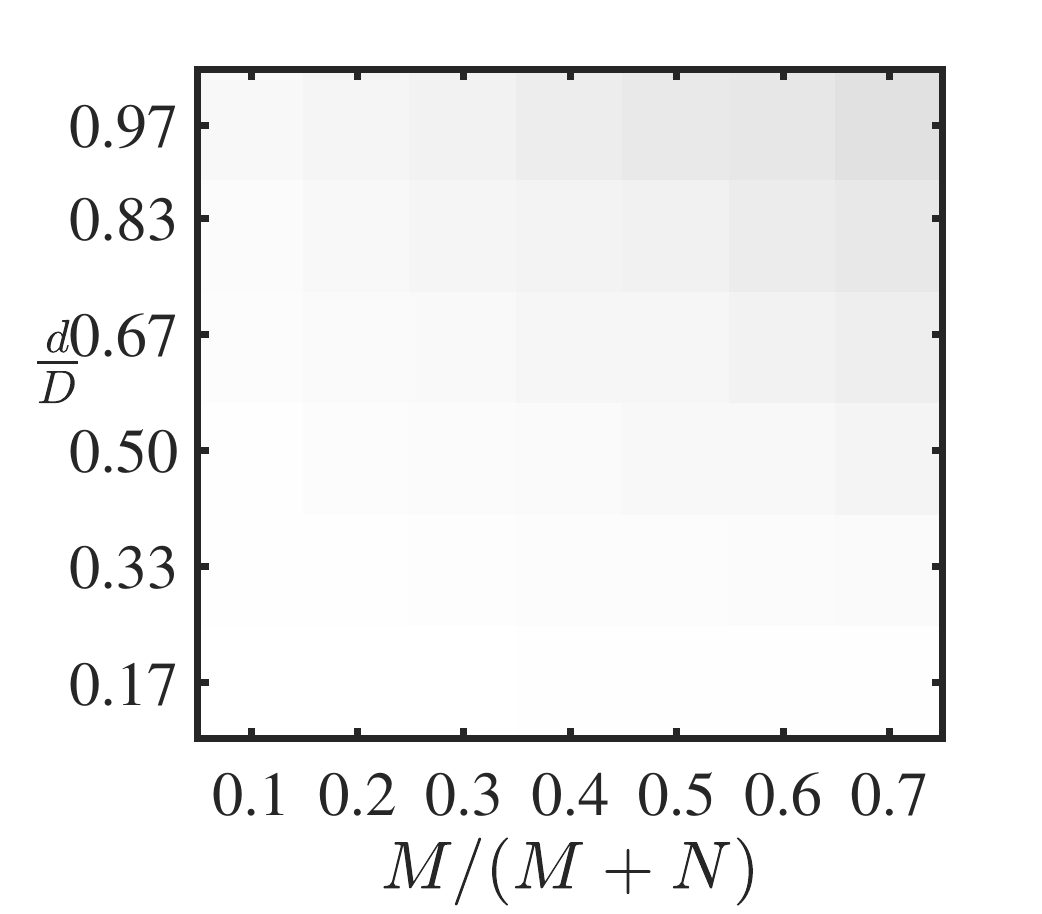}}
	\caption{$\phi_0$}
\end{subfigure}	
\hfill
\begin{subfigure}{0.32\linewidth}
	\centerline{
		\includegraphics[width=2in]{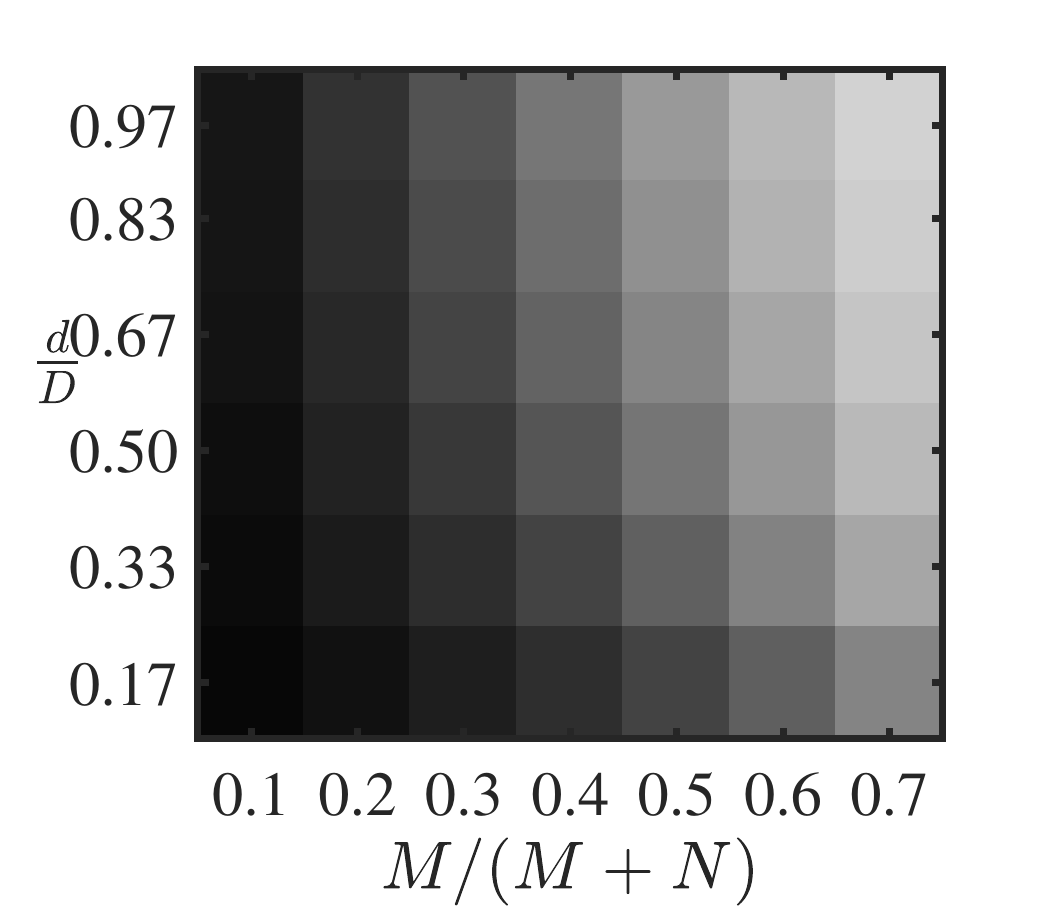}}
	\caption{$\phi^\natural$ in \eqref{eq:angle lp succed}}
\end{subfigure}	
\hfill
\begin{subfigure}{0.32\linewidth}
	\centerline{
		\includegraphics[width=2in]{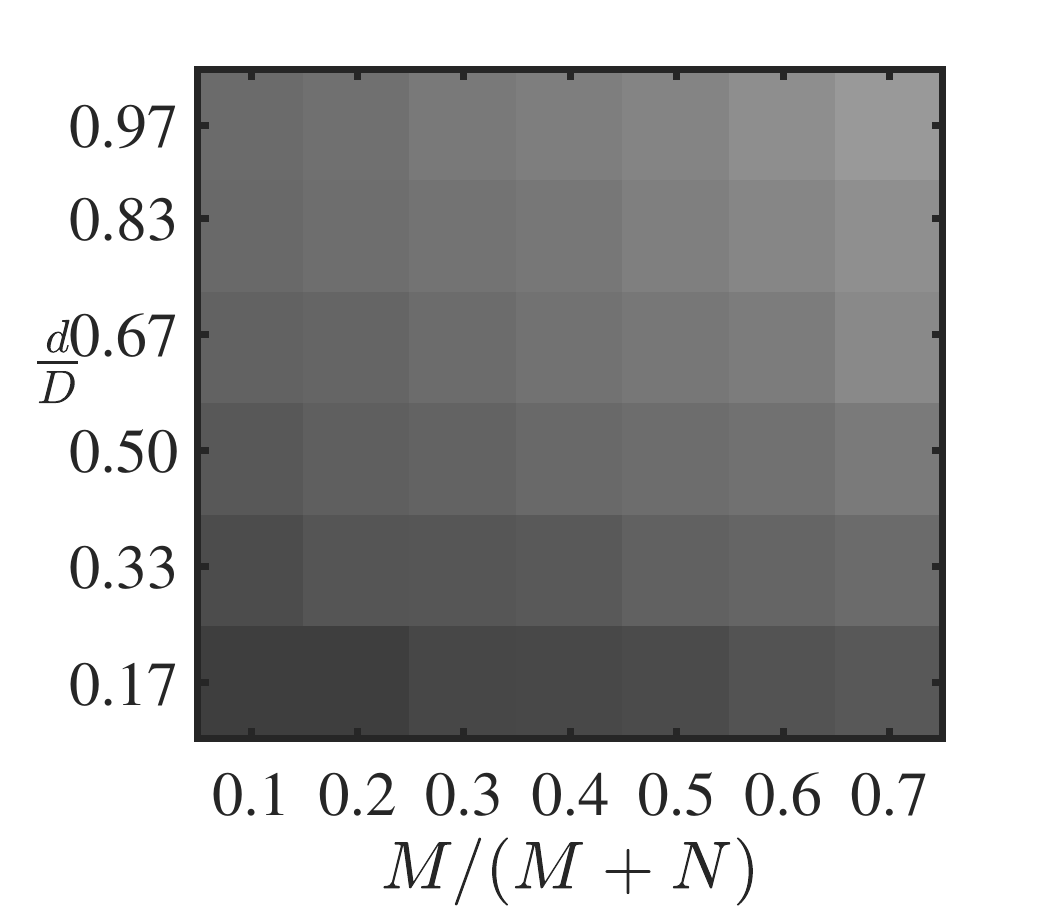}}
	\caption{$\phi^\star$ in \eqref{eq:angle sequential lp succed}}
\end{subfigure}	
\vfill
\begin{subfigure}{0.32\linewidth}
	\centerline{
		\includegraphics[width=2in]{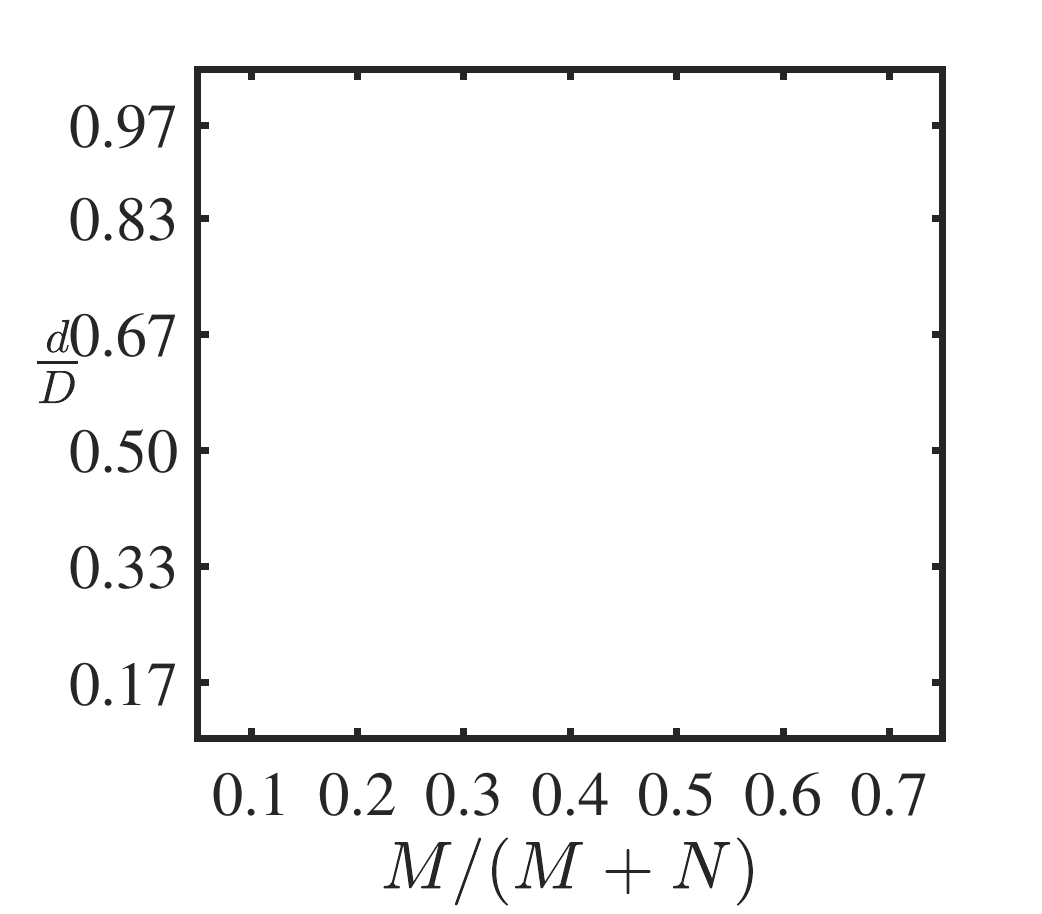}}
	\caption{check \eqref{eq:angle lp succed}}
	\label{fig:check-ALP-N1500-d}
\end{subfigure}	
\hfill
\begin{subfigure}{0.32\linewidth}
	\centerline{
		\includegraphics[width=2in]{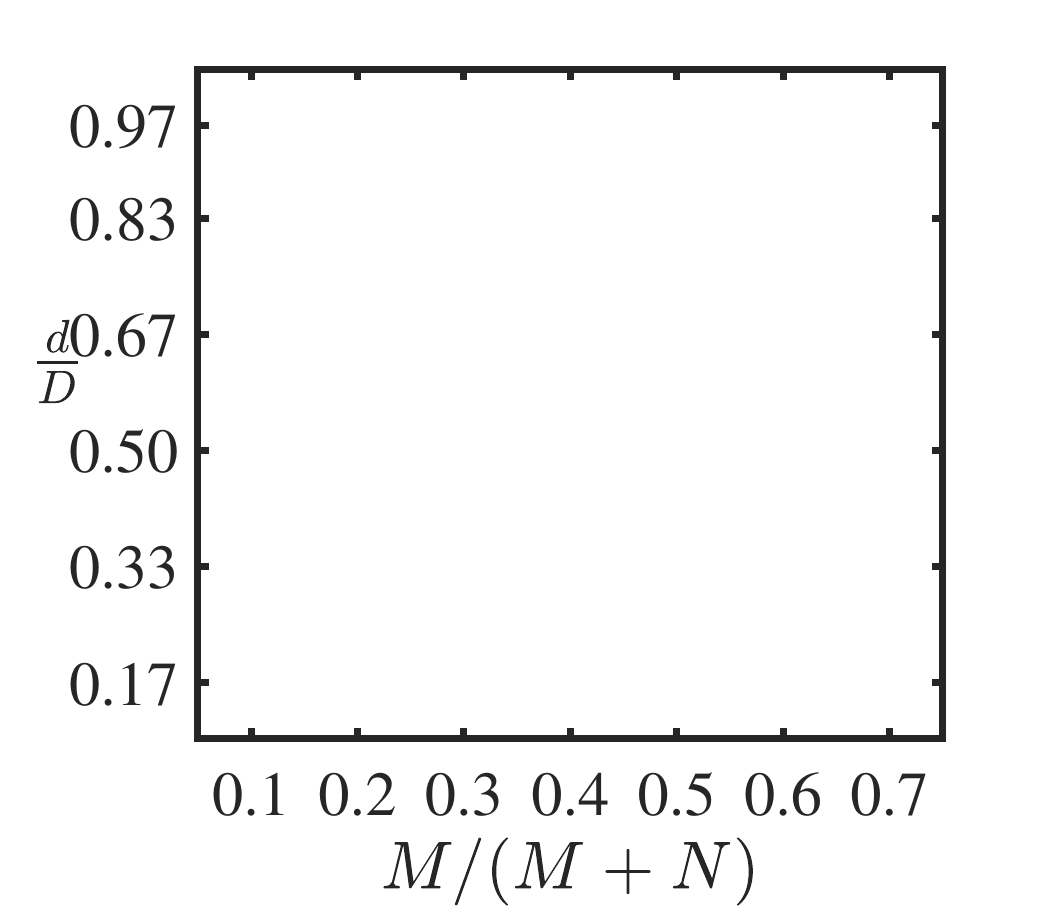}}
	\caption{check \eqref{eq:angle sequential lp succed}}
	\label{fig:check-ALP-N1500-e}
\end{subfigure}	
\hfill
\begin{subfigure}{0.32\linewidth}
	\centerline{
		\includegraphics[width=2in]{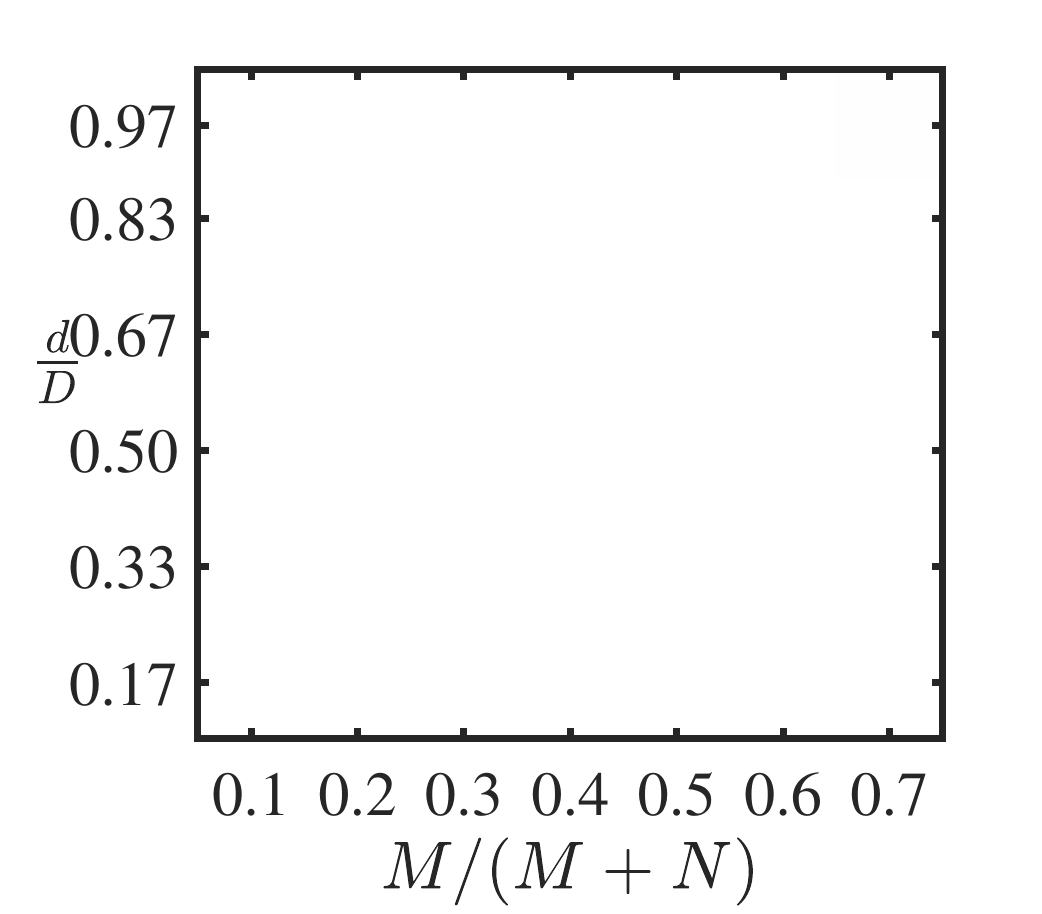}}
	\caption{$\phi_1$}
	\label{fig:check-ALP-N1500-f}
\end{subfigure}	
	\caption{Similar to \Cref{fig:check-ALP-N500} except that $N = 1500$. 
	}\label{fig:check-ALP-N1500}
\end{figure}

\subsection{Demonstration of the Convergence of the PSGM}

\begin{figure}
		\begin{subfigure}{0.32\linewidth}
		\centerline{
			\includegraphics[width=2.2in]{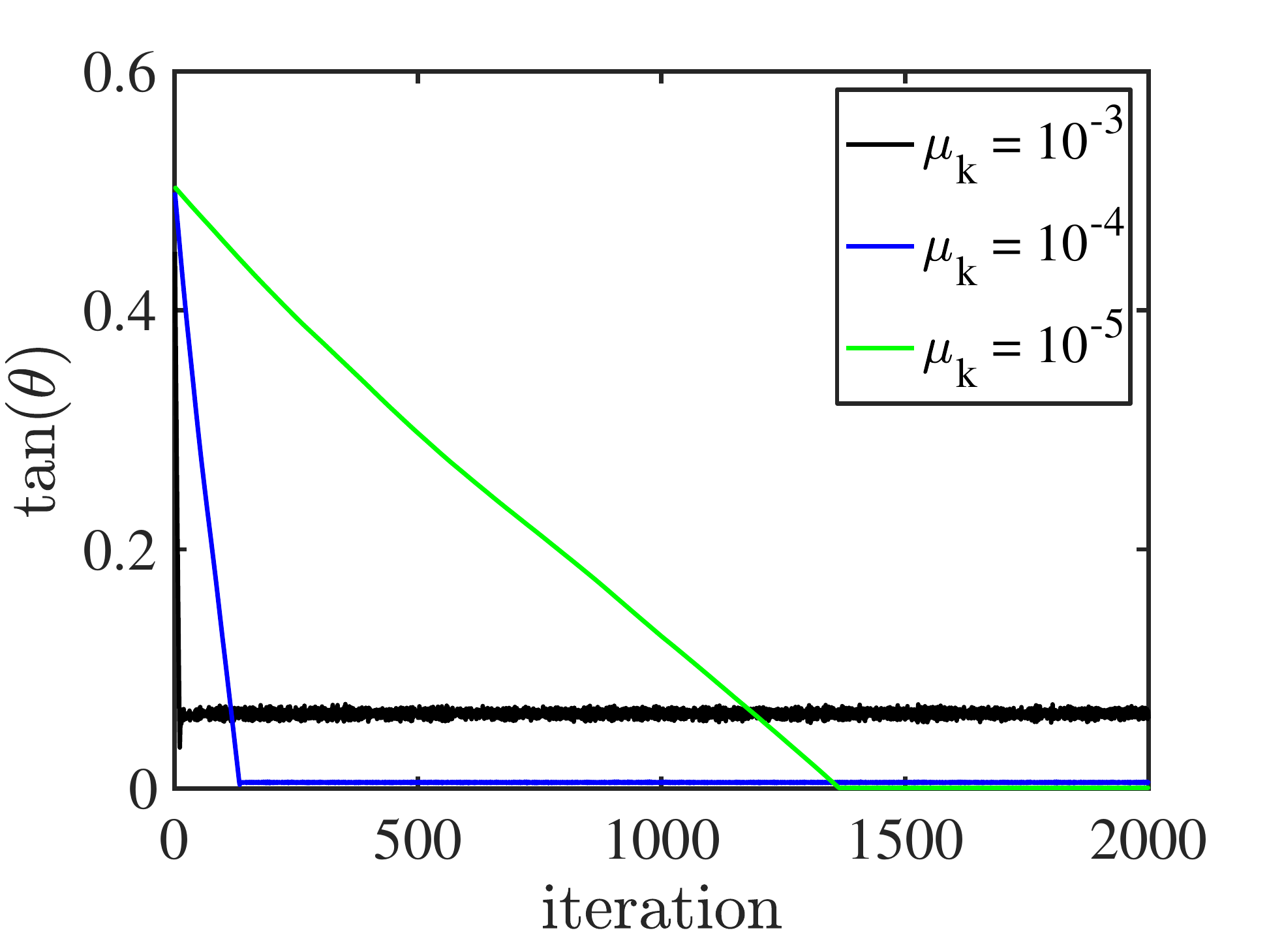}}
		\caption{}
		\label{fig:PSGM convergence-a}
	\end{subfigure}
	\hfill
	\begin{subfigure}{0.32\linewidth}
	\centerline{
		\includegraphics[width=2.2in]{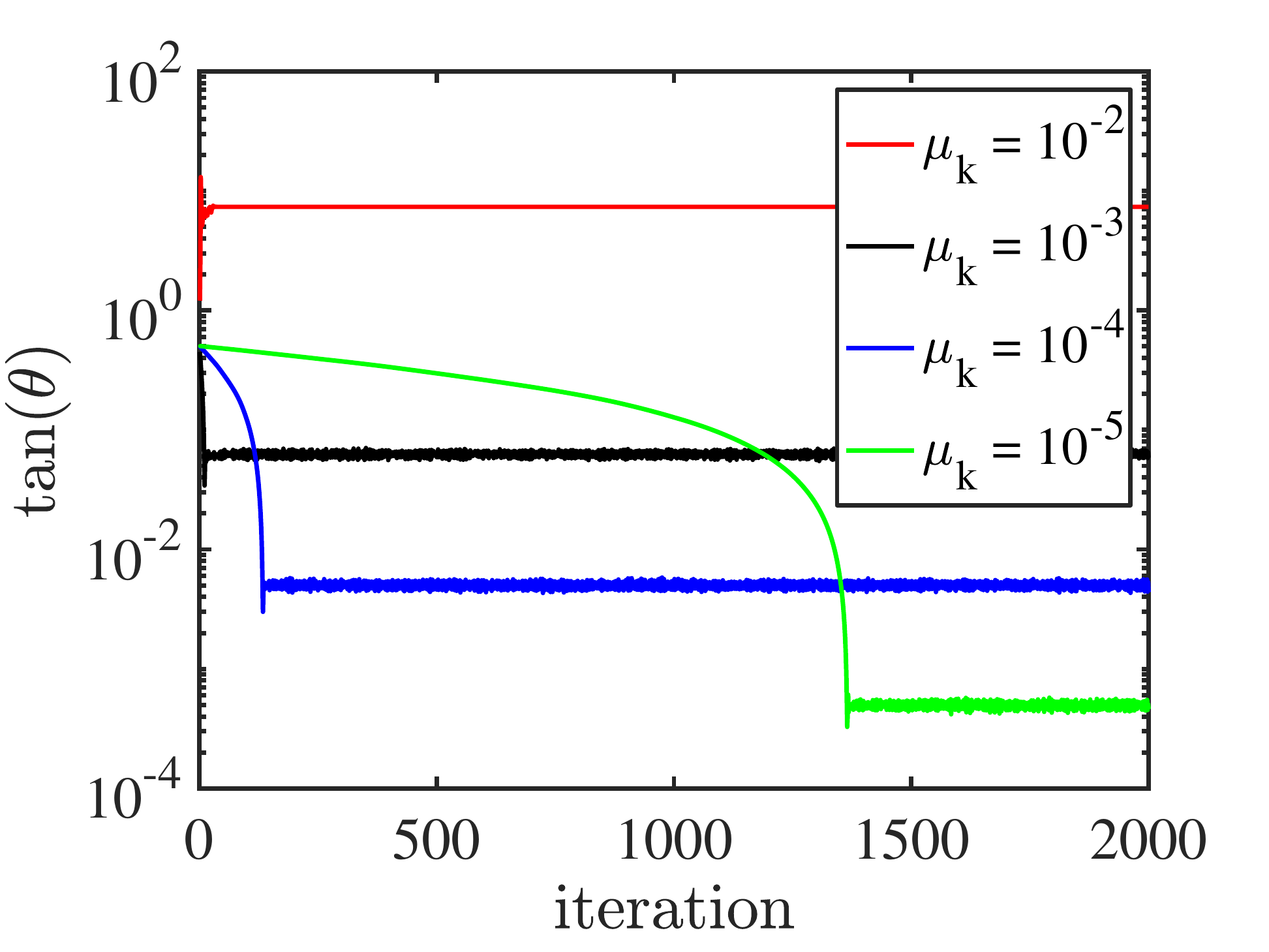}}
	\caption{}
	\label{fig:PSGM convergence-b}
\end{subfigure}
		\hfill
	\begin{subfigure}{0.32\linewidth}
		\centerline{
			\includegraphics[width=2.2in]{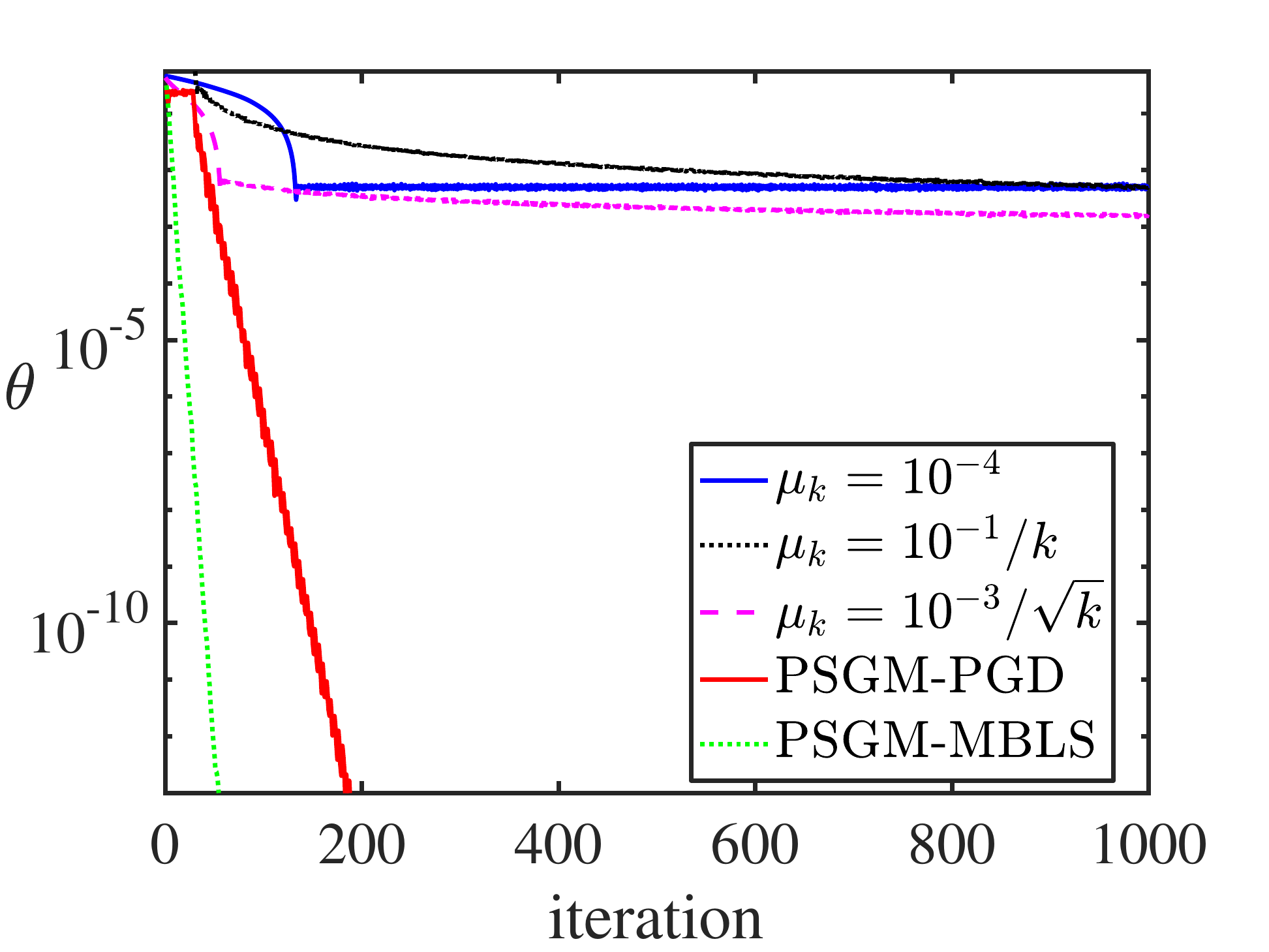}}
		\caption{}
		\label{fig:PSGM convergence-c}
	\end{subfigure}
	\caption{\small Convergence of PSGM with (a) constant step sizes (in the normal scale), (b) constant step sizes (in the log scale),  (c) different choices of diminishing step sizes. We omit the result for $\mu_k = 10^{-2}$ in (a) because it results in a relatively large $\theta$.  Here $D = 30, d = 29, N = 500, \frac{M}{M + N}= 0.7$. }\label{fig:PSGM convergence}
\end{figure}

We now use synthetic experiments to verify the proposed PSGM algorithms. Similar to the previous setup, we fix $D = 30$, randomly sample a subspace $\calS$ of dimension $d = 29$, and uniformly at random sample $N = 500$ inliers and $M = 1167$ outliers (so that the outlier ratio $M/(M+N) = 0.7$) with unit $\ell_2$-norm. Inspired by the Piecewise Geometrically Diminishing (PGD) step sizes, we also use a modified backtracking line search (MBLS) that always uses the previous step size as an initialization for finding the current one within a backtracking line search~\cite[Section 3.1]{Nocedal:06} strategy, which dramatically reduces the computational time compared with a standard backtracking line search. The corresponding algorithm is denoted by PSGM-MBLS. We set $K_0 = 30$, $K = 4$ and $\beta = 1/2$ for PGD step sizes with initial step size obtained by one iteration of a backtracking line search and denote the corresponding algorithm by PSGM-PGD. We define $\widehat\vb_0$ to be the bottom eigenvector of $\widetilde \bfcalX\widetilde\bfcalX^\top$. \Cref{fig:PSGM convergence} displays the convergence of the PSGM (see \Cref{alg:PSGM}) with different choices of step sizes. As we observed from \Cref{fig:PSGM convergence-a} on the constant step sizes,  at the begining $\tan(\theta_k)$  decreases almost at a certain rate (which is proportional to $\mu_k = \mu$) in each iteration until it reaches a certain level (that is also proportional to $\mu_k = \mu$), and then it  bounds around under this level, in coincidence with  the analysis in \Cref*{thm:convergence PSGM all}. \Cref{fig:PSGM convergence-c} shows the convergence of the PSGM with different choices of diminishing step sizes. We observe linear convergence for both PSGM-PGD and PSGM-MBLS, which converge much faster than PSGM with constant step sizes or classical diminishing step sizes.

We use synthetic experiments under different settings to further verify the proposed PSGM algorithm with piecewise exponentionally diminishing  step sizes.  \Cref{fig:PSGM-piecewise-exponentionally-diminishing} displays the convergence of $\theta$ (to $0$)  with different $d$, $D$, $N$ and outlier ratio $\gamma$. In particular, \Cref{fig:PSGM-varyd} shows the convergence of $\theta$ with $D = 30, N = 500, \gamma = 0.7$ and different subspace dimension $d$. We observe $R$-linear convergence in this case, irrespectively the subspace dimension $d$. \Cref{fig:PSGM-varyd-large} displays similar results but with larger $D$ and $N$. In \Cref{fig:PSGM-varyM}, we set $D = 30, d = 29, N = 500$ and vary the outlier ratio $\gamma$ from $0.1$ to $0.9$. We observe $R$-linear convergence expept for the case $\gamma = 0.9$, in which we have much more outliers than inliers. Interestingly, as shown in \Cref{fig:PSGM-varyM-large}, when we increase to $N = 1500$ and keep the other parameters the same as in \Cref{fig:PSGM-varyM}, the PSGM algorithm has  $R$-linear convergence even for $\gamma = 0.9$. This coincides with the fact that the larger $N$, the more likely the condition \eqref{eq:condition for PSGM} is satisfied. Finally we  display experiments with varied $N$ in \Cref{fig:PSGM-varyN} and \Cref{fig:PSGM-varyN-large}. We also observe $R$-linear convergence for PSGM with piecewise exponentionally diminishing  step sizes given sufficient number of inliers.

\begin{figure}[htb!]
	\begin{subfigure}{0.32\linewidth}
		\centerline{
			\includegraphics[width=2.2in]{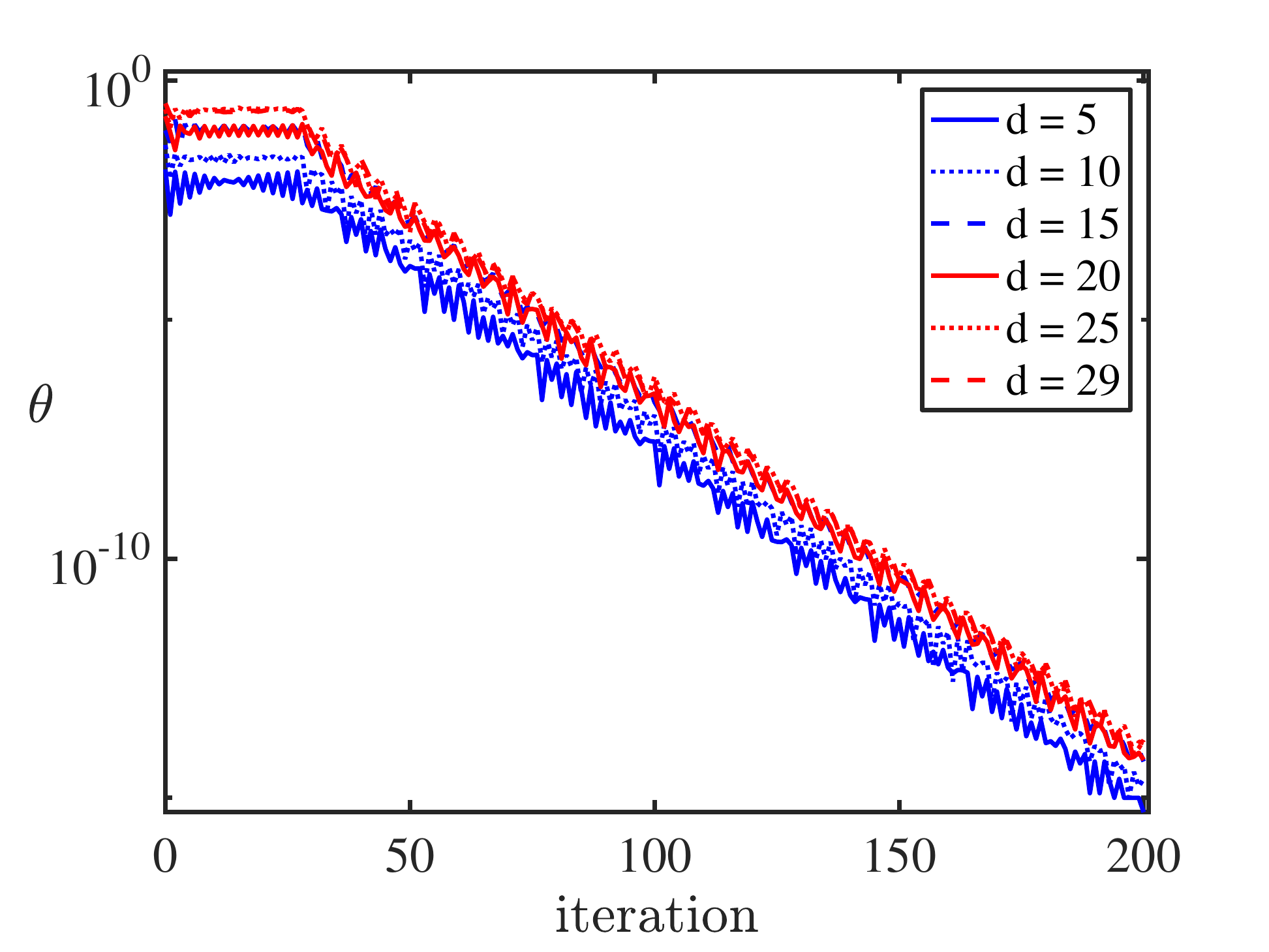}}
		\caption{$D = 30, N = 500, \gamma = 0.7$}
		\label{fig:PSGM-varyd}	\end{subfigure}
	\hfill
	\begin{subfigure}{0.32\linewidth}
		\centerline{
			\includegraphics[width=2.2in]{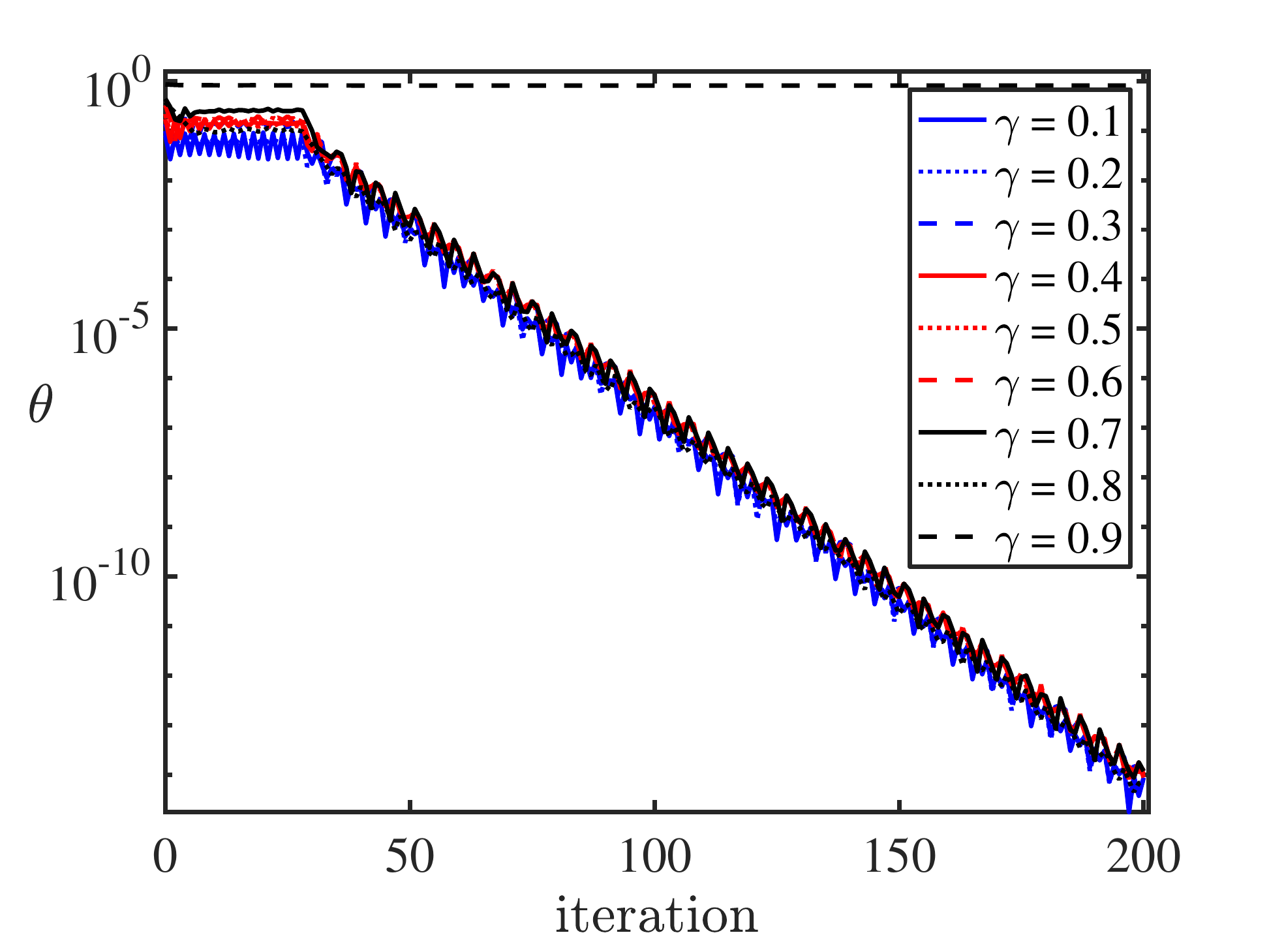}}
		\caption{$D = 30, d = 29, N = 500$}
		\label{fig:PSGM-varyM}	\end{subfigure}
	\hfill
	\begin{subfigure}{0.32\linewidth}
		\centerline{
			\includegraphics[width=2.2in]{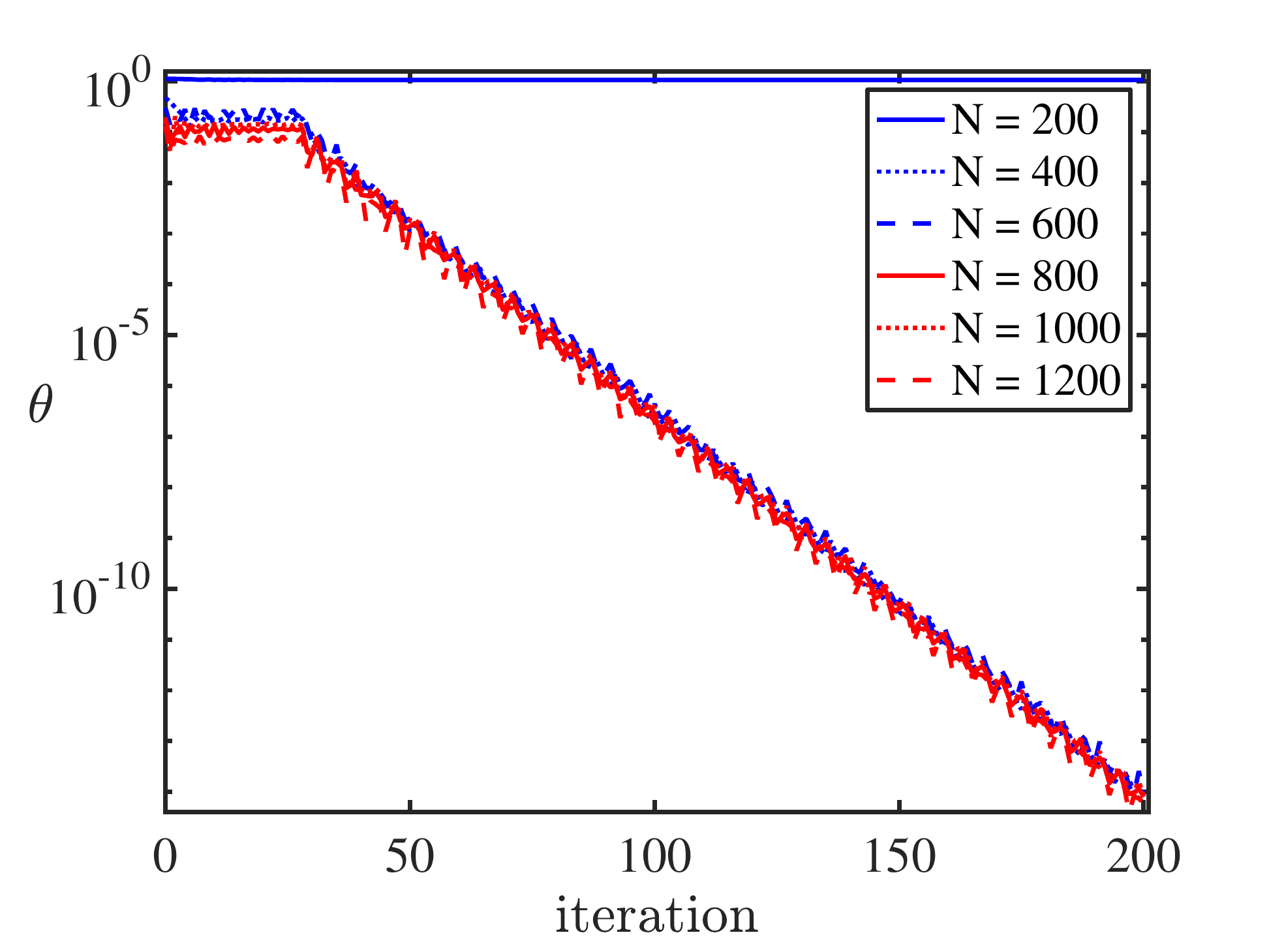}}
		\caption{$D = 30, d = 29, M = 1000$}
		\label{fig:PSGM-varyN}	\end{subfigure}
	\vfill
	\begin{subfigure}{0.32\linewidth}
		\centerline{
			\includegraphics[width=2.2in]{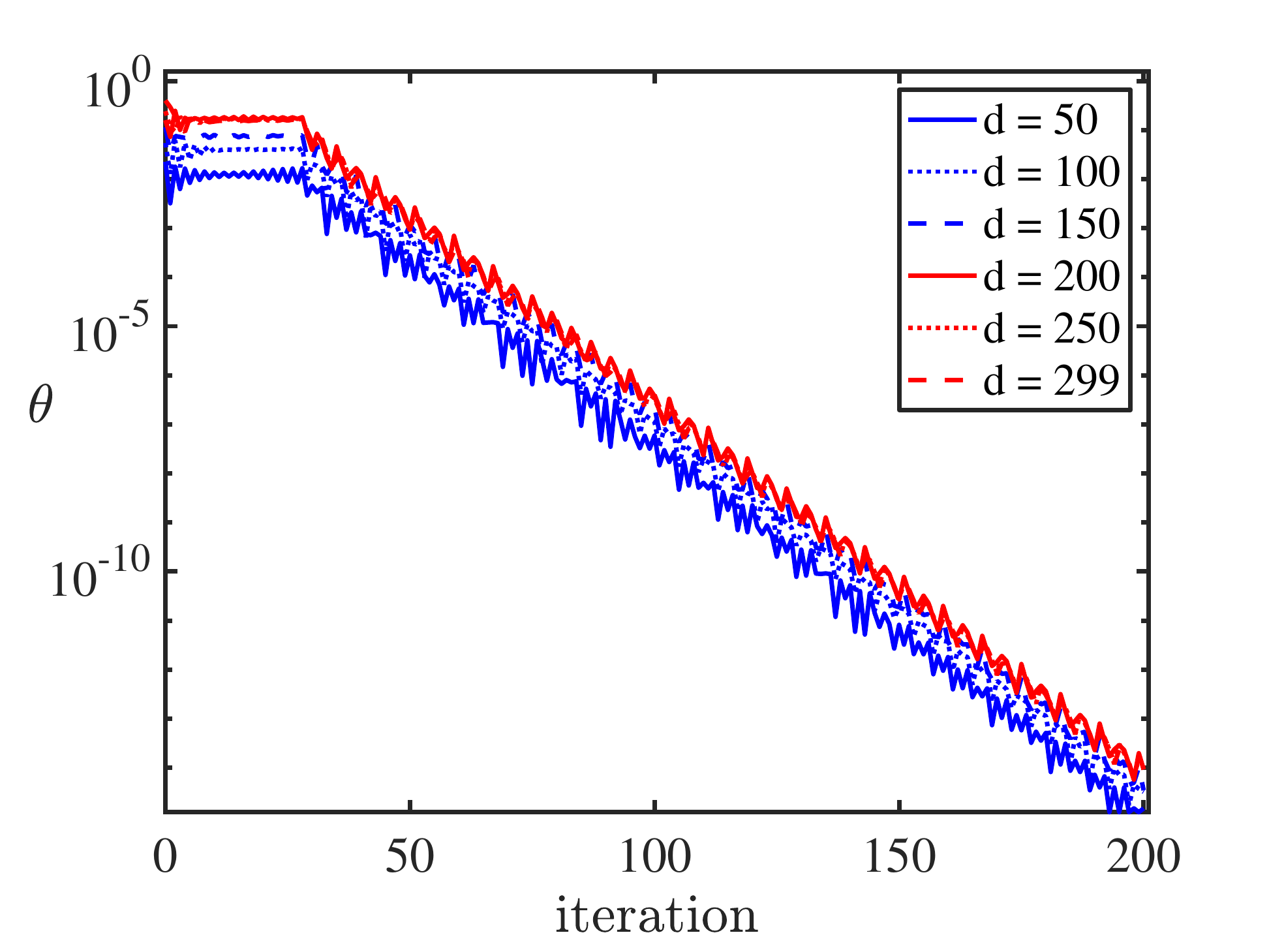}}
		\caption{$D = 300, N = 5000, \gamma  = 0.7$}
		\label{fig:PSGM-varyd-large}	\end{subfigure}
	\hfill
	\begin{subfigure}{0.32\linewidth}
		\centerline{
			\includegraphics[width=2.2in]{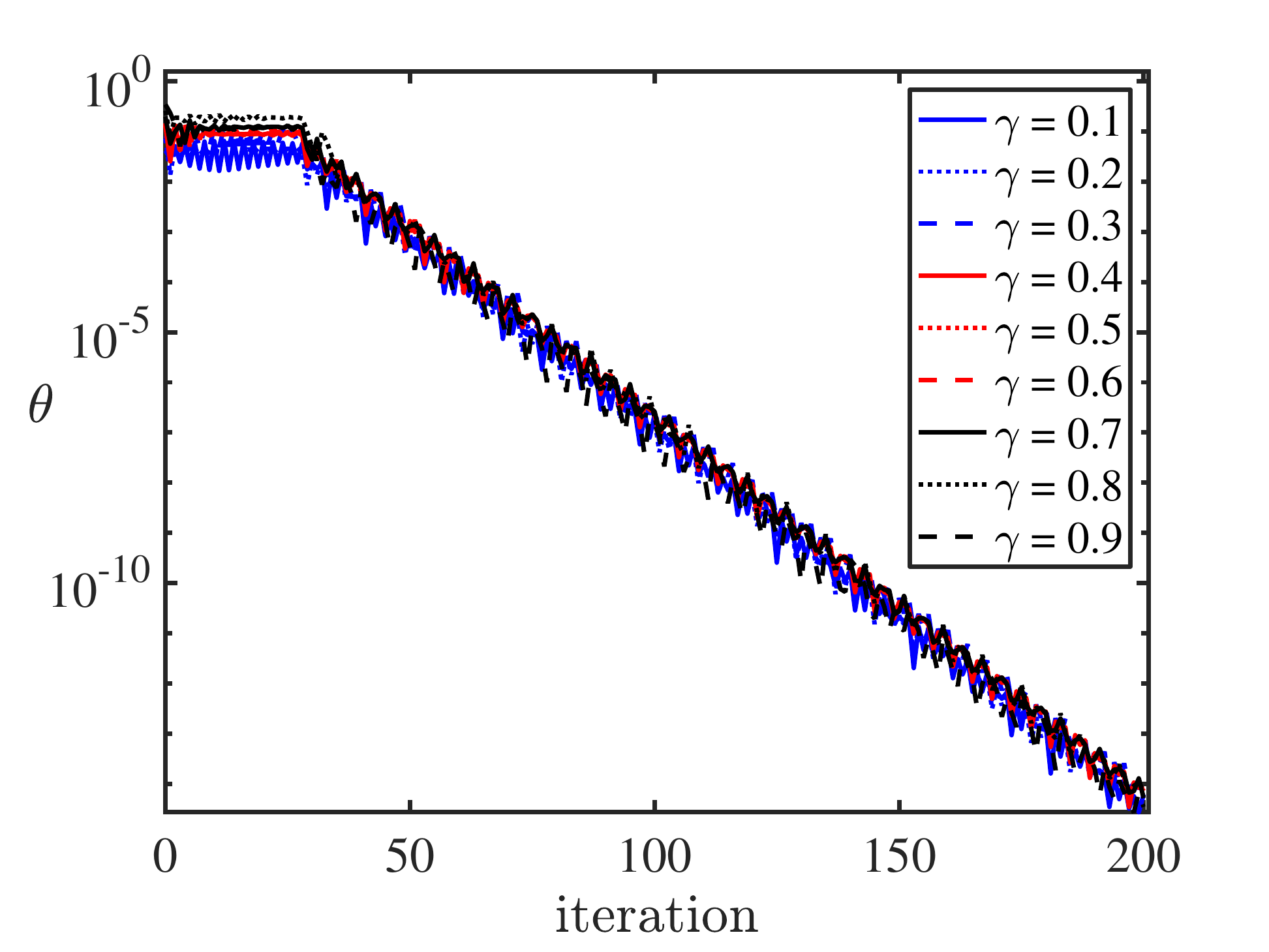}}
		\caption{$D = 30, d = 29, N = 1500$}
		\label{fig:PSGM-varyM-large}	\end{subfigure}
	\hfill
	\begin{subfigure}{0.32\linewidth}
		\centerline{
			\includegraphics[width=2.2in]{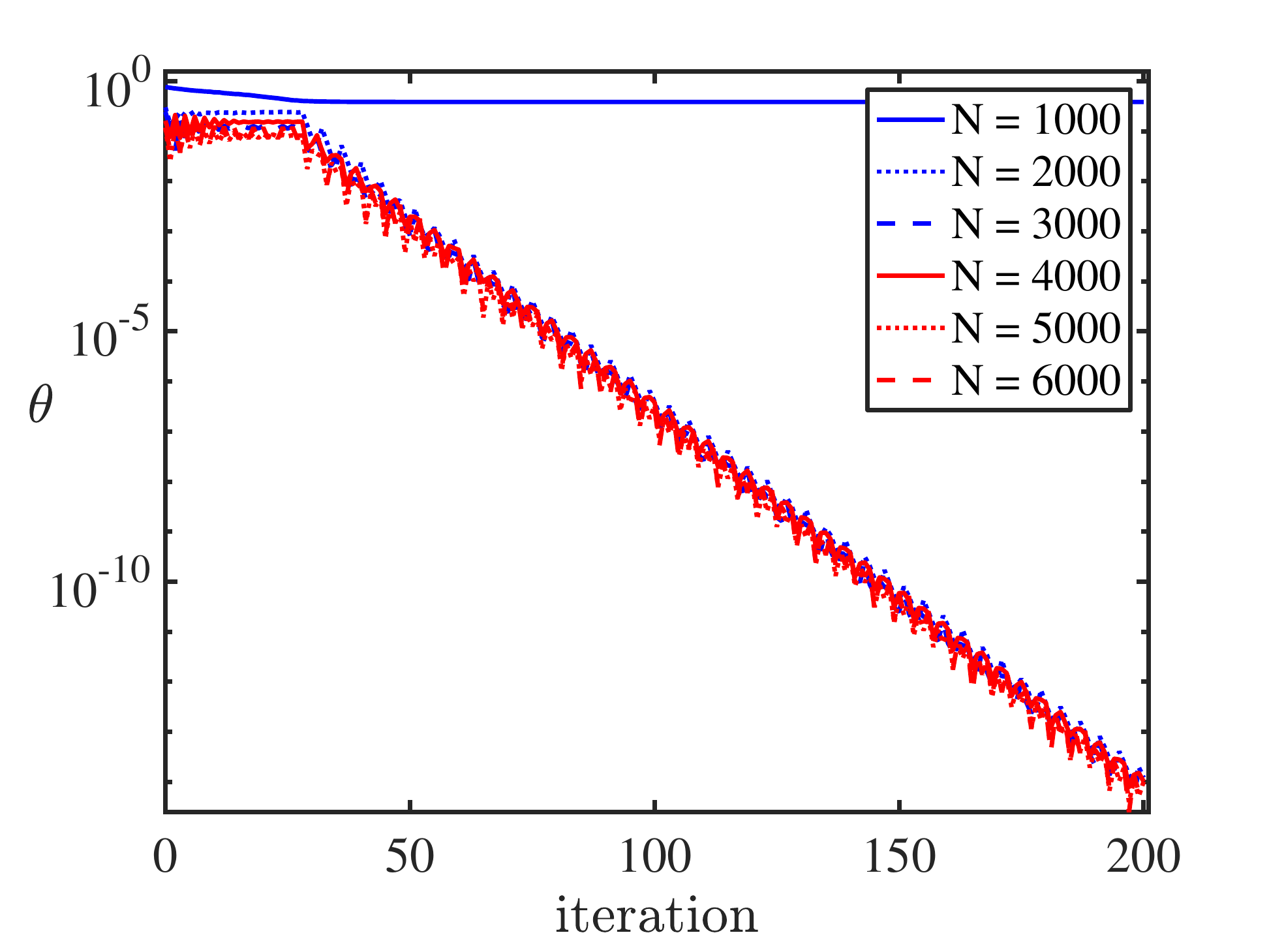}}
		\caption{$D = 100, d = 299, M = 3000$}
		\label{fig:PSGM-varyN-large}	\end{subfigure}
	\caption{Convergence of the PSGM with piecewise exponentionally diminishing  step size with $K_0 = 30, K = 4$, $\beta = \frac{1}{2}$, and $\mu_0$ determined by line search method. Here $\gamma=\frac{M}{M+N} $ denotes the outlier ratio.}\label{fig:PSGM-piecewise-exponentionally-diminishing}
\end{figure}

In \Cref{fig:PSGMvsIRLS-a} and \Cref{fig:PSGMvsIRLS-b} we compare PSGM algorithms with the ALP method (DPCP-ALP) and the IRLS algorithm (DPCP-IRLS) proposed in~\citep{Tsakiris:DPCP-Arxiv17}. First observe that, as expected, although ALP finds a normal vector in few iterations, it has the highest time complexity because it solves an LP during each iteration. \Cref{fig:PSGMvsIRLS-b} indicates that one iteration of ALP consumes more time than the whole procedure for PSGM. We also note that aside from the theoretical guarantee for PSGM-PGD, it also converges faster than IRLS  (in terms of computing time), which lacks a convergence guarantee. 

\begin{figure}
	\begin{subfigure}{0.48\linewidth}
		\centerline{
			\includegraphics[width=2.5in]{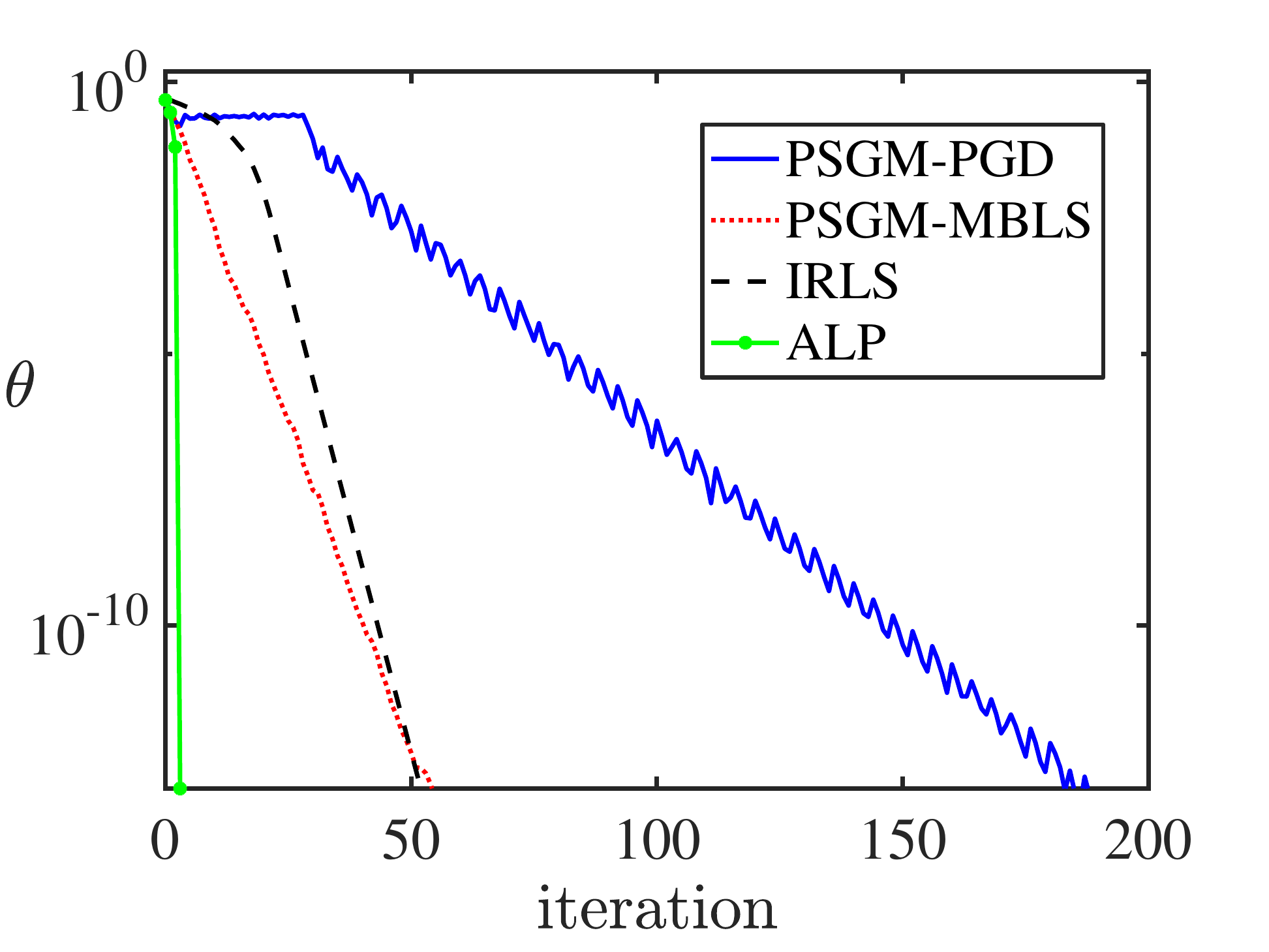}}
		\caption{}
		\label{fig:PSGMvsIRLS-a}
	\end{subfigure}
	\hfill
	\begin{subfigure}{0.48\linewidth}
		\centerline{
			\includegraphics[width=2.5in]{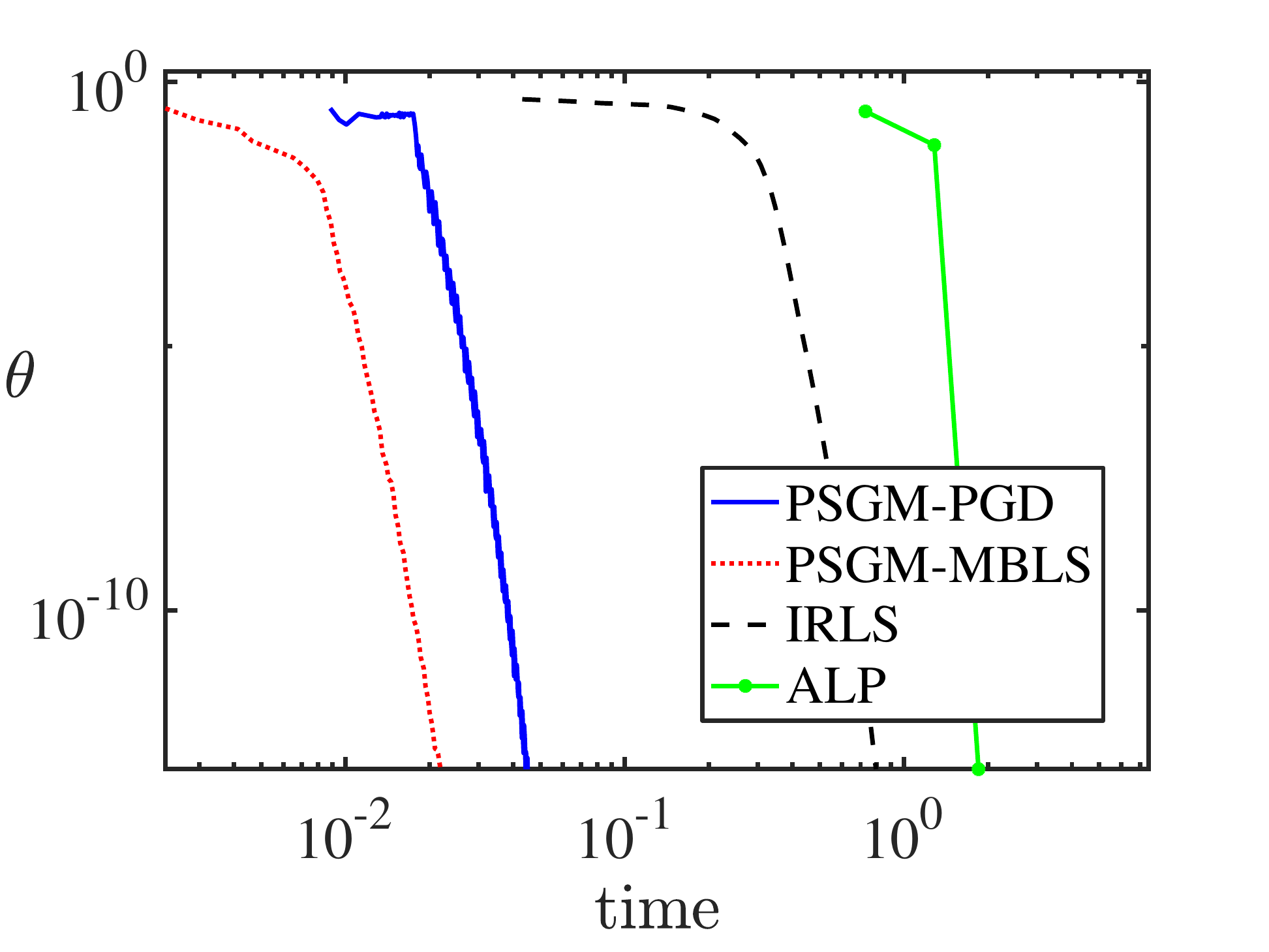}}
		\caption{}
		\label{fig:PSGMvsIRLS-b}
	\end{subfigure}
	\caption{\small  Comparison of PSGM with ALP and IRLS in~\citep{Tsakiris:DPCP-Arxiv17} in terms of (a) iterations and (b) computing time. Here $D = 30, d = 29, N = 500, \frac{M}{M + N}= 0.7$. }\label{fig:PSGMvsIRLS}
\end{figure}

\subsection{Phase Transition in Terms of $M$ and $N$}
Using the same setup for generating outliers and inliers, we fix the ambient dimension $D = 30$ and the subspace dimension $d$ and vary the number of outliers $M$ and the number of inliers $N$ to illustrate \Cref{thm:global-opt-random-model}. \Cref{fig:MversusN} displays the principal angle from $\calS^\perp$ of the solution to the DPCP problem computed by the PSGM-MBLS algorithm for $d = 15$ and $d = 29$. We observe that the phase transition is indeed quadratic, indicating that DPCP can tolerate as many as $O(N^2)$ outliers as predicted by \Cref{thm:global-opt-random-model}. The relationship between $M$ and $d$ can also be observed by comparing \Cref{fig:MversusN-d15} with \Cref{fig:MversusN-d29}. Particularly, we can see that when fix the ambient dimension $D$ and the number of inliers $N$, the subspace with smaller ambient dimension $d$ can tolerate more outliers,  coincidence with $M = O(\frac{N^2}{d})$ outliers in \Cref{thm:global-opt-random-model}.

\begin{figure}
\begin{subfigure}{0.48\linewidth}
	\centering
	\includegraphics[width=2.5in]{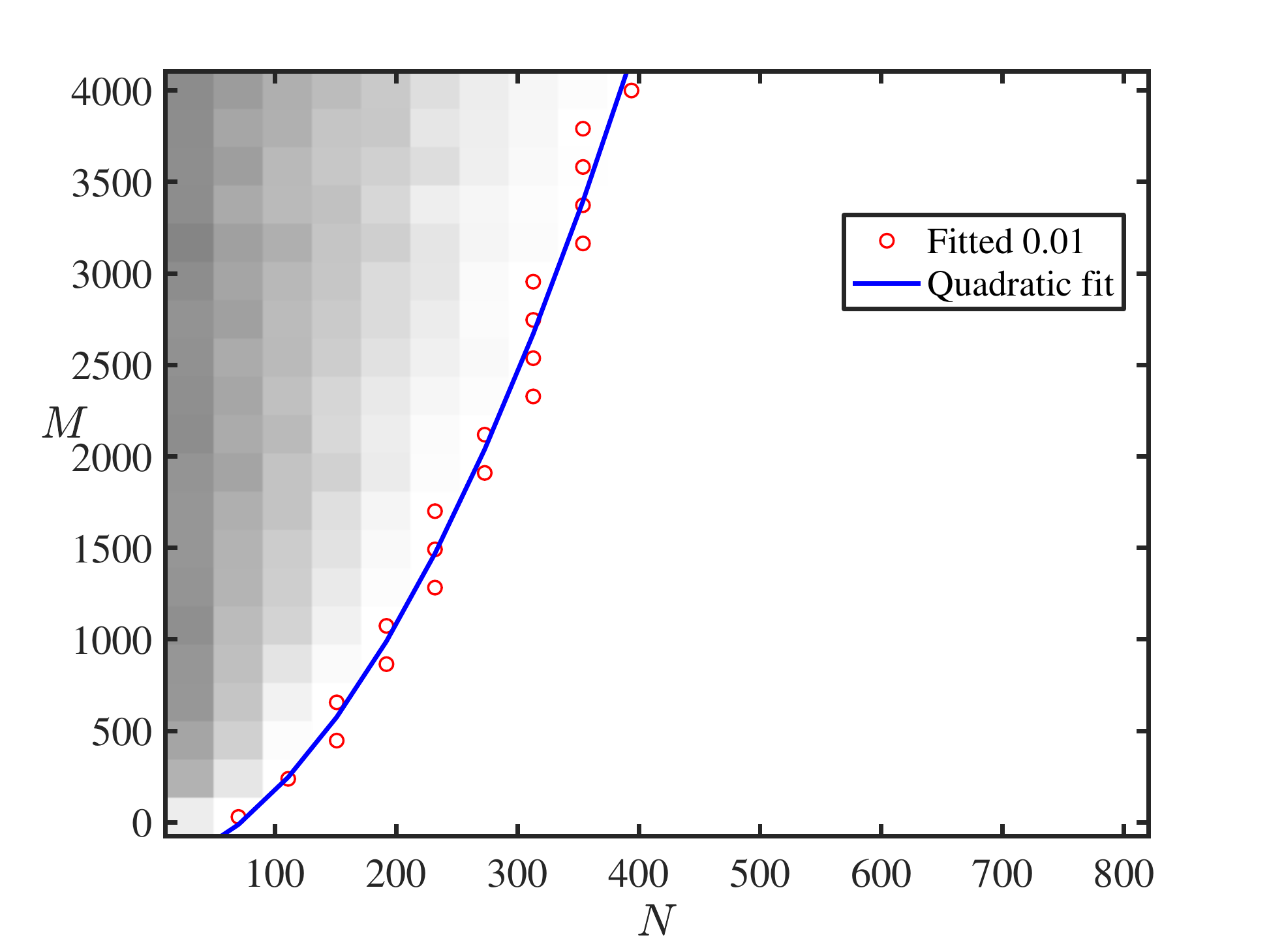}
		\caption{$d = 15$}
\label{fig:MversusN-d15}
\end{subfigure}
\hfill	
\begin{subfigure}{0.48\linewidth}
	\centering
	\includegraphics[width=2.5in]{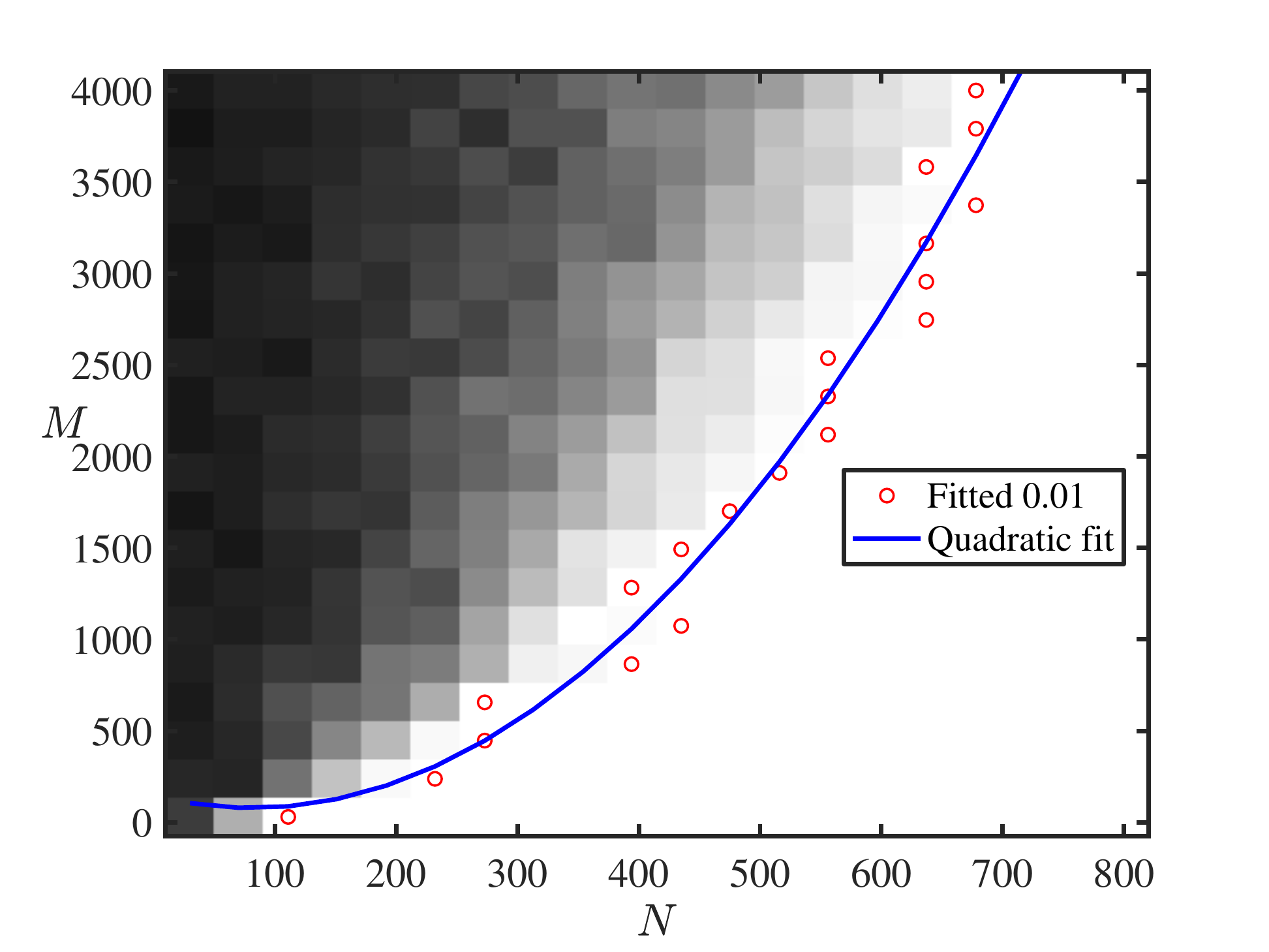}
	\caption{$d = 29$}
	\label{fig:MversusN-d29}
\end{subfigure}
	\caption{\small The principal angle $\theta$ between the solution to the DPCP problem~\eqref{eq:dpcp} and $\calS^\perp$ (black corresponds to $\frac{\pi}{2}$ and white corresponds to $0$) when $D = 30$ and (a) $d = 15$ and (b) $d = 29$.  For each $M$, we find the smallest $N$ (red dots) such that $\theta\leq 0.01$. The blue quadratic curve indicates the least-squares fit to these points. Results are averaged over 10 independent trails.}\label{fig:MversusN}
\end{figure}

\subsection{Experiments on Real $3$D Point Cloud Road Data}

We compare DPCP-PSGM (with a modified backtracking line search)
with RANSAC \cite{RANSAC}, $\ell_{2,1}$-RPCA \cite{Xu:NIPS10} and REAPER \cite{lerman2015robust} on the road detection challenge\footnote{Coherence Pursuit \cite{Rahmani:arXiv16} is not applicable to this experiment because forming the required correlation matrix of the thousands of $3$D points is prohibitively expensive.} of the KITTI dataset ~\cite{geiger2013vision}, recorded from a moving platform while driving in and around Karlsruhe, Germany. This dataset consists of image data together with corresponding $3$D points collected by a rotating 3D laser scanner. In this experiment we use only the $360^{\circ}$ $3$D point clouds with the objective of determining the $3$D points that lie on the road plane (inliers) and those off that plane (outliers). Typically, each $3$D point cloud is on the order of $100,\!000$ points including about $50\%$ outliers. Using homogeneous coordinates this can be cast as a robust hyperplane learning problem in $\R^4$. Since the dataset is not annotated for that purpose, we manually annotated a few frames (e.g., see the left column of Fig. \ref{fig:Frames}). Since DPCP-PSGM is the fastest method (on average converging in about $100$ milliseconds for each frame on a $6$ core $6$ thread Intel (R) i$5$-$8400$ machine), we set the time budget for all methods equal to the running time of DPCP-PSGM. For RANSAC we also compare with 10 and 100 times that time budget. Since $\ell_{2,1}$-RPCA does not directly return a subspace model, we extract the normal vector via SVD on the low-rank matrix returned by that method. Table \ref{table:ROC} reports the area under the Receiver Operator Curve (ROC), the latter obtained by thresholding the distances of the points to the hyperplane estimated by each method, using a suitable range of different thresholds\footnote{For RANSAC, we also use each such threshold as its internal thresholding parameter.}. As seen, even though a low-rank method, $\ell_{2,1}$-RPCA performs reasonably well but not on par with DPCP-PSGM and REAPER, which overall tend to be the most robust methods. On the contrary, for the same time budget, RANSAC, which is a popular choice in the computer vision community for such outlier detection tasks, is essentially failing due to an insufficient number of iterations. Even allowing for a $100$ times higher time budget still does not make RANSAC the best method, as it is outperformed by DPCP-PSGM on  five out of the seven point clouds (1, 45, and 137 in KITTY-CITY-5, and 0 and 21 in KITTY-CITY-48).    

\begin{table}[H]
	\LARGE
	\centering
	\caption{Area under ROC for annotated $3$D point clouds with index 1, 45, 120, 137, 153 in KITTY-CITY-5 and 0, 21 in KITTY-CITY-48. The number in parenthesis is the percentage of outliers.} \label{table:ROC}
	\vspace{2pt}
	\resizebox{\textwidth}{!}{%
		\begin{tabular}{@{}llcccccccllcll@{}}
			&  & \multicolumn{1}{l}{} & \multicolumn{3}{c}{{\ul KITTY-CITY-5}} & \multicolumn{1}{l}{} & \multicolumn{1}{l}{} & \multicolumn{6}{c}{{\ul KITTY-CITY-48}} \\[2pt]
			Methods &  & 1(37\%) & 45(38\%) & 120(53\%) & 137(48\%) & 153(67\%) &  & \multicolumn{3}{c}{0(56\%)} & \multicolumn{3}{c}{21(57\%)} \\ \cmidrule(r){1-1} \cmidrule(lr){3-7} \cmidrule(l){9-14} 
			DPCP-PSGM &  & \textbf{0.998} & \textbf{0.999} & 0.868 & \textbf{1.000} & 0.749 &  & \multicolumn{3}{c}{\textbf{0.994}} & \multicolumn{3}{c}{\textbf{0.991}} \\
			REAPER &  & 0.998 & 0.998 & 0.839 & 0.999 & 0.749 &  & \multicolumn{3}{c}{\textbf{0.994}} & \multicolumn{3}{c}{0.982} \\
			$\ell_{2,1}$-RPCA &  & 0.841 & 0.953 & 0.610 & 0.925 & 0.575 &  & \multicolumn{3}{c}{0.836} & \multicolumn{3}{c}{0.837} \\
			RANSAC &  & 0.596 & 0.592 & 0.569 & 0.551 & 0.521 &  & \multicolumn{3}{c}{0.534} & \multicolumn{3}{c}{0.531} \\
			10xRANSAC &  & 0.911 & 0.773 & 0.717 & 0.654 & 0.624 &  & \multicolumn{3}{c}{0.757} & \multicolumn{3}{c}{0.598} \\
			100xRANSAC &  & 0.991 & 0.983 & \textbf{0.965} & 0.955 & \textbf{0.849} &  & \multicolumn{3}{c}{0.974} & \multicolumn{3}{c}{0.902}
		\end{tabular}%
	}
\end{table}

\begin{figure}[t!]
	\centering
	\includegraphics[width=\linewidth]{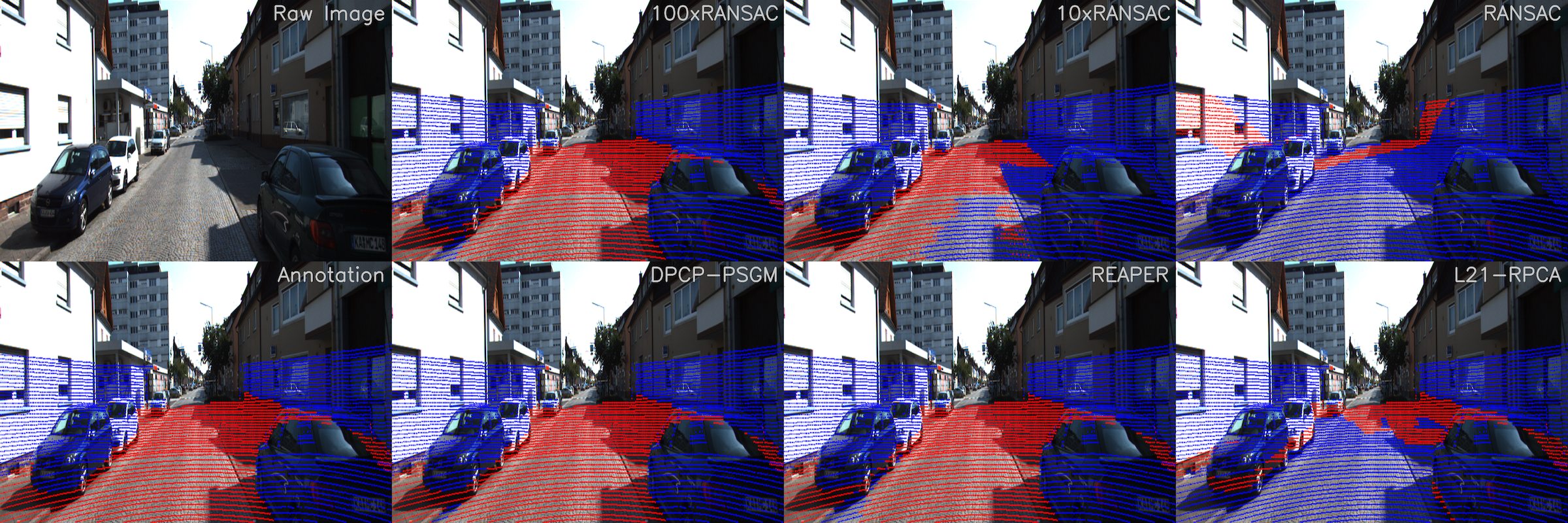}
	\caption{Frame $21$ of dataset KITTI-CITY-48: raw image, projection of annotated $3$D point cloud onto the image, and detected inliers/outliers using a ground-truth threshold on the distance to the hyperplane for each method. The corresponding F$1$ measures are DPCP-PSGM ({\bf 0.933}), REAPER (0.890), $\ell_{21}$-RPCA (0.248), RANSAC (0.023), 10xRANSAC (0.622), and 100xRANSAC (0.824).}
	\label{fig:Frames}
\end{figure}

\section{Conclusions}

We provided an improved analysis for the global optimality of the Dual Principal Component Pursuit (DPCP) method, which in particular suggests that DPCP can handle up to $O((\# \text{inliers})^2)$ outliers. We also presented a scalable first-order method that only uses matrix-vector multiplications, for which we established global convergence guarantees for various step size selection schemes, regardless of the non-convexity and non-smoothness of the DPCP optimization problem. Finally, experiments on $3$D point cloud road data demonstrate that DPCP-PSGM is able to outperform RANSAC when the latter is run with the same computational budget.

\appendix

\section{Proof of \Cref{thm:quantities-ranom-model} }\label{sec:prf-quantities-ranom-model}

After presenting some useful preliminary results, we prove  \Cref{thm:quantities-ranom-model}  by individually convering the four terms $c_{\bfcalO,\min}$, $c_{\bfcalO,\max}$, $\eta_{\bfcalO}$ and $\eta_{\bfcalO}$.
\subsection{Preliminaries}
Suppose $X_1,\ldots,X_n$ are $n$ independent and identically distributed (i.i.d.) random observations from a probability measure $P$ on a measurable space $(\calX,\calA)$. Given a measurable function $f:\calX\rightarrow \R$, the {\em empirical process} evaluated at $f$ is defined as
\begin{equation}
\setG_n f:=\sqrt{n} \left(\frac{1}{n}\sum_{i=1}^n f(X_i) - \int f \dif P\right),
\label{eq:empirical process}
\end{equation}
where $\int f \dif P$ is the expectation of $f$ under $P$ and $\frac{1}{n}\sum_{i=1}^n f(X_i)$ is called the {\em empirical distribution}. There are several results concerning the supreme of $\setG_n f$ over a given class $\calF$ of measurable functions.

Define an {\em envelope function} $F:\calX \rightarrow \R$ such that $|f|\leq F$ for every $f\in\calF$. The $L_r(P)$-norm is defined as $\|f\|_{L_{r}(P)} = (\int |f|^{r} \dif P)^{1/r}$. We need one more definition for the so-called {\em bracket number} which (informally speaking) measures the size of a class functions $\calF$. Given two functions $l$ and $u$, the {\em bracket} $[l,u]$ is the set of all functions $f$ with $l\leq f\leq u$. An $\epsilon$-bracket in $L_r(P)$ is a bracket $[l,u]$ with $\int (u-l)^{r}\dif P \leq \epsilon^r$ (since $l\leq u$, it is equivalent to say $\|u-l\|_{L_r(P)}\leq \epsilon$). The bracket number	$N_{[]}(\epsilon,\calF,L_2(P))$ is the minimum number of $\epsilon$-brackets needed to cover $\calF$.

\begin{lem}[~\citep{Van1998:asymptotic}, Cor. 19.35]
	For any class $\calF$ of measurable functions with envelope function $F$,
	\begin{equation}
	\mathbb{E}\left[ \sup_{f\in\calF}  \left|\setG_n f\right| \right]\lesssim J_{[]}(\|F\|_{P,2},\calF,L_2(P)),
	\label{eq:empirical-process-bracket-integral}
	\end{equation}
	
	where $J_{[]}(\|F\|_{P,2},\calF,L_2(P))$ is called the bracketing integral:
	\begin{equation}
	J_{[]}(\|F\|_{L_2(P)},\calF,L_2(P)) =\int_0^{\|F\|_{L_2(P)}}\sqrt{\log \left( N_{[]}(\epsilon,\calF,L_2(P)) \right)  }\dif \epsilon.
	\end{equation}
	\label{thm:exp-suprema-eprirical-process}
\end{lem}

\begin{lem}[McDiarmid's Inequality, \citep{Mcdiarmid:1989}]
	Let $Z_1, \dots, Z_n$ be real-valued independent random variables. Let $f: \mathbb{R}^n \rightarrow \mathbb{R}$ be a function that satisfies
	\begin{equation}
	\sup_{z_1,\cdots,z_n,z_i'}\Bigg|f(z_1,\cdots,z_{i-1},z_i,z_{i+1},\cdots,z_n) - f(z_1,\dots,z_{i-1},z_i',z_{i+1},\cdots,z_n)\Bigg| \leq c_i,
	\nonumber	\end{equation}
	for every $i = 1,\cdots,n$. Then
	\begin{equation}
	\mathbb{P}\Bigg[\Bigg|f(Z_1, \cdots, Z_n)- \mathbb{E}\Bigg[ f(Z_1, \cdots, Z_n)\Bigg]\Bigg|\geq \epsilon \Bigg]\leq 2\exp\Bigg(-\frac{2\epsilon^2}{\sum_{i=1}^n c_i^2}\Bigg).
	\nonumber	\end{equation}
	\label{lem:mcd}
\end{lem}

\begin{lem}[Rademacher Comparison, \citep{Ledoux:Springer2013}, Eqn. (4.20)]
	Let $F: \mathbb{R} \to \mathbb{R}$ be convex and increasing. Let $\varphi_i: \mathbb{R} \to \mathbb{R}$, $i \leq N$, be 1-Lipschitz functions such that $\varphi_i(0) = 0$. Let $\varepsilon_i$ be Rademacher random variables. Then, for any bounded subset $T$ in $\mathbb{R}^N$
	\begin{equation}
	\mathbb{E}\Bigg[F\Bigg(\sup_{(t_1, t_2,\cdots,t_N) \in T}\sum_{i=1}^N\varepsilon_i \varphi_i(t_i)\Bigg)\Bigg]\leq\mathbb{E}\Bigg[F\Bigg(\sup_{(t_1, t_2,\cdots,t_N) \in T}\sum_{i=1}^N\varepsilon_i t_i\Bigg)\Bigg].
	\end{equation}
	\label{lem:420}
\end{lem}

\begin{lem}[Rademacher Symmetrization, \citep{Kakade:2011}, Thm. 1.1]
	Let $F$ be a class of functions $f:\mathbb{R}\rightarrow \mathbb{R}$ such that $0\leq f(z) \leq 1$. Let $\varepsilon_i$ be Rademacher random variables. Then for independent and identically distributed random variables $Z_1,\dots,Z_n$, we have
	\begin{align}
	\mathbb{E}\Bigg[ \sup_{f \in F} \Bigg( \frac{1}{n}\sum_{i=1}^n f(Z_i) - \mathbb{E}[f(Z)] \Bigg) \Bigg] &\leq 2 \mathbb{E}\Bigg[\sup_{f\in F} \frac{1}{n}\sum_{i=1}^n\varepsilon_i f(Z_i) \Bigg],\\
	\mathbb{E}\Bigg[ \sup_{f \in F} \Bigg( \mathbb{E}[f(Z)] - \frac{1}{n}\sum_{i=1}^n f(Z_i) \Bigg) \Bigg] &\leq 2 \mathbb{E}\Bigg[\sup_{f\in F} \frac{1}{n}\sum_{i=1}^n\varepsilon_i f(Z_i) \Bigg].
	\end{align}
	\label{lem:rad}
\end{lem}


We also require a standard result about the covering number of the sphere. Denote by $\calN_\epsilon$ an $\epsilon$-net of $\setS^{D-1}$ if every point $\vb\in\setS^{D-1}$ can be approximated to within $\epsilon$ by some point $\vb'\in\calN_\epsilon$. The minimal cardinality of an $\epsilon$-net, denoted by $\calN(\setS^{D-1},\epsilon)$, is called the covering number of $\setS^{D-1}$.
\begin{lem}(Covering Number of the Sphere, \cite[Lemma 5.2]{Vershynin2010:introduction})
	For every $\epsilon>0$, the covering number of the sphere $\setS^{D-1}$ satisfies
	\begin{equation}
	\calN(\setS^{D-1},\epsilon)\leq \left(1 + \frac{2}{\epsilon}\right)^D.
	\end{equation}
	\label{lem:conver-number-sphere}
\end{lem}

We finally require one more result converning the probability that $\sign(\vo^\top\vb) = \sign(\vo^\top\vb')$ when $\vb$ is very close to $\vb'$.
\begin{lem}
	Denote by $\setB(\vb,\epsilon_1)$ the set of points that around $\vb$:
	\begin{equation}
	\setB(\vb,\epsilon_1): = \left\{\vb'\in\setS^{D-1}: \|\vb - \vb'\|_2 \leq \epsilon_1   \right\}.
	\nonumber	\end{equation}
	
	Let $\vo\in\setS^{D-1}$ be drawn independently and uniformaly at random from the unit sphere $\setS^{D-1}$. For any $\vb\in\setS^{D-1}$ and $\epsilon>0$, define
	\begin{align}
	\overline\setA:= \left\{ \vo\in\setS^{D-1}: \sign(\vo^\top\vb) = \sign(\vo^\top\vb'), \ \forall \   \vb' \in \setB(\vb,\epsilon_1) \right\}.
	\label{eq:bar-setA}\end{align}	
	Then
	\begin{equation}
	\P{\vo\in\overline\setA^c} \lesssim \left\{ \begin{matrix} \epsilon_1, & D = 2, \\ c_D D\epsilon_1^2, & D \geq 3.  \end{matrix}  \right.
	\nonumber	\end{equation} where $\lesssim$ means smaller than up to a universal constant which is independent of $D$.
	\label{lem:Ac}
\end{lem}

\begin{proof}
	Without loss of generality, suppose $\vb = \ve_1$ which is a length-$D$ vector with 1 in the first entry and 0 elsewhere. When $\eps_1 > 1$, note that $ \mathbb{P}[\vo\in\overline\setA^c]\leq 1 \lesssim c_D D\epsilon_1^2$. The rest is to consider the more interesting case that $\epsilon_1 \leq 1$. Toward that end, first note that for any $\vb' \in \setB(\vb,\epsilon_1) $, we have $\ve_1^\top\vb'\leq 1-\epsilon_1$, which further implies that $\sign(\vo^T \vb') = \sign(\vo^\top\vb)$ when $|o_1|\geq \eps_1$. Let $o_1$ be the first element in $\vo$, we then have 
	\begin{align*}
	\mathbb{P}\left[\overline\setA^c\right] &= \mathbb{P}\left[\left\{\vo\in\setS^{D-1}: \ \exists \   \vb' \in \setB(\vb,\epsilon_1), \sign(\vo^\top\vb) \neq \sign(\vo^\top\vb')  \right\}\right]\\
	&\leq  \mathbb{P}\left[\left\{\vo\in\setS^{D-1}: o_1 \in [-\epsilon_1, \epsilon_1]\right\}\right]
	\end{align*}
	To calculate the probability, we use the spherical coordinates. Denote by 
\[
\vo =\begin{bmatrix}\cos(\theta_1) \\ \sin(\theta_1)\cos(\theta_2) \\  \vdots  \\ \sin(\theta_1)\cdots\sin(\theta_{D-2})\cos(\theta_{D-1}) \\  \sin(\theta_1)\cdots\sin(\theta_{D-2})\sin(\theta_{D-1})\end{bmatrix},
\]
where $\theta_1, \cdots \theta_{D-2} \in [0, \pi)$ and $\theta_{D-1} \in [0, 2\pi)$. Also let $I_n = \int_0^\pi \sin^n(\theta) \d \theta$. When $D > 2$, let $\mathcal{A}$ be the area of the unit sphere. We have
	\begin{align*}
	\P{\left\{\vo\in\setS^{D-1}: o_1 \in [-\epsilon_1, \epsilon_1]\right\}} &= \frac{\int_{\mathbb{S}^{D-1} \cap o_1 \in [-\eps_1, \eps_1]} |\cos(\theta_1)| \d \mathcal{A}}{\int_{\mathbb{S}^{D-1}} \d \mathcal{A}} \\
	&= \frac{2\int_{\arccos(\eps_1)}^{\pi/2} \cos(\theta_1) \sin^{D-2}(\theta_1) \d \theta_1 \prod_{j=1}^{D-3} I_j}{\prod_{j=1}^{D-2} I_j} \\
	&= \frac{2\int_{\arccos(\eps_1)}^{\pi/2} \cos(\theta) \sin^{D-2}(\theta) \d \theta}{I_{D-2}} \\
	&= \frac{2 \sin^{D-1}(\theta)\bigg|_{\arccos(\eps_1)}^{\pi/2}}{(D-1)I_{D-2}} \\
& = 2  \frac{1 - (1-\eps_1^2)^{(D-1)/2}}{(D-1)I_{D-2}}.
	\end{align*}
	Note that
	\begin{align*}
	\frac{2}{(D-1)I_{D-2}} &= \frac{1}{D-1} \cdot \frac{1}{\int_0^\pi \sin^{D-2}(\theta) \d\theta} \\
	&= \frac{1}{D-1} \cdot \frac{D-2}{D-3} \cdot \frac{1}{\int_0^\pi \sin^{D-4}(\theta) \d\theta} \\
	&= \frac{1}{D-1} \cdot \frac{D-2}{D-3} \cdot \frac{D-4}{D-5}\cdots \\
	&= c_D.
	\end{align*}
	Then, applying Taylor's approximation to the term $(1-\eps_1^2)^{(D-1)/2}$ at $\eps_1 = 0$ gives
	\begin{align*}
	(1-\eps_1^2)^{(D-1)/2} &= 1 + 0 \cdot \eps_1 + \frac{(D-1)(1-x^2)^\frac{D-1}{2}((D-2)x^2-1)}{2(x^2-1)^2} \bigg|_0 \eps_1^2 + O(\eps_1^3) \\
	&= 1 - \frac{D-1}{2} \eps_1^2 + O(\eps_1^3).
	\end{align*}
Therefore, we have
	\begin{align*}
	\P{\left\{\vo\in\setS^{D-1}: o_1 \in [-\epsilon_1, \epsilon_1]\right\}} &= \frac{2}{(D-1)I_{D-2}} \cdot (1 - (1-\eps_1^2)^{(D-1)/2}) \\
	&\lesssim c_D D \eps_1^2.
	\end{align*}
	
	The proof of $D =2$ follows from a similar argument.
\end{proof}

\subsection{Bounding $c_{\bfcalO,\min}$}
\label{sec:bounding-cOmin}
We first repeat the result in \Cref{thm:quantities-ranom-model} concerning $c_{\bfcalO,\min}$.
\begin{lem} \label{thm:comin}
	Let $\vo_1,\dots,\vo_M$ be uniformly distributed on $\mathbb{S}^{D-1}$. Then for any $t > 0$
	\begin{equation*}
	\mathbb{P}\Bigg[ c_{\mathbfcal{O},\min} \leq c_D - \left(2+\frac{t}{2} \right)\frac{1}{\sqrt{M}} \Bigg] \leq 2 \exp(-t^2/2).
	\end{equation*}	
\end{lem}

\begin{proof} We first present a useful result for proving \Cref{thm:comin}.
\begin{lem} \label{lem:Eco}
	Let $\vo_1,\dots,\vo_M$ be uniformly distributed on $\mathbb{S}^{D-1}$. Then
	\begin{align*}
	\mathbb{E}\Bigg[ \sup_{||\vb||_2 =1}\Bigg(\sum_{j=1}^M|\vb^\top\vo_j|-c_D\Bigg)\Bigg] &\leq \frac{2}{\sqrt{M}},\\
	\mathbb{E}\Bigg[ \sup_{||\vb||_2 =1}\Bigg(c_D-\sum_{j=1}^M|\vb^\top\vo_j|\Bigg)\Bigg] &\leq \frac{2}{\sqrt{M}}.
	\end{align*}	
\end{lem}
\begin{proof}
	Applying Lemma \ref{lem:rad} using the class of functions $F = \{f_{\vb}:\mathbb{S}^{D-1}\to \mathbb{R}; f_{\vb}(\vo) = |\vb^\top\vo|; \vb \in \mathbb{S}^{D-1}\}$, we have
	\begin{align*}
	\mathbb{E}\Bigg[ \sup_{||\vb||_2 =1}\Bigg(\sum_{j=1}^M|\vb^\top\vo_j|-c_D\Bigg)\Bigg] &\leq \mathbb{E}\Bigg[ \sup_{||\vb||_2 =1} \sum_{j=1}^M \varepsilon_j |\vb^\top\vo_j| \Bigg],\\
	\mathbb{E}\Bigg[ \sup_{||\vb||_2 =1}\Bigg(c_D-\sum_{j=1}^M|\vb^\top\vo_j|\Bigg)\Bigg] &\leq \mathbb{E}\Bigg[ \sup_{||\vb||_2 =1} \sum_{j=1}^M \varepsilon_j |\vb^\top\vo_j| \Bigg].
	\end{align*} Next, we apply Lemma \ref{lem:420} with $\varphi_j(\cdot)=|\cdot|$, $t_j = \vb^\top\vo_j$ to get the further bound
	\begin{align*}
	\mathbb{E}\Bigg[ \sup_{||\vb||_2 =1} \sum_{j=1}^M \varepsilon_j |\vb^\top\vo_j| \Bigg] &\leq 2 \mathbb{E}\Bigg[ \sup_{||\vb||_2 = 1} \sum_{j=1}^M\varepsilon_j \vb^\top\vo_j \Bigg] = \frac{2}{M} \mathbb{E}\Bigg[ \sup_{||\vb||_2 = 1} \left\langle \vb, \sum_{j=1}^M\varepsilon_j \vo_j \right \rangle \Bigg]\\
	&= \frac{2}{M} \mathbb{E}\left[ \left|\left| \sum_{j=1}^M\varepsilon_j \vo_j \right| \right|_2 \right]\leq \frac{2}{M} \sqrt{\mathbb{E}\left[ \left|\left| \sum_{j=1}^M\varepsilon_j \vo_j \right| \right|_2^2 \right]}\\
	&= \frac{2}{M} \sqrt{\mathbb{E}\left[ M + \sum_{i\neq j} \varepsilon_i\varepsilon_j \vo_i^T \vo_j \right]}= \frac{2}{M}\sqrt{M} = \frac{2}{\sqrt{M}},
	\end{align*} where the second inequality follows by $\mathbb{E}[Z]^2 \leq \mathbb{E}[Z^2]$.
\end{proof}

We are now ready to prove \Cref{thm:comin}.
	First note that
	\begin{equation*}
	Mc_{\mathbfcal{O},\min} = \inf_{||\vb||_2=1}\sum_{j=1}^M(|\vb^\top\vo_j| -c_D) + Mc_D = Mc_D - \sup_{||\vb||_2=1}\sum_{j=1}^M(c_D - |\vb^\top\vo_j|).
	\end{equation*} Applying Lemma \ref{lem:Eco}, we have
	\begin{equation*}
	\mathbb{E}\Bigg[ \sup_{||\vb||_2=1}\sum_{j=1}^M(c_D - |\vb^\top\vo_j|) \Bigg] \leq 2 \sqrt{M}.
	\end{equation*}
	Applying Lemma \ref{lem:mcd} the same way as in proof of Theorem \ref{thm:comax}, we have
	\begin{equation*}
	\mathbb{P}\Bigg[ \sup_{||\vb||_2=1}\sum_{j=1}^M(c_D - |\vb^\top\vo_j|) \geq 2\sqrt{M} + \epsilon \Bigg] \leq 2 \exp(-2\epsilon^2/M).
	\end{equation*}
	Therefore,
	\begin{equation*}
	\mathbb{P} \Bigg[Mc_D - \sup_{||\vb||_2=1}\sum_{j=1}^M(c_D - |\vb^\top\vo_j|) \leq Mc_D - 2\sqrt{M} - \epsilon \Bigg] \leq 2 \exp(-2\epsilon^2/M),
	\end{equation*}
	and setting $\epsilon = t\sqrt{M}/2$ we get
	\begin{equation*}
	\mathbb{P}\Bigg[ M c_{\mathbfcal{O},\min} \leq Mc_D - \left( 2+\frac{t}{2} \right)\sqrt{M} \Bigg] \leq 2 \exp(-t^2/2).
	\end{equation*}	
\end{proof}

\subsection{Bounding $c_{\bfcalO,\max}$}
We first repeat the result in \Cref{thm:quantities-ranom-model} concerning $c_{\bfcalO,\max}$.

\begin{lem} \label{thm:comax}
	Let $\vo_1,\dots,\vo_M$ be uniformly distributed on $\mathbb{S}^{D-1}$. Then for any $t > 0$
	\begin{equation*}
	\mathbb{P}\Bigg[ c_{\mathbfcal{O},\max} \geq c_D + \left(2+\frac{t}{2} \right)\frac{1}{\sqrt{M}} \Bigg] \leq 2 \exp(-t^2/2).
	\end{equation*}	
\end{lem}

\begin{proof}
	First note that
	\begin{equation*}
	Mc_{\mathbfcal{O},\max} = \sup_{||\vb||_2=1}\sum_{j=1}^M(|\vb^\top\vo_j| -c_D) + Mc_D.
	\end{equation*} Since $\mathbb{S}^{D-1}$ is compact, there exists $\vb^* \in \mathbb{S}^{D-1}$ which achieves the supremum in the expression above. Then for any $\vv_1, \vv_2, \cdots, \vv_M, \vv_k' \in \mathbb{S}^{D-1}$, we have
	\begin{align*}
	&\Bigg| \sup_{||\vb||_2=1}\sum_{j=1}^M(|\vb^\top\vv_j| -c_D) - \sup_{||\vb||_2=1} \Bigg( \sum_{j\neq k}(|\vb^\top\vv_j| -c_D) + |\vb^\top\vv_k'| -c_D \Bigg) \Bigg|\\
	&\leq \Bigg| \sum_{j=1}^M(|\vect{b^*}^\top\vv_j| -c_D) - \Bigg( \sum_{j\neq k}(|\vect{b^*}^\top\vv_j| -c_D) + |\vect{b^*}^\top\vv_k'| -c_D \Bigg) \Bigg|\\
	&= \Bigg| |\vect{b^*}^\top\vv_k| - |\vect{b^*}^\top\vv_k'| \Bigg| \leq 1.
	\end{align*} Applying Lemma \ref{lem:mcd} with $c_k = 1$ and using Lemma \ref{lem:Eco}, we obtain
	\begin{equation*}
	\mathbb{P}\Bigg[ \sup_{||\vb||_2=1}\sum_{j=1}^M(|\vb^\top\vo_j| -c_D) \geq 2\sqrt{M} + \epsilon \Bigg] \leq 2 \exp(-2\epsilon^2/M).
	\end{equation*}
	Finally, set $\epsilon = t\sqrt{M}/2$ to get
	\begin{equation*}
	\mathbb{P}\Bigg[M c_{\mathbfcal{O},\max} \geq Mc_D + \left(2+\frac{t}{2} \right)\sqrt{M} \Bigg] \leq 2 \exp(-t^2/2).
	\end{equation*}
\end{proof}

\subsection{Bounding $c_{\bfcalX,\min}$}
The proof of the result concerning $c_{\bfcalX,\min}$  in \Cref{thm:quantities-ranom-model} follows similar argument as the one for $c_{\bfcalO,\min}$ in \Cref{sec:bounding-cOmin}.

\subsection{Bounding $\eta_{\bfcalO}$}

Recall that $\eta_{\bfcalO}$ defined in \eqref{eq:eta O} is equivalent to
\e
\eta_{\bfcalO}=\frac{1}{M}\max_{\vg,\vb\in\setS^{D-1},\vg\perp \vb}\    \left|\vg^\top \bfcalO \sign(\bfcalO^\top\vb)\right|.
\nonumber\ee
We repeat the result in \Cref{thm:quantities-ranom-model} concerning the above $\eta_{\bfcalO}$.

\begin{lem} \label{thm:bg}
	Let $\vo_1,\dots,\vo_M$ be uniformly distributed on $\mathbb{S}^{D-1}$. Then for any $t > 0$
	\begin{equation}
	\mathbb{P}\Bigg[ \sup_{\vb,\vg\in \mathbb{S}^{D-1}, \vb\bot \vg}\Bigg|\sum_{j=1}^M\sign(\vb^\top\vo_j) \vg^\top\vo_j\Bigg| \gtrsim (1+t)\sqrt{D}\log \left(\sqrt{c_D D}\right) \sqrt{M} \Bigg] \leq 2 \exp(-t^2/2).
	\end{equation}
\end{lem}

\begin{proof}
Before givin out the main proofs, we first preset the following useful result concerning the expectation of $\eta_{\bfcalO}$.

\begin{lem}
	Suppose $\vo_1,\cdots,\vo_M,$ are drawn independently and uniformly at random from the unit sphere $\setS^{D-1}$. Then
	\begin{equation}
	\mathbb{E}\Bigg[ \sup_{\vb,\, \vg \in \setS^{D-1}, \, \vb \perp \vg} \left|\sum_{j=1}^M\sign(\vb^\top\vo_j) \vg^\top\vo_j\right| \Bigg]  \lesssim\sqrt{D}\log \left(\sqrt{c_D D}\right) \sqrt{M},
	\end{equation} where $\lesssim$ means smaller than up to a universal constant which is independent of $D$ and $M$.
	\label{thm:exp-etaO}
\end{lem}

\begin{proof}
	The main idea for proving \Cref{thm:exp-etaO} is to view 
	\[
	\frac{1}{\sqrt{M}} \sup_{\vb,\, \vg \in \setS^{D-1}, \, \vb \perp \vg} \left|\sum_{j=1}^M\sign(\vb^\top\vo_j) \vg^\top\vo_j\right|
	\] as an empirical process and then utilize \Cref{thm:exp-suprema-eprirical-process}. Towards that end, define the set 
	\begin{equation}
	\setF:=\left\{(\vb,\vg):\vb,\vg\in\setS^{D-1}, \vb \perp \vg \right\}.
	\nonumber    \end{equation} We further define the parameterized function as
	\begin{equation}
	f_{\vb,\vg}(\vo):= \sign(\vb^\top\vo) \vg^\top\vo.
	\nonumber    \end{equation} The class of functions we are interested in is $\calF: = \left\{f_{\vb,\vg}: (\vb,\vg)\in\setF  \right\}$.
	
	Note that for any $f_{\vb,\vg}\in\calF$ (i.e., $(\vb,\vg)\in\setF$), we have
	\begin{equation}
	\mathbb{E}\left[ f_{\vb,\vg}(\vo) \right] = \mathbb{E}\left[  \sign(\vb^\top\vo) \vg^\top\vo\right] = 0,
	\nonumber    \end{equation} which together with \eqref{eq:empirical process} indicates that
	\begin{equation}
	\sum_{j=1}^M\sign(\vb^\top\vo_j) \vg^\top\vo_j = \sqrt{M}\setG_{M}f_{\vb,\vg},
	\nonumber    \end{equation} where $\setG_{M}f_{\vb,\vg}$ is the empirical process of $f_{\vb,\vg}$.
	
	To utilize \Cref{thm:exp-suprema-eprirical-process}, the rest of the proof is to show the corresponding bracketing integral is finite for our problem.  Since $|f_{\vb,\vg}(\vo)|\leq \|\vo\|_2$ for any $(\vb,\vg)\in\setF$, we know $F(\vo) = \|\vo\|_2$ is the envelope function of $\calF$ and $\|F\|_{P,2} = 1$. Thus, we only need to consider the the bracket integral $J_{[]}(1,\calF,L_2(P))$, where $P$ is now a probability measure on the unit sphere. To that end, we first  compute the bracket number $N_{[]}(\epsilon,\calF,L_2(P))$.
	
	Since our function $f_{\vb,\vg}$ is parameterized by $(\vb,\vg)$, covering the class of functions $\calF$ is related to covering the set $\setF$. For any fixed $(\vb,\vg)\in\setF$, define the set of points that around $(\vb,\vg)$:
	\begin{equation}
	\setB((\vb,\vg),\epsilon_1): = \left\{(\vb',\vg')\in\setF: \sqrt{\|\vb - \vb'\|_2^2 + \|\vg - \vg'\|_2^2}\leq \epsilon_1   \right\}. 
	\nonumber   \end{equation} Then, denote by 
	\begin{equation}
	\setA:= \left\{ \vo\in\R^{D}: \sign(\vo^\top\vb) = \sign(\vo^\top\vb'), \ \forall \   (\vb',\vg') \in \setB((\vb,\vg),\epsilon_1) \right\}.
	\nonumber    \end{equation} When $\vb$ is close to $\vb'$, then $\setA$ should cover most of $\vo$. If $\vo\in\setA$, then for any $(\vb',\vg')\in \setB((\vb,\vg),\epsilon_1)$ we have
	\begin{equation*}
	|f_{\vb,\vg}(\vo) - f_{\vb',\vg'}(\vo)| = |\vo^\top(\vg - \vg')|\leq \| \vg - \vg' \|_2\|\vo\|_2 \leq \epsilon_1.        
	\end{equation*} On the other hand,  if $\vo\in\setA^c$, then for any $(\vb',\vg')\in \setB((\vb,\vg),\epsilon_1)$ we have
	\begin{equation*}
	|f_{\vb,\vg}(\vo) - f_{\vb',\vg'}(\vo)| = |\vo^\top(\vg + \vg')|\leq \| \vg + \vg' \|_2\|\vo\|_2 \leq 2.        
	\end{equation*} To summary, we have
	\begin{equation}
	|f_{\vb,\vg}(\vo) - f_{\vb',\vg'}(\vo)|  \leq \epsilon_1 \delta_{\setA}(\vo) + 2\delta_{\setA^c}(\vo), \ \forall \ (\vb',\vg')\in \setB((\vb,\vg),\epsilon_1).
	\label{eq:f close to f'}
	\end{equation}
	
	We now define a bracket $[l,u]$ by 
	\begin{align*}
	&l(\vo) = f_{\vb,\vg}(\vo) - \epsilon_1 \delta_{\setA} (\vo) - 2\delta_{\setA^c}(\vo),\\
	&u(\vo) = f_{\vb,\vg}(\vo) + \epsilon _1\delta_{\setA} (\vo) + 2\delta_{\setA^c}(\vo),
	\end{align*} where the indicator function $\delta_{\setA} (\vo)$ is defined as $\delta_{\setA} (\vo) = \left\{\begin{matrix} 1, & \vo\in\setA \\0, & \vo\in \setA^c \end{matrix}\right.$. Due to \eqref{eq:f close to f'}, we have $f_{\vb',\vg'}\in[l,u]$ for all $(\vb',\vg')\in \setB((\vb,\vg),\epsilon_1)$. Also,
	\begin{equation}    
	\begin{split}
	\|u-l\|_{L_2(P)} &= \|2 \epsilon_1 \delta_{\setA} (\vo) + 4\delta_{\setA^c}(\vo) \|_{L_2(P)} = \sqrt{4\epsilon_1^2 \mathbb{P}[\vo\in\setA] + 16 \mathbb{P}[\vo\in\setA^c]}\\ &< 2\epsilon_1 + 4 \sqrt{\mathbb{P}[\vo\in\setA^c]}\leq  2\epsilon_1 + 4 \sqrt{\mathbb{P}[\vo\in\overline\setA^c]}\leq 2\epsilon_1 + 4\sqrt{c_1 c_D D}\epsilon_1,
	\end{split}
	\label{eq:L2 u minues l}
	\end{equation} where $\overline\setA$ is defined in \eqref{eq:bar-setA}, the second inequality utilizes the fact $\overline\setA \subset \setA$, $c_1$ is a universial constant, and the last inequality follows because $\mathbb{P}[\vo\in\setA^c] \leq c_1 c_D  D \epsilon_1^2$ according to \Cref{lem:Ac} (we only consider $D>2$ here, but the proof for $D=2$ follows similarly). 
	
	Finally, the number of brackets to cover $\calF$ is equal to the number of such balls $\setB((\vb,\vg),\epsilon_1)$ that cover $\setF$. Utilizing \Cref{lem:conver-number-sphere}, the covering number for $\setF$ is 
	\begin{equation}
	\calN(\setF,\epsilon_1) \leq \left(1 + \frac{2\sqrt{2}}{\epsilon_1}\right)^{2D}.
	\label{eq:cover-number-F}
	\end{equation}
	
	Recall the definition that the bracket number	$N_{[]}(\epsilon,\calF,L_2(P))$ is the minimum number of $\epsilon$-brackets needed to cover $\calF$, where an $\epsilon$-bracket in $L_2(P)$ is a bracket $[l,u]$ with $\|u-l\|_{L_2(P)}\leq \epsilon$. Thus, by letting $2\epsilon_1 + 4\sqrt{c_1 c_D D}\epsilon_1 = \epsilon$ and plugging  this into \eqref{eq:cover-number-F}, we obtain the bracket number	
	\begin{equation*}
	N_{[]}(\epsilon,\calF,L_2(P)) \leq \left(1 + c_2\frac{\sqrt{c_D D}}{\epsilon}\right)^{2D},    
	\end{equation*} where $c_2$ is a universal constant. Now plug this into \Cref{thm:exp-suprema-eprirical-process} gives
	\begin{align*}
	\frac{1}{\sqrt{M}}\mathbb{E}\Bigg[ \sup_{\vb,\, \vg \in \setS^{D-1}, \, \vb \perp \vg} \left|\sum_{j=1}^M\sign(\vb^\top\vo_j) \vg^\top\vo_j\right| \Bigg] & \lesssim\int_0^{1}\sqrt{\log \left(1 + c_2\frac{\sqrt{c_D D}}{\epsilon}\right)^{2D}  }\dif \epsilon \\ &\lesssim\sqrt{D}\log \left(\sqrt{c_D D}\right) .
	\end{align*}
	
\end{proof}

We are now ready to prove \Cref{thm:bg}.
For $\vv_1, \vv_2, \cdots, \vv_M, \vv_k' $ any points of $\mathbb{S}^{D-1}$,	since the product of compact spaces is compact, there exist $\vb^*,\vg^* \in \mathbb{S}^{D-1}$ for which the value \[
\sup_{\vb,\vg\in \mathbb{S}^{D-1}, \vb\bot \vg}|\sum_{j=1}^M\sign(\vb^\top\vv_j) \vg^\top\vv_j|\]
 is achieved. Then, we have
	\begin{align}
	&\left| \sup_{\vb,\vg\in \mathbb{S}^{D-1}, \vb\bot \vg}\left| \sum_{j=1}^M\sign(\vb^\top\vv_j) \vg^\top\vv_j \right| - \sup_{\vb,\vg\in \mathbb{S}^{D-1}, \vb\bot \vg}\left| \sum_{j\neq k}\sign(\vb^\top\vv_j) \vg^\top\vv_j + \sign(\vb^\top\vv_k')\vg^\top\vv_k' \right| \right|\\
	&\leq \left|\left|\sum_{j=1}^M\sign(\vect{b^*}^\top\vv_j) \vect{g^*}^\top\vv_j \right| - \left| \sum_{j\neq k}\sign(\vect{b^*}^\top\vv_j) \vect{g^*}^\top\vv_j + \sign(\vect{b^*}^\top\vv_k')\vect{g^*}^\top\vv_k' \right| \right|\\
	&\leq \Bigg| \sign(\vb^{*\T}\vv_k)\vect{g^*}^\top\vv_k - \sign(\vect{b^*}^\top\vv_k')\vect{g^*}^\top\vv_k' \Bigg| \le 2,\end{align} where the second inequality follows from the reverse triangle inequality. Applying Lemma \ref{lem:mcd} with $c_k = 2$ and using \Cref{thm:exp-etaO}, we obtain
	\begin{equation}
	\mathbb{P}\Bigg[ \sup_{\vb,\vg\in \mathbb{S}^{D-1}, \vb\bot \vg}\left|\sum_{j=1}^M\sign(\vb^\top\vo_j) \vg^\top\vo_j \right| \gtrsim \sqrt{D}\log \left(\sqrt{c_D D}\right) \sqrt{M} + \epsilon \Bigg] \leq 2 \exp\Bigg(-\frac{2\epsilon^2}{4M}\Bigg).
	\end{equation} Finally, set $\epsilon = t\sqrt{M}$ to get
	\begin{equation}
	\mathbb{P}\Bigg[ \sup_{\vb,\vg\in \mathbb{S}^{D-1}, \vb\bot \vg}\left| \sum_{j=1}^M\sign(\vb^\top\vo_j) \vg^\top\vo_j \right| \gtrsim \left(\sqrt{D}\log \left(\sqrt{c_D D}\right) + t\right) \sqrt{M} \Bigg] \leq 2 \exp(-t^2/2).
	\end{equation}
\end{proof}

\section{Proof of \Cref{thm:convergence PSGM}}
\label{sec:prf-convergence PSGM}
To begin, we first assume that
\e
\tan(\theta_k) <\frac{N c_{\bfcalX,\min} }{N\eta_{\bfcalX} + M\eta_{\bfcalO}}
\label{eq:tan theta small}\ee
holds for all $k\geq 0$. As a consequence of \eqref{eq:tan theta small}, we have

\e\begin{split}
	\cos(\theta_k) &= \frac{1}{\sqrt{1 + \tan^2(\theta_{k})}} > \frac{1}{\sqrt{1 +\left(\frac{N c_{\bfcalX,\min} }{N\eta_{\bfcalX} + M\eta_{\bfcalO}}\right)^2 }} \geq \frac{1}{\sqrt{2}} \frac{N\eta_{\bfcalX} + M\eta_{\bfcalO}}{N c_{\bfcalX,\min} } \\&= \frac{\mu'(N\eta_{\bfcalX} + M\eta_{\bfcalO})}{\sqrt{2}\mu' Nc_{\bfcalX,\min} } \geq \frac{\mu'(N\eta_{\bfcalX} + M\eta_{\bfcalO})}{(1 - \mu' Mc_{\bfcalO,\max}) },
\end{split}	\label{eq:require con theta}\ee
where the second inequality follows \eqref{eq:condition for PSGM} that $N\eta_{\bfcalX} + M\eta_{\bfcalO}\leq Nc_{\bfcalX,\min} $, and the last inequality utilizes \eqref{eq:step size} that $1 - \mu' Mc_{\bfcalO,\max} \geq \frac{3}{4} \geq 3\mu' Nc_{\bfcalX,\min}$. With \eqref{eq:tan theta small} and \eqref{eq:require con theta}, we now prove the two main arguments in \Cref{thm:convergence PSGM}.

\paragraph{ Case ($i$):} We first consider the case where $\theta_k$ is large compared with $\mu_k$:
\e
\sin(\theta_{k-1}) > \frac{\mu_k \|\bfcalX^\top\vs_{k-1}\|_1}{(1 - \mu_k \|\bfcalO^\top\widehat \vb_{k-1}\|_1)}.
\ee
It follows from \eqref{eq:bk useful} that
\e
\begin{split}
	\vb_k 
	& = (1 - \mu_k\|\bfcalO^\top\widehat \vb_{k-1}\|_1)\widehat \vb_{k-1} - \mu_k \|\bfcalX^\top\vs_{k-1}\|_1 \vs_{k-1} - \mu_k (\vp_{k-1} + \vq_{k-1})\\
	& = \left((1 - \mu_k \|\bfcalO^\top\widehat \vb_{k-1}\|_1)\sin(\theta_{k-1}) -\mu_k \|\bfcalX^\top\vs_{k-1}\|_1\right)\vs_{k-1}\\ & \quad + (1 - \mu_k \|\bfcalO^\top\widehat \vb_{k-1}\|_1)\cos(\theta_{k-1}) \vn_{k-1} - \mu_k (\vp_{k-1} + \vq_{k-1}),
\end{split}
\nonumber\ee
which furthr implies that
\e\begin{split}
	&\tan(\theta_k) = \frac{\|\calP_{\calS}(\vb_k)\|_2}{\|\calP_{\calS^\perp}(\vb_k)\|_2} \\
	&= \frac{ \left\|\calP_{\calS}\left(\left((1 - \mu_k \|\bfcalO^\top\widehat \vb_{k-1}\|_1)\sin(\theta_{k-1}) -\mu_k \|\bfcalX^\top\vs_{k-1}\|_1\right)\vs_{k-1} - \mu_k (\vp_{k-1} + \vq_{k-1}) \right)\right\|_2  }{ \left\|\calP_{\calS^\perp}\left( (1 - \mu_k \|\bfcalO^\top\widehat \vb_{k-1}\|_1)\cos(\theta_{k-1}) \vn_{k-1} - \mu_k (\vp_{k-1} + \vq_{k-1}) \right)\right\|_2 }\\
	& \leq \frac{ \left|(1 - \mu_k \|\bfcalO^\top\widehat \vb_{k-1}\|_1)\sin(\theta_{k-1}) - \mu_k \|\bfcalX^\top\vs_{k-1}\|_1\right| + \mu_k \left\| \calP_{\calS}\left(\vp_{k-1} + \vq_{k-1}\right)\right\|_2  }{  (1 - \mu_k \|\bfcalO^\top\widehat \vb_{k-1}\|_1)\cos(\theta_{k-1}) - \mu_k\left\| \calP_{\calS^\perp}\left( \vp_{k-1} + \vq_{k-1}\right) \right\|_2 }\\
	& = \frac{ (1 - \mu_k \|\bfcalO^\top\widehat \vb_{k-1}\|_1)\sin(\theta_{k-1}) -\mu_k \|\bfcalX^\top\vs_{k-1}\|_1+ \mu_k \left\| \calP_{\calS}\left(\vp_{k-1} + \vq_{k-1}\right)\right\|_2  }{  (1 - \mu_k \|\bfcalO^\top\widehat \vb_{k-1}\|_1)\cos(\theta_{k-1}) - \mu_k\left\| \calP_{\calS^\perp}\left( \vp_{k-1} + \vq_{k-1}\right) \right\|_2 }\\
	&  = \frac{(1 - \mu_k \|\bfcalO^\top\widehat \vb_{k-1}\|_1)\sin(\theta_{k-1}) -\mu_k \|\bfcalX^\top\vs_{k-1}\|_1 + \mu_k \|\calP_{\calS}(\vp_{k-1} + \vq_{k-1}) \|_2}{ (1 - \mu_k \|\bfcalO^\top\widehat \vb_{k-1}\|_1)\cos(\theta_{k-1}) - \mu_k\|\vp_{k-1} + \vq_{k-1} \|_2 + \mu_k\|\calP_{\calS}(\vp_{k-1} + \vq_{k-1}) \|_2}.
\end{split}
\label{eq:tan theta_k useful}\ee

Based on the term $\Gamma_1 := \frac{(1 - \mu_k \|\bfcalO^\top\widehat \vb_{k-1}\|_1)\sin(\theta_{k-1}) -\mu_k \|\bfcalX^\top\vs_{k-1}\|_1}{ (1 - \mu_k \|\bfcalO^\top\widehat \vb_{k-1}\|_1)\cos(\theta_{k-1}) - \mu_k\|\vp_{k-1} + \vq_{k-1} \|_2}$, we further bound $\tan(\theta_k)$ from above. In particular,  when $\Gamma_1\leq 1$, we have (utilizing the fact that $\frac{a+x}{b+x}$ is an increasing function of $x$ when $x\geq0$ and $0<a\leq b$)
\begin{align*}
\tan(\theta_k) \leq \tan(\overline\theta_k) = \frac{(1 - \mu_k \|\bfcalO^\top\widehat \vb_{k-1}\|_1)\sin(\theta_{k-1}) -\mu_k \|\bfcalX^\top\vs_{k-1}\|_1 + \mu_k \|\vp_{k-1} + \vq_{k-1} \|_2}{ (1 - \mu_k \|\bfcalO^\top\widehat \vb_{k-1}\|_1)\cos(\theta_{k-1})}
\end{align*}
by plugging $\|\calP_{\calS}(\vp_{k-1} + \vq_{k-1}) \|_2 = \|\vp_{k-1} + \vq_{k-1}\|_2$ into the laste term in \eqref{eq:tan theta_k useful}.
When $\Gamma_1> 1$, we have (utilizing the fact that $\frac{a+x}{b+x}$ is a decreasing function of $x$ when $x\geq0$ and $a> b>0$)
\begin{align*}
\tan(\theta_k) \leq \tan(\widetilde\theta_k)= \frac{(1 - \mu_k \|\bfcalO^\top\widehat \vb_{k-1}\|_1)\sin(\theta_{k-1}) -\mu_k \|\bfcalX^\top\vs_{k-1}\|_1 }{ (1 - \mu_k \|\bfcalO^\top\widehat \vb_{k-1}\|_1)\cos(\theta_{k-1}) - \mu_k\|\vp_{k-1} + \vq_{k-1} \|_2 }
\end{align*}
by plugging $\|\calP_{\calS}(\vp_{k-1} + \vq_{k-1}) \|_2 = 0$ into the last term in \eqref{eq:tan theta_k useful}. Combinging the above two cases together gives
\e
\theta_k \leq \max\{\overline\theta_k,\widetilde\theta_k\}.
\label{eq:two upper bounds theta}\ee

In what follows, we obtain upper bounds for both $\overline\theta_k$ and $\widetilde\theta_k$. Towards that end, we first bound $\tan(\overline\theta_k)$ from above as
\begin{align*}
&\tan(\overline\theta_k) = \frac{(1 - \mu_k \|\bfcalO^\top\widehat \vb_{k-1}\|_1)\sin(\theta_{k-1}) -\mu_k \|\bfcalX^\top\vs_{k-1}\|_1 + \mu_k \|\vp_{k-1} + \vq_{k-1} \|_2}{ (1 - \mu_k \|\bfcalO^\top\widehat \vb_{k-1}\|_1)\cos(\theta_{k-1})}\\
& = \tan(\theta_{k-1}) - \frac{ \mu_k \|\bfcalX^\top\vs_{k-1}\|_1 - \mu_k \|\vp_{k-1} + \vq_{k-1} \|_2}{ (1 - \mu_k \|\bfcalO^\top\widehat \vb_{k-1}\|_1)\cos(\theta_{k-1})}\\
& \leq  \tan(\theta_{k-1}) - \mu_k \frac{N c_{\bfcalX,\min} - (N\eta_{\bfcalX} + M\eta_{\bfcalO})}{ (1 - \mu_k Mc_{\bfcalO,\min})\cos(\theta_{k-1})}\\
& \leq \tan(\theta_{k-1}) - \mu_k(Nc_{\bfcalX,\min} - (N\eta_{\bfcalX} + M\eta_{\bfcalO})),
\end{align*}
where the last inequality follows because $0<1 - \mu_k Mc_{\bfcalO,\min}\leq 1$. Thus,
\e
\tan(\theta_{k-1}) - \tan(\overline\theta_k) \geq \mu_k(Nc_{\bfcalX,\min} - (N\eta_{\bfcalX} + M\eta_{\bfcalO})).
\label{eq:bound tan bartheta}\ee

Similarly, $\tan(\widetilde\theta_k)$ can be bounded by
\begin{align*}
&\tan(\widetilde\theta_k) = \frac{(1 - \mu_k \|\bfcalO^\top\widehat \vb_{k-1}\|_1)\sin(\theta_{k-1}) -\mu_k \|\bfcalX^\top\vs_{k-1}\|_1 }{ (1 - \mu_k \|\bfcalO^\top\widehat \vb_{k-1}\|_1)\cos(\theta_{k-1}) - \mu_k\|\vp_{k-1} + \vq_{k-1} \|_2 }\\
& = \tan(\theta_{k-1}) -  \frac{ \mu_k \|\bfcalX^\top\vs_{k-1}\|_1 - \mu_k \tan(\theta_{k-1}) \|\vp_{k-1} + \vq_{k-1} \|_2}{ (1 - \mu_k \|\bfcalO^\top\widehat \vb_{k-1}\|_1)\cos(\theta_{k-1}) - \mu_k\|\vp_{k-1} + \vq_{k-1} \|_2 }\\
& \leq \tan(\theta_{k-1}) -  \frac{ \mu_k \|\bfcalX^\top\vs_{k-1}\|_1 - \mu_k \tan(\theta_{k-1}) \|\vp_{k-1} + \vq_{k-1} \|_2}{ (1 - \mu_k \|\bfcalO^\top\widehat \vb_{k-1}\|_1)\cos(\theta_{k-1})  }\\
& \leq \tan(\theta_{k-1}) - \mu_k \frac{ c_{\bfcalX,\min} - \tan(\theta_{k-1})(N\eta_{\bfcalX} + M\eta_{\bfcalO})}{ (1 - \mu_k Mc_{\bfcalO,\min})\cos(\theta_{k-1}) }\\
& \leq \tan(\theta_{k-1}) - \mu_k(Nc_{\bfcalX,\min} - \tan(\theta_{k-1})(N\eta_{\bfcalX} + M\eta_{\bfcalO})),
\end{align*}
where the forth line utilizes $\|\bfcalX^\top\vs_{k-1}\|_1 \geq c_{\bfcalX,\min}$, $\|\bfcalO^\top\widehat \vb_{k-1}\|_1 \geq c_{\bfcalO,\min}$, $\|\vp_{k-1} + \vq_{k-1} \|_2\leq \|\vp_{k-1} \|_2 + \| \vq_{k-1} \|_2\leq N\eta_{\bfcalX} + M\eta_{\bfcalO}$, assumption \eqref{eq:tan theta small} ensuring $c_{\bfcalX,\min} - \tan(\theta_{k-1})(N\eta_{\bfcalX} + M\eta_{\bfcalO})>0$, and \eqref{eq:step size} ensuring $1 - \mu_k Mc_{\bfcalO,\max}>0$. It follows that
\e
\tan(\theta_{k-1}) - \tan(\overline\theta_k) \geq \mu_k(Nc_{\bfcalX,\min} - \tan(\theta_{k-1})(N\eta_{\bfcalX} + M\eta_{\bfcalO})),
\nonumber\ee
which together with \eqref{eq:bound tan bartheta} and \eqref{eq:two upper bounds theta} gives
\e
\tan(\theta_{k-1}) - \tan(\theta_k) \geq \mu_k\min\{N c_{\bfcalX,\min} - (N\eta_{\bfcalX} + M\eta_{\bfcalO}), Nc_{\bfcalX,\min} - \tan(\theta_{k-1})(N\eta_{\bfcalX} + M\eta_{\bfcalO})\}.
\label{eq:decay tan theta proof}\ee
This proves \eqref{eq:decay region I}.

\paragraph{Case ($ii$):} We now consider the other case where $\theta_{k-1}$ is relatively small compared with $\mu_k$:
\e
\sin(\theta_{k-1}) \leq \frac{\mu_k \|\bfcalX^\top\vs_{k-1}\|_1}{(1 - \mu_k \|\bfcalO^\top\widehat \vb_{k-1}\|_1)}.\ee
In this case, instead of showing that $\theta_k\leq \theta_{k-1}$ (actually it is possible that $\theta_k> \theta_{k-1}$), we turn to characterize the width of the vibration, i.e., $\theta_k$  is also small and if it increase, it will not increase too much. Towards that end, we first bound $\cos(\theta_{k-1})$ as
\e\begin{split}
	(1 - \mu_k Mc_{\bfcalO,\max})\cos(\theta_{k-1}) &= (1 - \mu_k Mc_{\bfcalO,\max})\sqrt{1-\sin^2(\theta_{k-1})}\\ &\geq \sqrt{ (1 - \mu_k Mc_{\bfcalO,\max})^2 - (\mu_k Nc_{\bfcalX,\max})^2}\\
	& \geq \sqrt{ (1 - \mu_0 Mc_{\bfcalO,\max})^2 - (\mu_0 Nc_{\bfcalX,\max})^2} \\&\geq  \sqrt{(1-1/4)^2 -(1/4)^2}  \geq \frac{\sqrt{2}}{2}.
\end{split}\label{eq:cos theta large}\ee

We now use a similar but slightly different approach as in \eqref{eq:tan theta_k useful} to bound $\theta_k$:
\begin{align*}
&\tan(\theta_k) = \frac{\|\calP_{\calS}(\vb_k)\|_2}{\|\calP_{\calS^\perp}(\vb_k)\|_2}\\& = \frac{ \left\|\calP_{\calS}\left(\left((1 - \mu_k \|\bfcalO^\top\widehat \vb_{k-1}\|_1)\sin(\theta_{k-1}) -\mu_k \|\bfcalX^\top\vs_{k-1}\|_1\right)\vs_{k-1} + \mu_k (\vp_{k-1} + \vq_{k-1}) \right)\right\|_2  }{ \left\|\calP_{\calS^\perp}\left( (1 - \mu_k \|\bfcalO^\top\widehat \vb_{k-1}\|_1)\cos(\theta_{k-1}) \vn_{k-1} - \mu_k (\vp_{k-1} + \vq_{k-1}) \right)\right\|_2 }\\
& \leq \frac{ \left|(1 - \mu_k \|\bfcalO^\top\widehat \vb_{k-1}\|_1)\sin(\theta_{k-1}) -\mu_k \|\bfcalX^\top\vs_{k-1}\|_1\right| + \mu_k \left\| \calP_{\calS}\left(\vp_{k-1} + \vq_{k-1}\right)\right\|_2  }{  (1 - \mu_k \|\bfcalO^\top\widehat \vb_{k-1}\|_1)\cos(\theta_{k-1}) - \mu_k\left\| \calP_{\calS^\perp}\left( \vp_{k-1} + \vq_{k-1}\right) \right\|_2 }\\
& = \frac{ \mu_k \|\bfcalX^\top\vs_{k-1}\|_1 - (1 - \mu_k \|\bfcalO^\top\widehat \vb_{k-1}\|_1)\sin(\theta_{k-1}) + \mu_k \left\| \calP_{\calS}\left(\vp_{k-1} + \vq_{k-1}\right)\right\|_2  }{  (1 - \mu_k \|\bfcalO^\top\widehat \vb_{k-1}\|_1)\cos(\theta_{k-1}) - \mu_k\left\| \calP_{\calS^\perp}\left( \vp_{k-1} + \vq_{k-1}\right) \right\|_2 }\\
& \leq \frac{ \mu_k \|\bfcalX^\top\vs_{k-1}\|_1 + \mu_k \left\| \calP_{\calS}\left(\vp_{k-1} + \vq_{k-1}\right)\right\|_2  }{  (1 - \mu_k \|\bfcalO^\top\widehat \vb_{k-1}\|_1)\cos(\theta_{k-1}) - \mu_k\left\| \calP_{\calS^\perp}\left( \vp_{k-1} + \vq_{k-1}\right) \right\|_2 }\\
& \leq \frac{\mu_k Nc_{\bfcalX,\max}  + \mu_k \|\calP_{\calS}(\vp_{k-1} + \vq_{k-1}) \|_2}{ (1 - \mu_k Mc_{\bfcalO,\max})\cos(\theta_{k-1}) - \mu_k\|\calP_{\calS^\perp}(\vp_{k-1} + \vq_{k-1}) \|_2}\\
&  = \frac{\mu_k Nc_{\bfcalX,\max}  + \mu_k \|\calP_{\calS}(\vp_{k-1} + \vq_{k-1}) \|_2}{ (1 - \mu_k Mc_{\bfcalO,\max})\cos(\theta_{k-1}) - \mu_k\|\vp_{k-1} + \vq_{k-1} \|_2 + \mu_k\|\calP_{\calS}(\vp_{k-1} + \vq_{k-1}) \|_2}.
\end{align*}

Similar to the argument utilized for \eqref{eq:two upper bounds theta}, and with the abuse of notations as in \eqref{eq:two upper bounds theta}, we have
\e
\theta_k \leq \max\{\overline\theta_k,\widetilde\theta_k\},
\nonumber\ee
where
\begin{align*}
\tan(\overline\theta_k) &= \frac{\mu_k Nc_{\bfcalX,\max} + \mu_k \|\vp_{k-1} + \vq_{k-1} \|_2}{ (1 - \mu_k Mc_{\bfcalO,\max})\cos(\theta_{k-1})}\\
& \leq \frac{\mu_k Nc_{\bfcalX,\max} + \mu_k (N\eta_{\bfcalX} + M\eta_{\bfcalO})}{ (1 - \mu_k Mc_{\bfcalO,\max})\cos(\theta_{k-1})} \\
&= \mu_k\frac{ c_{\bfcalX,\max} + (N\eta_{\bfcalX} + M\eta_{\bfcalO})}{ (1 - \mu_k Mc_{\bfcalO,\max})\cos(\theta_{k-1})},
\end{align*}
and
\begin{align*}
\tan(\widetilde\theta_k) &= \frac{\mu_k Nc_{\bfcalX,\max} }{ (1 - \mu_k Mc_{\bfcalO,\max})\cos(\theta_{k-1}) - \mu_k\|\vp_{k-1} + \vq_{k-1} \|_2 }\\
& \leq \frac{\mu_k Nc_{\bfcalX,\min} }{ (1 - \mu_k Mc_{\bfcalO,\max})\cos(\theta_{k-1}) - \mu_k(N\eta_{\bfcalX} + M\eta_{\bfcalO}) }.
\end{align*}
The above three equations indicate that
\e\begin{split}
	\tan(\theta_k) \leq \max\bigg\{ &\mu_k \frac{N c_{\bfcalX,\max} + (N\eta_{\bfcalX} + M\eta_{\bfcalO})}{ (1 - \mu_k Mc_{\bfcalO,\max})\cos(\theta_{k-1})}, \\ &\frac{ \mu_k  N c_{\bfcalX,\max} }{ (1 - \mu_k Mc_{\bfcalO,\max})\cos(\theta_{k-1}) - \mu_k(N\eta_{\bfcalX} + M\eta_{\bfcalO}) } \bigg\}.
\end{split}	\label{eq:theta small region II}\ee

The first term inside the $\{\}$ of \eqref{eq:theta small region II} can be further bounded by
\e\begin{split}
	\mu_k  \frac{N c_{\bfcalX,\max} + (N\eta_{\bfcalX} + M\eta_{\bfcalO})}{ (1 - \mu_k Mc_{\bfcalO,\max})\cos(\theta_{k-1})} &= \frac{\mu_k}{\mu'}  \frac{\mu' N c_{\bfcalX,\max} + \mu'(N\eta_{\bfcalX} + M\eta_{\bfcalO})}{ (1 - \mu_k Mc_{\bfcalO,\max})\cos(\theta_{k-1})} \\
	& \leq \frac{\mu_k}{\mu'} \frac{\mu' N c_{\bfcalX,\max} + \mu' Nc_{\bfcalX,\min} }{ (1 - \mu_k Mc_{\bfcalO,\max})\cos(\theta_{k-1})} \\
	&\leq\frac{1/4 + 1/4}{\sqrt{2}/2}\frac{\mu_k}{\mu'}= \frac{1}{\sqrt{2}}\frac{\mu_k}{\mu'},
\end{split}\nonumber\ee
where the first inequality follows from \eqref{eq:condition for PSGM} that $N\eta_{\bfcalX} + M\eta_{\bfcalO}\leq N c_{\bfcalX,\min} $, and the second inequality we utilize \eqref{eq:step size} and $(1 - \mu_k Mc_{\bfcalO,\max})\cos(\theta_{k-1}) \geq \frac{\sqrt{2}}{2}$ from \eqref{eq:cos theta large}. Similarly, the second term inside the $\{\}$ of \eqref{eq:theta small region II} can be bounded by
\e\begin{split}
	\frac{\mu_k Nc_{\bfcalX,\max} }{ (1 - \mu_k Mc_{\bfcalO,\max})\cos(\theta_{k-1}) - \mu' Nc_{\bfcalX,\min} } &= \frac{\mu_k}{\mu'}\frac{\mu' Nc_{\bfcalX,\max} }{ (1 - \mu_k Mc_{\bfcalO,\max})\cos(\theta_{k-1}) - \mu' Nc_{\bfcalX,\min} } \\&\leq \frac{1/4 }{ \frac{\sqrt{2}}{2} - 1/4 }\frac{\mu_k}{\mu'} = \frac{1}{(2\sqrt{2}-1)}\frac{\mu_k}{\mu'},
\end{split}\nonumber\ee
where the first inequality utilizes \eqref{eq:step size} and \eqref{eq:cos theta large}. It follows from the above two equations that
\e\begin{split}
	\tan(\theta_k) & \leq  \frac{1}{\sqrt{2}}\frac{\mu_k}{\mu'}.
\end{split}	\label{eq:theta small region II 2}\ee

\paragraph{Proof of \eqref{eq:tan theta small}} The remaining part is to show that \eqref{eq:tan theta small} holds for all $k\geq 0$. We prove it by induction. Due to the condition for the initialization in \eqref{eq:initialization of GPD}, \eqref{eq:tan theta small}  holds for $k = 0$.

In what follows, we suppose that \eqref{eq:tan theta small} holds for $k\geq 0$ and prove that \eqref{eq:tan theta small} holds  for $k+1$. Towards that end, we first invoke \eqref{eq:decay tan theta proof} and \eqref{eq:theta small region II 2} to obtain that either $\tan(\theta_{k+1}) < \tan(\theta_k)$ or $
\tan(\theta_{k+1}) \leq \frac{1}{\sqrt{2}}\frac{\mu_{k+1}}{\mu'}$.  In the former case, we automatically have  \eqref{eq:tan theta small}  for $k+1$ since$\theta_k < \theta_{0}$.

We now consider the other case $
\tan(\theta_{k+1}) \leq \frac{1}{\sqrt{2}}\frac{\mu_{k+1}}{\mu'}$, which along with \eqref{eq:step size}  and \eqref{eq:condition for PSGM} gives
\[
\tan(\theta_{k+1}) \leq \frac{1}{\sqrt{2}}\frac{\mu_{k+1}}{\mu'} \leq \frac{1}{\sqrt{2}} < \frac{N c_{\bfcalX,\min}}{N\eta_{\bfcalX} + M\eta_{\bfcalO}}.
\]
Thus, \eqref{eq:tan theta small} also holds for $k+1$. By induction, we conclude that \eqref{eq:tan theta small} holds for all $k\geq 0$. This completes the proof of \Cref{thm:convergence PSGM}.

\vskip 0.2in
\bibliography{biblio/dataset,biblio/learning,biblio/math,biblio/sparse,biblio/vidal,biblio/vision,biblio/geometry}

\end{document}